\documentclass{article}
\usepackage[english]{babel}
\usepackage[utf8]{inputenc}\usepackage{lmodern}

\usepackage{johd}
\usepackage[utf8]{inputenc} 
\usepackage[T1]{fontenc}    
\usepackage{hyperref}       
\hypersetup{
	colorlinks=true,       
	linkcolor=red,          
	citecolor=blue,        
	filecolor=magenta,      
	urlcolor=cyan           
}
\usepackage{url}            
\usepackage{booktabs}       
\usepackage{amsfonts}       
\usepackage{nicefrac}       
\usepackage{microtype}      
\usepackage{amsmath,amsfonts}       
\usepackage{nicefrac}       
\usepackage{microtype}      
\usepackage{amsfonts,amssymb,amsmath,bm,paralist,ifthen}
\usepackage{algorithm2e}
\usepackage{amsmath}
\usepackage{comment}
\usepackage{graphicx}

\usepackage{siunitx,booktabs}

\usepackage{amsmath}
\usepackage{amsfonts,bm}
\usepackage{amssymb}
\usepackage{multirow}

\newtheorem{theorem}{Theorem}
\newtheorem{lemma}{Lemma}
\newtheorem{proposition}{Proposition}
\newtheorem{remark}{Remark}
\newtheorem{corollary}{Corollary}

\usepackage{subfigure}
\makeatletter
\newcommand\makebig[2]{%
	\@xp\newcommand\@xp*\csname#1\endcsname{\bBigg@{#2}}%
	\@xp\newcommand\@xp*\csname#1l\endcsname{\@xp\mathopen\csname#1\endcsname}%
	\@xp\newcommand\@xp*\csname#1r\endcsname{\@xp\mathclose\csname#1\endcsname}%
}
\makeatother

\makebig{biggg} {3.0}
\makebig{Biggg} {3.5}
\makebig{bigggg}{4.0}
\makebig{Bigggg}{4.5}

\usepackage[english]{babel}

\newenvironment{rcases}
{\left.\begin{aligned}}
	{\end{aligned}\right\rbrace}
\newcommand{\norm}[1]{\left\lVert#1\right\rVert}

\usepackage{booktabs,makecell}

\usepackage{siunitx}
\makeatletter
\def\underbracex#1#2{\mathop{\vtop{\m@th\ialign{##\crcr
				$\hfil\displaystyle{#2}\hfil$\crcr
				\noalign{\kern3\p@\nointerlineskip}%
				#1\crcr\noalign{\kern3\p@}}}}\limits}

\def\underbracea{\underbracex\upbracefilla}

\def\upbracefilla{$\m@th \setbox\z@\hbox{$\braceld$}%
	\bracelu\leaders\vrule \@height\ht\z@ \@depth\z@\hfill 
	\kern\p@\vrule \@width\p@\kern\p@\vrule \@width\p@\kern\p@\vrule \@width\p@
	$}

\def\underbraceb{\underbracex\upbracefillb}

\def\upbracefillb{$\m@th \setbox\z@\hbox{$\braceld$}%
	\vrule \@width\p@\kern\p@\vrule \@width\p@\kern\p@\vrule \@width\p@\kern\p@
	\leaders\vrule \@height\ht\z@ \@depth\z@\hfill\bracerd
	\braceld\leaders\vrule \@height\ht\z@ \@depth\z@\hfill
	\kern\p@\vrule \@width\p@\kern\p@\vrule \@width\p@\kern\p@\vrule \@width\p@
	$}

\def\upbracefillc{$\m@th \setbox\z@\hbox{$\braceld$}%
	\vrule \@width\p@\kern\p@\vrule \@width\p@\kern\p@\vrule \@width\p@\kern\p@
	\leaders\vrule \@height\ht\z@ \@depth\z@\hfill
	\kern\p@\vrule \@width\p@\kern\p@\vrule \@width\p@\kern\p@\vrule \@width\p@
	$}

\def\upbracefilld{$\m@th \setbox\z@\hbox{$\braceld$}%
	\vrule \@width\p@\kern\p@\vrule \@width\p@\kern\p@\vrule \@width\p@\kern\p@
	\leaders\vrule \@height\ht\z@ \@depth\z@\hfill\braceru$}

\def\underbracebd{\underbracex\upbracefillbd}
\def\upbracefillbd{$\m@th \setbox\z@\hbox{$\braceld$}%
	\vrule \@width\p@\kern\p@\vrule \@width\p@\kern\p@\vrule \@width\p@\kern\p@
	\bracerd\braceld
	\leaders\vrule \@height\ht\z@ \@depth\z@\hfill\braceru$}
\newcommand{\blue}[1]{\textcolor{blue}{#1}}

\title{Online Stochastic Gradient Descent Learns Linear Dynamical Systems from A Single Trajectory}

\author{
	Navid Reyhanian and Jarvis Haupt\\
	Department of Electrical and Computer Engineering\\
	University of Minnesota\\
	Minneapolis, MN 55455 \\
	\texttt{\{navid,jdhaupt\}@umn.edu} \\
}

\date{} 

\begin{document}
	
	\maketitle
	
\begin{abstract}%
	This work investigates the problem of estimating the weight matrices of a stable time-invariant linear
	dynamical system from a single sequence of noisy measurements. We show that if the unknown
	weight matrices describing the system are in Brunovsky canonical form, we can efficiently estimate
	the ground truth unknown matrices of the system from a linear system of equations formulated based
	on the transfer function of the system, using both online and offline stochastic gradient descent
	(SGD) methods. Specifically, by deriving concrete complexity bounds, we show that SGD converges
	linearly in expectation to any arbitrary small Frobenius norm distance from the ground truth weights.
	To the best of our knowledge, ours is the first work to establish linear convergence characteristics for
	online and offline gradient-based iterative methods for weight matrix estimation in linear dynamical
	systems from a single trajectory. Extensive numerical tests verify that the performance of the
	proposed methods is consistent with our theory, and show their superior performance relative to
	existing state of the art methods.
\end{abstract}

\section{Introduction}
Consider the linear time-invariant dynamical system giving rise to a single (or multiple) finite trajectory of noisy outputs $\{\mathbf{y}_t\}_t$, described by the following dynamics:
\begin{subequations}
	\begin{align}
		&\mathbf{h}_{t+1}=\mathbf{A}\mathbf{h}_{t}+\mathbf{B}\mathbf{u}_t,\label{eq:state}\\&
		\mathbf{y}_t=\mathbf{C}\mathbf{h}_t+\mathbf{D}\mathbf{u}_t+\boldsymbol{\zeta}_t,\label{eq:out}
	\end{align}\label{eq:sys}
\end{subequations}
where $\mathbf{h}_t$, $\mathbf{u}_t$ and $\boldsymbol{\zeta}_t$ represent the hidden state, the input and the noise of the measurement at time instance $t$, respectively. Here, the weight matrices $\mathbf{A}$, $\mathbf{B}$, $\mathbf{C}$ and $\mathbf{D}$ parameterize the system; we consider these as unknowns here. 

In the system described in \eqref{eq:state}-\eqref{eq:out}, the hidden state $\mathbf{h}_{t}$ cannot be measured. Instead, the system is indirectly measured from outputs. System identification for \eqref{eq:sys} -- the problem of identifying the unknown weight matrices (or a set of weight matrices giving identical dynamics) -- is involved in a wide variety of time-series analyses, robotics, economics, and modern control problems. Examples include text translation, time-series predictions, speech recognition, and many others \cite{graves2013speech,bahdanau2014neural,liu2015regularized,tu2012dynamical,thieffry2019trajectory}. 

Besides their vital applications in the control theory, there is recent interest from the machine learning community in linear dynamical systems due to their connections with recurrent neural networks (RNNs). Indeed, similar to linear dynamical systems, RNNs process the inputs to the system using their internal hidden states \cite{oymak2019stochastic,hardt2018gradient}. Explorations of the connections between the linear dynamical systems and RNNs are fairly recent (see the aforementioned, as well as  \cite{laurent2016recurrent,oymak2019stochastic,chang2019antisymmetricrnn,sattar2020non}), and elucidating these connections plays a critical role in better understanding RNNs, such as long short-term memories (LSTMs) and gated recurrent units (GRUs), which have achieved significant success in different applications. 

\subsection{Relevant Work}\label{sec:litrev}
As a means of placing this work in proper context in the broader literature, we classify relevant work to this paper into two major domains: the papers that study linear dynamical systems and papers that study RNNs.  
\paragraph{Linear Dynamical Systems.}
A rich literature exists in control and systems theory on the identification of linear dynamical systems; see, e.g., \cite{ljung1999system,ho1966effective,venkatesh2001system,aastrom1971system}. More recent literature concentrates on data-driven approaches and provides sample complexity bounds, such as \cite{pereira2010learning,hardt2018gradient,oymak2019non,faradonbeh2018finite,sarkar2019finite,wagenmaker2020active,simchowitz2018learning,tsiamis2019finite,sarkar2018fast,simchowitz2019learning,zheng2020non,sun2020finite}. Given noisy observations generated by a discrete linear dynamical system, a gradient projection approach is proposed in \cite{hardt2018gradient} to minimize the population risk of learning an unknown, stable, single-input and single-output (SISO) system at a sublinear convergence rate. If $q(z)=z^n+a_1z^{n-1}+\dots+a_n$ is the characteristic polynomial of the system, \cite{hardt2018gradient} assumes that $\{a_i\}_{i=1}^n$ are such that the real and imaginary parts of $q(z)/z^n$ satisfy $\Re(q(z)/z^n)>|\Im(q(z)/z^n)|$ for any $z$, where $|z|=1$. The gradient approach in \cite{hardt2018gradient} fails if the $\{a_i\}_{i=1}^n$ in the characteristic polynomial of the underlying system do not satisfy the above assumption, and when $n$ increases, the above assumption becomes more difficult to be satisfied.
The SISO results are extended in \cite{hardt2018gradient} to multiple-input and multiple-output (MIMO) systems, where unknown transformation matrices $\mathbf{A}$ and $\mathbf{B}$ have the Brunovsky canonical form. 

Learning unknown weight matrices of an observable and controllable stable linear dynamical system is studied in \cite{oymak2019non}. Unlike \cite{hardt2018gradient} that updates the estimation of unknown weight matrices in each iteration using a subset of samples, the approach given in \cite{oymak2019non} processes all samples at the same time. In particular, there a set of $T$ Markov parameters of the system, denoted by $\boldsymbol{\Theta}_T$, are first estimated in \cite{oymak2019non}. Then, a Ho-Kalman algorithm that uses SVD is proposed to estimate the weight matrices from the estimated Hankel matrix. Although the identified weight matrices by \cite{oymak2019non} build an equivalent system that has an identical performance to the unknown system, the weight matrices are not necessarily equal to those for the underlying system. In \cite{simchowitz2019learning}, the authors provide complexity bounds for the estimated Markov parameters, where a prefiltered least squares approach is proposed to mitigate the effect of truncated Markov parameters and the measurement noise. Similar to \cite{oymak2019non}, \cite{simchowitz2019learning,sarkar2019finite,tsiamis2019finite,zheng2020non,sun2020finite} use Ho-Kalman type algorithms. The drawback of these approaches is that the size of the Hankel matrix increases quadratically with the number of estimated Markov parameters, which increases the cost of the corresponding SVDs. The estimation errors decay at a rate of $1/\sqrt{N}$ in \cite{oymak2019non,sarkar2019finite,tsiamis2019finite,zheng2020non,sun2020finite}, where $N$ denotes the trajectory length. 

Different assumptions about the stability, system order and the number of required trajectories to excite the unknown system are made in existing papers. When the spectral radius of $\mathbf{A}$ is less than one, i.e., $\rho(\mathbf{A})<1$, the linear dynamical system becomes stable. It is marginally stable and unstable if $\rho(\mathbf{A})\leq1$ and $\rho(\mathbf{A})>1$, respectively. Table \ref{tab:list} summarizes different assumptions in existing papers. The approaches in the above papers require all input-output samples to be stored in the memory, which makes them (potentially) memory inefficient. Furthermore, the approaches explained above are not necessarily scalable since they simultaneously process all samples of one (or multiple) trajectory to learn weight matrices. From the last column of Table \ref{tab:list}, we observe that only \cite{hardt2018gradient} provides conditions and guarantees for its proposed algorithm to converge to the ground truth weight matrices. 

If the dynamics of a system can be fully described only by \eqref{eq:state} and the system output is generated by 
$\mathbf{y}_t=\mathbf{h}_{t+1}$, the system is directly measured. Unlike the above papers that address the identification of indirectly measured systems, a number of papers study directly measured systems from a single trajectory. In \cite{simchowitz2018learning} and \cite{faradonbeh2018finite}, the estimation of $\mathbf{A}$ from a system with dynamic $\mathbf{h}_{t+1}=\mathbf{A}\mathbf{h}_t+\mathbf{u}_t$ is studied. Similarly, \cite{sarkar2018fast} studied estimating $\mathbf{A}$ and $\mathbf{B}$ from $\mathbf{h}_{t+1}=\mathbf{A}\mathbf{h}_t+\mathbf{B}\mathbf{u}_t+\boldsymbol{\zeta}_t$ via a regression method; where error bounds are provided. The same dynamics are considered in \cite{wagenmaker2020active}, where $\mathbf{A}$ is unknown and $\mathbf{B}$ is considered to be known. It is proven in \cite{wagenmaker2020active} that the estimation of $\mathbf{A}$ can be accelerated if inputs are controlled rather than merely being Gaussian.

\begin{table*}
	\centering
	
	\caption{\textsc{A summary of recent non-asymptotic analysis for LTI system learning}}
	\scalebox{0.7}{
		\begin{tabular}{c c c c c c c c c c}
			\hline \hline
			\multirow{2}{*}{ Paper} &  
			\multirow{2}{*}{\thead{ Known\\ order}} & 
			\multirow{2}{*}{ Meas.} & 
			\multirow{2}{*}{ Type} & 
			\multirow{2}{*}{ Stability} & 
			\multirow{2}{*}{ \thead{ \# of \\ trajectories}} & 
			\multirow{2}{*}{ Online} & 
			\multirow{2}{*}{ Inputs} & 
			\multicolumn{2}{c}{Estimation error}\\ 
			&&&&&&&& $\mathbf{\Theta}_T$ & \thead{$\mathbf{A},\mathbf{B}$\\ $\mathbf{C},\mathbf{D}$} \\ \cmidrule(lr){1-10}
			\cite{simchowitz2019learning} & No & Indirect & MIMO & $\rho(\mathbf{A})\leq 1$ & Single  & No & Gaussian & $\mathcal{O}(\frac{1}{\sqrt{N}})$ & --- \\
			\cite{oymak2019non} & No & Indirect & MIMO & $\rho(\mathbf{A})< 1$ & Single  & No & Gaussian & $\mathcal{O}(\frac{1}{\sqrt{N}})$ & --- \\
			\cite{sarkar2019finite} & Yes & Indirect & MIMO & $\rho(\mathbf{A})<  1$ & Single  & No & Gaussian & $\mathcal{O}(\frac{1}{\sqrt{N}})$ & --- \\
			\cite{tsiamis2019finite} & Yes & Indirect & MIMO & $\rho(\mathbf{A}) \leq 1$ & Single  & No & Gaussian & $\mathcal{O}(\frac{1}{\sqrt{N}})$ & --- \\
			\cite{hardt2018gradient} & Yes & Indirect & MIMO & $\rho(\mathbf{A})<  1$ & Multiple  & No & Gaussian & --- & $\mathcal{O}(\frac{1}{\sqrt{N}})$  \\
			\cite{zheng2020non} & No & Indirect & MIMO & Any & Multiple  & No & Gaussian & $\mathcal{O}(\frac{1}{\sqrt{N}})$ & --- \\
			\cite{sun2020finite} & No & Indirect & MISO & Any & Multiple  & No & Gaussian & $\mathcal{O}(\frac{1}{\sqrt{N}})$ & --- \\
			This paper & Yes & Indirect & MIMO & $\rho(\mathbf{A})< 1$ & Single  & Yes & Gaussian & $\mathcal{O}(\frac{1}{\sqrt{N}})$ & $\mathcal{O}(\frac{1}{\sqrt{N}})$ \\
			\cmidrule(lr){1-10}
			\cite{wagenmaker2020active} & Yes & Direct & MIMO & $\rho(\mathbf{A})<  1$ & Single  & No & Controlled & --- & $\mathcal{O}(\frac{1}{\sqrt{N}})$  \\
			\cite{sarkar2018fast} & Yes & Direct & MIMO & Any & Single  & No & Gaussian & --- & $\mathcal{O}(\frac{1}{\sqrt{N}})$ \\
			\cite{faradonbeh2018finite} & Yes & Direct & MIMO & Any & Single  & No & Gaussian & --- & ---  \\
			\cite{simchowitz2018learning} & Yes & Direct & MIMO & $\rho(\mathbf{A})\leq 1$ & Single  & No & Gaussian & --- & $\mathcal{O}(\frac{1}{\sqrt{N}})$ \\
			\bottomrule
	\end{tabular}}
	\label{tab:list}
\end{table*}
\paragraph{Recurrent Neural Networks.}
It is common to consider RNNs as non-linear dynamical systems. A growing number of papers have recently studied training RNNs and provided theoretical guarantees for the problem. The connections between RNNs and state equations of simple dynamical systems are characterized in \cite{chang2019antisymmetricrnn,oymak2019stochastic,sattar2020non}, where a neural network architecture is proposed to capture long-term dependencies enabled by the stability property of its underlying differential equation. In \cite{oymak2019stochastic}, a discrete-time dynamical system controlled by the state equation is considered, and an SGD algorithm is proposed to learn weight matrices of the state equation when the output layer activation function is a leaky rectified linear unit (ReLU). The approach in \cite{oymak2019stochastic} is extended in \cite{sattar2020non}, where the noise of measurements is also considered in the recursion dynamics. 
In \cite{tallec2018can,bahmani2019convex}, similar dynamical systems managed by the state equation given in \cite{oymak2019stochastic} are studied with different activation functions. In \cite{tallec2018can}, the activation function is hyperbolic tangent, however, it is differentiable and strongly convex in \cite{bahmani2019convex}. 
To prove the convergence of the proposed algorithms in the above papers for learning unknown weight matrices, it is assumed that the hidden state of the system is observable.
In practice, however, large RNNs with complex state evolutions are required to increase the representation power of networks. When $\mathbf{h}_{t+1}=\tanh(\mathbf{A}\mathbf{h}_{t}+\mathbf{B}\mathbf{u}_t) $ is considered instead of \eqref{eq:state}, a particular class of RNNs is obtained. The identification of this class via a non-linear regression is studied in \cite{vural2020rnn}. With continually running the RNN and implementing a gradient method to update $\mathbf{A}$, $\mathbf{B}$ and $\mathbf{C}$ from a non-linear regression, \cite{vural2020rnn} shows that a local minima of the problem can be obtained.

\subsection{Summary of Our Contributions}
We study the identification of a stable linear dynamical system based on a single sequence of input-output pairs. We formulate a finite sum problem to efficiently learn the truncated Markov parameters of the system. The formulated problem becomes strongly convex when the system input is white Gaussian noise. We prove that the sequence length strictly decreases the Frobenius norm distance between the regression solution and the truncated ground truth Markov parameters with a rate of $1/\sqrt{N}$. However, when the trajectory length increases, the complexity of solving the finite sum problem via the pseudo-inverse method increases. We propose low iteration cost online and offline stochastic gradient descent (SGD) algorithms to efficiently learn truncated Markov parameters. The offline SGD algorithms works on a batch of input-output pairs, however, the online SGD uses the most recent input-output pair to implement a gradient step in an online streaming fashion and then discards it. Therefore, it is storage efficient as compared with the existing methods. Via novel complexity bounds, we prove that when the system input is Gaussian, the proposed SGD algorithms \emph{linearly} converge in expectation to the finite sum solution. Unlike full-batch methods in \cite{hardt2018gradient,oymak2019non,sarkar2019finite,tsiamis2019finite}, an update step in our SGD algorithms is simply implemented via one input-output pair rather than a trajectory. 

When the unknown weight matrices have \emph{Brunovsky canonical form}, which is perhaps the most widely used form in control theory \cite{martin1978applications,hazewinkel1983representations,liu2008global}, we propose a novel approach to uniquely identify the ground truth weight matrices from a linear system of equations formulated based on the SGD iterates and the transfer function of the linear dynamical system. This is unlike widely used Ho-Kalman methods in \cite{simchowitz2019learning,oymak2019non,sarkar2019finite,tsiamis2019finite,zheng2020non,sun2020finite} that estimate some weight matrices to find a system with an equivalent performance. We solve the proposed linear system of equations in each iteration of the SGD algorithms. We use the derived bounds for the proposed SGD algorithms to develop complexity bounds for the identification of unknown weight matrices. We provide guarantees that the estimated weight matrices from the proposed linear system built from SGD iterates \emph{linearly} converge in expectation to the ground truth values.  Extensive numerical tests confirm the linear convergence of proposed approaches and demonstrate that they outperform the existing state of the art methods. 
\section{Problem Setup}
\label{gen_inst}

\paragraph{Notation.}
Bold upper-case and lower-case letters are used to denote matrices and vectors, respectively. The trace of matrix $\mathbf{M}$ is denoted by $\text{Tr}(\mathbf{M})$. $\mathbf{M}'$ denotes the transpose of $\mathbf{M}$. Given a matrix $\mathbf{M}$, $\norm{\mathbf{M}}_F$ denotes the Frobenius norm, and given a vector $\mathbf{m}$, $\norm{\mathbf{m}}_2$ denotes the $\ell_2$-norm. A diagonal matrix is denoted by $\mathbf{M}_{\mathbf{m}}$, where elements of $\mathbf{m}$ are on the diagonal. The vector of elements of $\mathbf{m}$ raised to power $2$ is denoted by $\mathbf{m}^{\cdot 2}$. We denote $(i,j)^{\text{th}}$ element of $\mathbf{M}$ by $[\mathbf{M}]_{ij}$. The spectral radius of matrix $\mathbf{M}$ is denoted by $\rho(\mathbf{M})$ and its spectral norm is denoted by $\norm{\mathbf{M}}_2$. The Hermitian adjoint
of $\mathbf{M}$ is denoted by $\mathbf{M}^H$.
\paragraph{Setup.}
As alluded above, we consider a \emph{time-invariant} linear dynamical  system characterized by matrices $\mathbf{A}\in \mathbb{R}^{n\times n}$, $\mathbf{B}\in \mathbb{R}^{n\times m}$, $\mathbf{C}\in \mathbb{R}^{p\times n}$ and $\mathbf{D}\in \mathbb{R}^{p\times m}$ as follows:
\begin{align}\textstyle
	&\mathbf{h}_{t+1}=\mathbf{A}\mathbf{h}_{t}+\mathbf{B}\mathbf{u}_t,\nonumber\\&
	\mathbf{y}_t=\mathbf{C}\mathbf{h}_t+\mathbf{D}\mathbf{u}_t+\boldsymbol{\zeta}_t,\nonumber
\end{align}
where $\mathbf{u}_t$ is an external control input vector at time instance $t$, $\mathbf{y}_t$ is the vector of system outputs, and $\boldsymbol{\zeta}_t$ is the noise of measurement. In the above model, the hidden state is denoted by $\mathbf{h}_t$, and $n$ is called the \emph{order of the system}. In addition, $\mathbf{A}$, $\mathbf{B}$, $\mathbf{C}$ and $\mathbf{D}$ are unknown transformation matrices. We assume that the system is stable, and thus, $\rho(\mathbf{A})<1$. Furthermore, we assume matrices $\mathbf{A}$, $\mathbf{B}$, $\mathbf{C}$ and $\mathbf{D}$ have bounded Frobenius norms. Based on one sequence of input-output pairs $\{\mathbf{u}_t, \mathbf{y}_t\}_{t=1}^N$, and assuming $n$ is known (similar to \citealp{oymak2019non,hardt2018gradient,oymak2019stochastic}), we aim to learn the unknown matrices and characterize complexity bounds for the accuracy of the estimated unknowns.

Consider that $T$ is a finite time horizon. Each $y_t, t\geq T-1$, can be expanded recursively using $\mathbf{u}_t,\dots,\mathbf{u}_{t-T+1}$ and $\mathbf{h}_{t-T+1}$ as follows:
\begin{align}\textstyle
	\mathbf{y}_t=\sum_{i=1}^{T-1}\mathbf{C}\mathbf{A}^{i-1}\mathbf{B}\mathbf{u}_{t-i}+\mathbf{D}\mathbf{u}_t+\boldsymbol{\zeta}_t+\mathbf{C}\mathbf{A}^{T-1}\mathbf{h}_{t-T+1},\label{eq:expand}
\end{align}
when $\mathbf{A}\mathbf{h}_{t-1}+\mathbf{B}\mathbf{u}_{t-1}$ is substituted for each $\mathbf{h}_{t}, t\in \{t,\dots,t-T+2\}$.
Suppose that the input signal $\mathbf{u}_t$ at each time instance is random and follows a normal distribution  $\mathcal{N}(\mathbf{0},\boldsymbol{\Sigma}_{\boldsymbol{\sigma}^{\cdot 2}})$, where $\boldsymbol{\Sigma}_{\boldsymbol{\sigma}^{\cdot 2}}$ is the covariance matrix. Furthermore, $\boldsymbol{\zeta}_t$ also follows a normal distribution $\mathcal{N}(\mathbf{0},\boldsymbol{\Sigma}_{\boldsymbol{\sigma}_\zeta^{\cdot 2}})$ and $\{\boldsymbol{\zeta}_t\}_t$ is independent of $\{\mathbf{u}_t\}_t$. Let $\mathbf{x}_t\in \mathbb{R}^{mT\times 1}$ denote a finite sequence of inputs with length $T$ as follows:
\begin{align}\textstyle
	\mathbf{x}_t=
	\begin{cases}
		[\mathbf{u}_t'\:\:\:\mathbf{u}_{t-1}'\:\:\: \mathbf{u}_{t-2}'\:\:\: \dots \:\:\:\mathbf{u}_{1}'\:\:\:\:\mathbf{0}\:\:\:\:\dots\:\:\:\:\mathbf{0}]', \hspace{2.5cm} \text{if}\:\:\:\:t< T,\\
		[\mathbf{u}_t'\:\:\:\mathbf{u}_{t-1}'\:\:\: \mathbf{u}_{t-2}'\:\:\: \dots \:\:\:\mathbf{u}_{t-T+1}']', \hspace{3.75cm}\text{if}\:\:\:\:t\geq T.\label{eq:x}
	\end{cases}
\end{align}
Using $\mathbf{x}_t$, we rewrite \eqref{eq:expand} as follows:
\begin{align}\textstyle
	&\mathbf{y}_t= \underbrace{[\mathbf{D}\:\:\:\:\mathbf{C}\mathbf{B}\:\:\:\:\mathbf{C}\mathbf{A}\mathbf{B}\:\:\:\:\dots \:\:\:\:\mathbf{C}\mathbf{A}^{T-2}\mathbf{B}]}_{\boldsymbol{\Theta}_T}\mathbf{x}_t+\mathbf{C}\mathbf{A}^{T-1}\mathbf{h}_{t-T+1}+\boldsymbol{\zeta}_{t}\nonumber.
\end{align} 
In the above equation, $\mathbf{h}_{t-T+1}$ is a linear combination of inputs and the initial state. In the following lemma, we bound the Frobenius norm of $\mathbf{C}\mathbf{A}^{T-1}\mathbf{h}_{t-T+1}$. 
\begin{lemma}\label{le:trun}
	Suppose that $\mathbf{A}=\mathbf{V}\boldsymbol{\Lambda}\mathbf{V}^{-1}$ is the eigenvalue decomposition for $\mathbf{A}$. We bound the norm of $\mathbf{C}\mathbf{A}^{T-1}\mathbf{h}_{t-T+1}$ when $t\geq T$ as follows:
	\begin{align}
		&\mathbb{E}_{\mathbf{u}}[\norm{\mathbf{C}\mathbf{A}^{T-1}\mathbf{h}_{t-T+1}}_2^2] \leq n^2\ell\norm{\mathbf{C}}_F^2\rho(\mathbf{A})^{2(T-1)}\nonumber\\
		&\times\left[n^2 \ell \rho(\mathbf{A})^2\norm{\mathbf{h}_0}_2^2+\frac{n^2\ell m\max(\boldsymbol{\sigma}^{\cdot 2})\rho(\mathbf{A})^2\norm{\mathbf{B}}_F^2}{1-\rho(\mathbf{A})^{2}}+m\max(\boldsymbol{\sigma}^{\cdot 2})\norm{\mathbf{B}}_F^2\right],\label{eq:hidden}
	\end{align}
	where $\ell=\norm{\mathbf{V}^{-1}}_F$.
\end{lemma}
From Lemma \ref{le:trun}, one can observe that the resulted error from truncation, $\norm{\mathbf{C}\mathbf{A}^{T-1}\mathbf{h}_{t-T+1}}_2^2$, decreases exponentially with the truncation length $T$. Thus, the error becomes very small for a large enough $T$. To reconstruct the system output $\mathbf{y}_t$, it is essentially enough to identify 
\begin{align}
	\boldsymbol{\Theta}_T=[\mathbf{D}\:\:\:\: \mathbf{C}\mathbf{B} \:\:\:\:\mathbf{C}\mathbf{A}\mathbf{B} \:\:\:\: \mathbf{C}\mathbf{A}^2\mathbf{B} \:\:\:\: \dots\:\:\:\: \mathbf{C}\mathbf{A}^{T-2}\mathbf{B}],\nonumber
\end{align}
where the size of the above unknown matrix is  $p\times m\:T$. 
We notice that $\boldsymbol{\Theta}_T$ incorporates the first $T$ Markov parameters; the first one is $\mathbf{D}$ and the rest are $\{\mathbf{C}\mathbf{A}^{i}\mathbf{B}\}_{i=0}^{T-2}$. To estimate $\boldsymbol{\Theta}_T$, we use a regression approach and formulate the following optimization:
\begin{align}
	\hat{\boldsymbol{\Theta}}_T=\arg\min_{\hat{\boldsymbol{\Theta}}_T}\lim_{N\rightarrow\infty}\frac{1}{2N}\sum_{t=1}^N\norm{\mathbf{y}_t- \hat{\boldsymbol{\Theta}}_T\mathbf{x}_t}_2^2.\label{opt:reg}
\end{align}
The above problem is strongly convex in $\hat{\boldsymbol{\Theta}}_T$ since the Hessian matrix (or the covariance of the inputs) $\lim_{N\rightarrow\infty}\frac{1}{N}\sum_{t=1}^N\mathbb{E}(\mathbf{x}_t\mathbf{x}_t')$ is positive definite. This means that a unique solution is attained from the above minimization problem. 
\begin{proposition}\label{pro:noise}
	The solution $\hat{\boldsymbol{\Theta}}_T$ from \eqref{opt:reg} is identical to the ground truth $\boldsymbol{\Theta}_T$  in spite of the noise of measurements and excluding the hidden state transformation $\mathbf{C}\mathbf{A}^{T-1}\mathbf{h}_{t-T+1}$ in \eqref{opt:reg}.
\end{proposition}
\begin{remark}\label{le:noise}
	The Markov parameters of the system can be learned from \eqref{opt:reg} if the process noise is considered in \eqref{eq:state}.
\end{remark}
\section{Regression Approach and Convergence Analysis}
This section is concerned with solving \eqref{opt:reg}. Overall, it is difficult to solve since an infinite sum of squared Frobenius norms are to be minimized. In practice, it is impossible to solve, as one cannot wait for  an infinite number of input-output pairs. We solve the following problem instead:
\begin{align}
	\hat{\boldsymbol{\Theta}}_T=\arg\min_{\hat{\boldsymbol{\Theta}}_T}\frac{1}{2N}\sum_{t=1}^{N}\norm{\mathbf{y}_t- \hat{\boldsymbol{\Theta}}_T\mathbf{x}_t}_2^2.\label{opt:lim}
\end{align}
Based on the finite collected input-output pairs, we estimate $\boldsymbol{\Theta}_T$. 
Due to the strong convexity of \eqref{opt:lim} when $N\geq 2T$ (i.e., $\frac{1}{N}\sum_{t=1}^{N}\mathbb{E}(\mathbf{x}_t\mathbf{x}_t') 	\succ \mathbf{0}$), increasing the number of samples $N$ strictly decreases the Frobenius norm distance between the minimizer of \eqref{opt:lim} and $\boldsymbol{\Theta}_T$. In the following theorem, we characterize the maximum Frobenius norm distance between $\boldsymbol{\Theta}_T$ and the minimizer of \eqref{opt:lim} as a function of $N$, the truncation length $T$, the covariance of inputs, and the measurement noise level.
\begin{theorem}\label{th:prob}
	For any given $N\geq 2T$, the maximum Frobenius norm distance between the first-order stationary solution to \eqref{opt:lim} and $\boldsymbol{\Theta}_T$ is upper-bounded as follows:
	\begin{align}
		&\mathbb{E}_{\mathbf{u}}\left[\mathbb{E}_{\boldsymbol{\zeta}}\left[\norm{\hat{\boldsymbol{\Theta}}_T-\boldsymbol{\Theta}_T}_F^2\right]\right]\leq\frac{n^2\ell \norm{\mathbf{C}}_F^2\rho(\mathbf{A})^{2(T-1)}m^3T^2(\max(\boldsymbol{\sigma}^{\cdot 2}))^3\norm{\mathbf{B}}_F^2 }{(N-T+1)\left(\min(\boldsymbol{\sigma}^{\cdot 2})\right)^2}  \nonumber\\
		&+ \frac{p m^2T^2 \max(\boldsymbol{\sigma}_\zeta^{\cdot 2})\max(\boldsymbol{\sigma}^{\cdot 2})}{(N-T+1)\left(\min(\boldsymbol{\sigma}^{\cdot 2})\right)^2} +\frac{n^4\ell^2 m^2T^2\rho(\mathbf{A})^{2T}\norm{\mathbf{C}}_F^2\left(\max(\boldsymbol{\sigma}^{\cdot 2})\right)^2\iota }{(N-T+1)\left(\min(\boldsymbol{\sigma}^{\cdot 2})\right)^2}=\chi_N^2,\label{eq:ub}
	\end{align} \textit{}
	where $\iota=\norm{\mathbf{h}_0}_2^2+\frac{ m\max(\boldsymbol{\sigma}^{\cdot 2})\norm{\mathbf{B}}_F^2}{1-\rho(\mathbf{A})^{2}}$.
\end{theorem}
Based on the above theorem, increasing the trajectory length $N$ drives the solution of \eqref{opt:lim} closer to $\boldsymbol{\Theta}_T$. Although the Frobenius norm distance between the solution of \eqref{opt:lim} and the ground truth strictly decreases with $N$, solving \eqref{opt:lim} globally by the pseudo-inverse method (e.g., \cite{oymak2019non,tsiamis2019finite,zheng2020non}), second-order methods (e.g., log barrier), and gradient descent methods are costly and challenging. The reason is that when $N$ and $T$ are large numbers, the calculation and inversion of $\frac{1}{N}\sum_{t=1}^{N}\mathbf{x}_t\mathbf{x}_t'$, which is $mT\times mT$ dimensional becomes expensive.  Therefore, a computationally faster and more cost-efficient approach is desired.
\subsection{Offline SGD}
To alleviate the computational cost of solving \eqref{opt:lim}, we propose a low iteration cost SGD algorithm, which works based on a fixed batch of input-output pairs. Since this algorithm uses a fixed batch size, we name it offline SGD. The $\tau^{\text{th}}$ iteration of the offline SGD is described in the following step:
\begin{align}
	\hat{\boldsymbol{\Theta}}_{\tau,T}=\hat{\boldsymbol{\Theta}}_{\tau-1,T}-\eta(\hat{\boldsymbol{\Theta}}_{\tau-1,T}\mathbf{x}_t-\mathbf{y}_t)\mathbf{x}_t',\label{eq:offlineup}
\end{align}
where $\eta$ is a constant learning rate and $t\in\{T,T+1,\dots,N\}$ is chosen with probability $\frac{1}{N-T+1}$. When we use the offline SGD instead of the traditional gradient descent to solve \eqref{opt:lim} , the complexity
\begin{minipage}{0.37\textwidth}
	of solving the problem in each iteration reduces from $\mathcal{O}(NpmT)$ to  $\mathcal{O}(pmT)$,	  which is a significant improvement if $N$ is large. The offline SGD is summarized in Algorithm \ref{al:sgdoffline}. In the following theorem, we bound the maximum expected distance between the offline SGD iterate $\hat{\boldsymbol{\Theta}}_{\tau,T}$ and $\boldsymbol{\Theta}_{T}$ as a function of the number of iterations, $T$, the covariance
\end{minipage}%
\hfill 	
\scalebox{1}{
	\begin{minipage}{0.6\textwidth}	\vspace{-0.1cm}
		\begin{algorithm}[H]
			\SetAlgoLined
			\textbf{Initialization}: Assign small value to $\hat{\boldsymbol{\Theta}}_{0,T}$\\
			\textbf{Input}: $\{\mathbf{x}_t,\mathbf{y}_t\}_{t=1}^N$, learning rate $\eta$\\
			\textbf{Output}: Estimation of $\boldsymbol{\Theta}_T$\\
			\For{$\tau$ from $1$ to $\text{END}$}{
				Uniformly at random choose $t\in\{T,T+1,\dots,N\}$\\
				$\hat{\boldsymbol{\Theta}}_{\tau,T}=\hat{\boldsymbol{\Theta}}_{\tau-1,T}-\eta(\hat{\boldsymbol{\Theta}}_{\tau-1,T}\mathbf{x}_t-\mathbf{y}_t)\mathbf{x}_t'$
			}
			\caption{Offline SGD algorithm to learn $\boldsymbol{\Theta}_T$}\label{al:sgdoffline}
		\end{algorithm} 
\end{minipage}}\vspace{0.1cm}
 of inputs, $N$, and noise levels. 
\begin{theorem}\label{th:offline}
	Let $\boldsymbol{\phi}_\tau$ denote the difference between $\hat{\boldsymbol{\Theta}}_{\tau,T}$ (in the $\tau^{\text{th}}$ iteration) and ground truth $\boldsymbol{\Theta}_{T}$  as $\boldsymbol{\phi}_\tau=\hat{\boldsymbol{\Theta}}_{\tau,T}-\boldsymbol{\Theta}_{T}$, and $\boldsymbol{w}_0=\hat{\boldsymbol{\Theta}}_{0,T}-\hat{\boldsymbol{\Theta}}_T$. Consider that the offline SGD minimizes \eqref{opt:lim} with a batch of size $N\geq 2T$, where each iteration is implemented based on \eqref{eq:offlineup} with $\eta \leq \frac{1}{m\:T\max(\boldsymbol{\sigma}^{\cdot 2})}$. Then, $\mathbb{E}_{\mathbf{u}}[\mathbb{E}_{\boldsymbol{\zeta}}[{\norm{\boldsymbol{\phi}_{\tau}}_F^2}]]$ in the $\tau^{\text{th}}$ iteration of the offline SGD can be upper-bounded as follows:
	\begin{align}
		&\mathbb{E}_{\mathbf{u}}\left[\mathbb{E}_{\boldsymbol{\zeta}}[{\norm{\boldsymbol{\phi}_{\tau}}_F^2}]\right]<\norm{\boldsymbol{\omega}_0}_F^2\left(1-2\eta m\:T\min(\boldsymbol{\sigma}^{\cdot 2})+2\eta^2m^2\:T^2\min(\boldsymbol{\sigma}^{\cdot 2})\max(\boldsymbol{\sigma}^{\cdot 2})\right)^\tau\hspace{-.2cm}+\Delta_N+\chi_N^2,\label{eq:samon}
	\end{align}
	where $	\Delta_N$ is given in \eqref{eq:mash}, which depends on different problem parameters such as $\chi_N$, $\eta$, and $T$. 
\end{theorem} 
\begin{corollary}
	The cost function in \eqref{opt:lim} is in expectation $mT\min(\boldsymbol{\sigma}^{\cdot 2})$-strongly convex and the Lipschitz constant for its gradient is $m\:T\max(\boldsymbol{\sigma}^{\cdot 2})$. When the step-size is $\eta = \frac{1}{2m\:T\max(\boldsymbol{\sigma}^{\cdot 2})}$, the fastest convergence rate is obtained. This convergence rate is equal to that given in the state of the art method \cite[Theorem 3.1]{gower2019sgd}. Compared to \cite[Appendix A]{oymak2019stochastic},	\eqref{eq:samon} is tighter since in \eqref{eq:samon} the third term in parenthesis depends linearly on the Lipschitz constant, while the dependence is quadratic in \cite[Appendix A]{oymak2019stochastic}. The dependence of the SGD error bound on the batch size and the truncation length is not studied in any of the aforementioned papers.
\end{corollary}
Theorem \ref{th:offline} states that Algorithm \ref{al:sgdoffline} linearly converges up to the sum of two additive constant terms, which are calibrated by $\boldsymbol{\sigma}_\zeta^{\cdot 2}$, $\mathbf{h}_0$, $T$, $\eta$,  $\boldsymbol{\sigma}^{\cdot 2}$, and the batch size $N$ as given in \eqref{eq:mash}. With a small enough $\Delta_N$, Algorithm \ref{al:sgdoffline} linearly converges to a region with a maximum distance of $\mathcal{O}(\frac{1}{\sqrt{N}})$ to the ground truth Markov parameters since $\chi_N$ decreases with rate $\mathcal{O}(\frac{1}{\sqrt{N}})$. We observe from \eqref{eq:mash} that one can make the two additive terms as small as desired by increasing $N$, which decreases $\chi_N^2$, and picking a smaller $\eta$, which slows down the convergence rate of Algorithm \ref{al:sgdoffline}. 
One drawback of full-batch methods in \cite{oymak2019stochastic,hardt2018gradient,sattar2020non} is that they simultaneously require all the samples to be stored and processed. Although Algorithm \ref{al:sgdoffline} decreases the cost of computation by utilizing one input-output pair in each iteration, it requires all samples to be stored. We proved in Theorem \ref{th:prob} that increasing the batch size helps to reach  a closer neighborhood of the ground truth solution. However, storing a large batch of input-output pairs is challenging and storage inefficient. Therefore, we propose an online SGD that does not require samples to be stored.
\subsection{Online SGD}
\begin{minipage}{0.45\textwidth}
	Motivated by Theorem \ref{th:prob}, we propose an algorithm which utilizes  newly arrived samples and discards the old ones.  We develop an SGD algorithm to learn $\boldsymbol{\Theta}_T$ in an online streaming fashion. The online SGD algorithm implements the descent on the loss function \eqref{opt:lim} in each iteration using a gradient obtained from the most recent input-output pair at time instance $t$ as follows: 
\end{minipage}\hfil
\begin{minipage}{0.55\textwidth}
	\vspace{-0.4cm}
		\begin{algorithm}[H]
		\SetAlgoLined
		\textbf{Initialization}: Assign small value to $\hat{\boldsymbol{\Theta}}_{0,T}$, $t=1$\\
		\textbf{Input}: $\{\mathbf{x}_i,\mathbf{y}_i\}_{i=1}^t$, learning rate $\eta$\\
		\textbf{Output}: Estimation of $\boldsymbol{\Theta}_T$\\
		\If{a new input-output pair arrives}{
			$\hat{\boldsymbol{\Theta}}_{t,T}=\hat{\boldsymbol{\Theta}}_{t-1,T}-\eta(\hat{\boldsymbol{\Theta}}_{t-1,T}\mathbf{x}_t-\mathbf{y}_t)\mathbf{x}_t'$\\
			$t=t+1$
		}
		\caption{Online SGD algorithm to learn $\boldsymbol{\Theta}_T$}\label{al:sgd}
	\end{algorithm} 
\end{minipage}
\begin{align}
	\hat{\boldsymbol{\Theta}}_{t,T}=\hat{\boldsymbol{\Theta}}_{t-1,T}-\eta(\hat{\boldsymbol{\Theta}}_{t-1,T}\mathbf{x}_t-\mathbf{y}_t)\mathbf{x}_t'.\label{eq:onlineup}
\end{align}
In each time instance $t$, one iteration is implemented. The proposed online SGD algorithm is summarized in Algorithm \ref{al:sgd}; we provide a corresponding convergence guarantee below.
\begin{theorem}\label{th:one}
	Consider that the online SGD minimizes \eqref{opt:lim}, where each iteration is implemented based on \eqref{eq:onlineup} with $\eta \leq \frac{1}{m\:T\max(\boldsymbol{\sigma}^{\cdot 2})}$. The maximum expected Frobenius norm distance $\mathbb{E}_{\mathbf{u}}[\mathbb{E}_{\boldsymbol{\zeta}}[{\norm{\boldsymbol{\phi}_{t}}_F^2}]]$ in $t^{\text{th}}$ iteration of the proposed online SGD when $\boldsymbol{\phi}_t=\hat{\boldsymbol{\Theta}}_{t,T}-\boldsymbol{\Theta}_{T}$ and $t\geq 2T$ is upper-bounded as:
	\begin{align}
		&\mathbb{E}_{\mathbf{u}}\left[\mathbb{E}_{\boldsymbol{\zeta}}\left[\norm{\boldsymbol{\phi}_t}_F^2\right]\right]
		<\norm{\boldsymbol{\omega}_0}_F^2\left(1-2\eta m\:T\min(\boldsymbol{\sigma}^{\cdot 2})+2\eta^2m^2\:T^2\min(\boldsymbol{\sigma}^{\cdot 2})\max(\boldsymbol{\sigma}^{\cdot 2})\right)^t\hspace{-.1cm}+\Delta_t+\chi_t^2,\label{eq:samoff}
	\end{align}
	where $	\Delta_t$ is $	\Delta_N|_{N=t}$  depends on different problem parameters such as $\chi_t$, $\eta$, and $T$. 
\end{theorem}

\section{Transfer Function Estimation and Recovery of Weight Matrices}
The transfer function of a linear dynamical system is obtained by taking $z$-transformation of the impulse response of the system and is computed as follows \cite[p. 267--p. 268]{luenberger1979introduction}:
\begin{align}
	\mathbf{G}(z)=\sum_{t=1}^{\infty} z^{-t}\mathbf{C}\mathbf{A}^{t-1}\mathbf{B}+\mathbf{D}=\mathbf{C}(z\mathbf{I}_{n\times n}-\mathbf{A})^{-1}\mathbf{B}+\mathbf{D}.\nonumber
\end{align}
We can rewrite the above transfer function as follows:
\begin{eqnarray}
	\mathbf{G}(z)=	\sum_{t=1}^{T-1} z^{-t}\mathbf{C}\mathbf{A}^{t-1}\mathbf{B}+\mathbf{D}+\mathbf{E}_{z,T},\label{eq:transf}
\end{eqnarray}
where $\mathbf{E}_{z,T}=\sum_{t=T}^{\infty} z^{-t}\mathbf{C}\mathbf{A}^{t-1}\mathbf{B}$. Given a large enough $T$, the Frobenius norm of $\mathbf{E}_{z,T}$ becomes close to zero as shown in the following lemma.
\begin{lemma}\label{le:te}
	The truncation error in computing the transfer function is upper-bounded as follows:
	\begin{align}
		\norm{\mathbf{E}_{z,T}}_F^2=\norm{\sum_{t=T}^\infty z^{-t}\mathbf{C}\mathbf{A}^{t-1}\mathbf{B}}_F^2\leq \frac{n^2\ell \norm{\mathbf{C}}_F^2\norm{\mathbf{B}}_F^2\rho(\mathbf{A})^{2(T-1)}}{1-\rho(\mathbf{A})^2|z|^{-2}}\label{eq:error}.
	\end{align}
	Give a large $T$, the RHS of \eqref{eq:error} tends to zero and the LHS is enforced to be very small. 
\end{lemma}
When $\norm{\mathbf{E}_{z,T}}_F^2$ is small enough, we can efficiently approximate $\mathbf{G}(z)$ using $T$ Markov parameters: $\{\mathbf{C}\mathbf{A}^{t-1}\mathbf{B}\}_{t=1}^{T-1}$ and $\mathbf{D}$, which are learned by Algorithms \ref{al:sgdoffline} and \ref{al:sgd}.
Upon the convergence of $\hat{\boldsymbol{\Theta}}_{t,T}$ (or $\hat{\boldsymbol{\Theta}}_{\tau,T}$), the first Markov parameter, $\mathbf{D}$, is learned and needs no further processing.
To recover $\mathbf{A}$, $\mathbf{B}$  and $\mathbf{C}$ from the estimated transfer function, we assume $\mathbf{A}$, $\mathbf{B}$ and $\mathbf{C}$ have Brunovsky canonical form \cite{brunovsky1970classification},  which is perhaps the most popular canonical form \cite{martin1978applications,hazewinkel1983representations,liu2008global}. In Brunovsky canonical form, we have:
\begin{align}\textstyle
	\mathbf{A}=
	\begin{bmatrix}
		\mathbf{0} & \mathbf{I}_{m\times m} & \mathbf{0} & \cdots & \mathbf{0} \\
		\mathbf{0} & \mathbf{0} &\mathbf{I}_{m\times m} & \cdots & \mathbf{0} \\
		\vdots  & \vdots  & \vdots & \ddots & \vdots\\
		\mathbf{0} & \mathbf{0} &\mathbf{0} & \cdots & \mathbf{I}_{m\times m} \\
		-a_n \mathbf{I}_{m\times m}& -a_{n-1} \mathbf{I}_{m\times m} &-a_{n-2} \mathbf{I}_{m\times m} & \cdots & -a_1 \mathbf{I}_{m\times m} 
	\end{bmatrix}, \hspace{.5cm}\mathbf{B}=
	\begin{bmatrix}
		\mathbf{0}  \\
		\vdots \\
		\mathbf{0}  \\
		\mathbf{I}_{m\times m}
	\end{bmatrix},\label{eq:ab}
\end{align}
where $\mathbf{A}\in \mathbb{R}^{nm\times nm}$, $\mathbf{B}\in \mathbb{R}^{nm\times m}$, and $\mathbf{C}\in \mathbb{R}^{p\times nm}$. 
To recover matrix $\mathbf{A}$, it is enough to find $\{a_i\}_{i=1}^n$. To recover $\mathbf{C}$, all elements should be estimated. In Brunovsky canonical form, $\mathbf{B}$ is known as given in \eqref{eq:ab}. The above special forms for $\mathbf{A}$ and $\mathbf{B}$ matrices help to find unknowns.
If $\mathbf{A}$ and $\mathbf{B}$ are in Brunovsky canonical form, $\mathbf{G}(z)$ is obtained as follows:
\begin{align}
	\mathbf{G}(z)=\mathbf{C}\mathbf{S}(z)+\mathbf{D},\label{eq:trans1}
\end{align}
where $\mathbf{S}(z)=(z\mathbf{I}_{nm\times nm}-\mathbf{A})^{-1}\mathbf{B}$ and can be rewritten as follows \cite[Lemma B.1]{hardt2018gradient}:
\begin{align}
	\mathbf{S}(z)=\frac{1}{z^n+a_1z^{n-1}+\dots+a_n}
	&\underbrace{\begin{bmatrix}
			\mathbf{I}_{m\times m} \:\:\:
			z\mathbf{I}_{m\times m} \:\:\:
			\dots  \:\:\:
			z^{n-1}\mathbf{I}_{m\times m} 
		\end{bmatrix}'}_{\mathbf{W}},\nonumber
\end{align}
where $\mathbf{W}\in\mathbb{C}^{nm\times m}$. The denominator of $\mathbf{S}(s)$ is called \emph{characteristic polynomial} and is denoted by $q(z)$. If $\mathbf{A}$, $\mathbf{B}$ and $\mathbf{C}$ have Brunovsky canonical form, the transfer function \eqref{eq:trans1} is uniquely realized by the state-space representation \cite{hardt2018gradient}. When the linear dynamical system is SISO, i.e., $m=p=1$, Brunovsky canonical form reduces to the \emph{controllable canonical form}. We match \eqref{eq:transf} and \eqref{eq:trans1} as follows:
\begin{align}
	\sum_{t=1}^{T-1} z^{-t}\mathbf{C}\mathbf{A}^{t-1}\mathbf{B}=\mathbf{C}\mathbf{S}(z)-\mathbf{E}_{z,T}.\label{eq:equality}
\end{align}
The LHS of the above equation can be efficiently estimated using $\hat{\boldsymbol{\Theta}}_{t,T}$ (or $\hat{\boldsymbol{\Theta}}_{\tau,T}$) from the regression problem. In Brunovsky canonical form, there are $n$ and $pnm$ unknown elements in $\mathbf{A}$ and $\mathbf{C}$, respectively. We need at least $n+pnm$ equations to identify unknowns. To find $n+pnm$ equations, we match both sides of \eqref{eq:equality} in $n+pnm$ complex frequencies. In particular, we choose $z$ such that it does not yield $\norm{\sum_{t=1}^{T-1} z^{-t}\mathbf{C}\mathbf{A}^{t-1}\mathbf{B}}_F=0$ or make it unbounded. For example, one can choose frequencies on the unit circle $z_k=e^{j\frac{\pi(k-1)}{n+pnm}}$, $k\in\{1,\dots,n+pnm\}$ if none of them is a pole or zero of $\sum_{t=1}^{T-1} z^{-t}\mathbf{C}\mathbf{A}^{t-1}\mathbf{B}$. 
By choosing $|z_k|=1$, one can avoid the linear system of equations built using \eqref{eq:equality} from becoming ill-conditioned. When $|z_k|\neq1$ and $n$ is large, $\{z_k^{n-i}\}_{i=1}^n$, which are  coefficients of $\{a_i\}_{i=1}^n$, become very different in terms of their absolute value, and the linear system of equations becomes ill-conditioned. Each side of \eqref{eq:equality} is a $p\times m$ matrix and yields $pm$ equations in each frequency. Therefore, having $n+pnm$ frequencies yields an over-determined consistent system. To represent the LHS of \eqref{eq:equality} in a compact form, we define $\boldsymbol{\vartheta}$ as follows:
\begin{align}
	\boldsymbol{\vartheta}=\begin{bmatrix}
		\mathbf{0}_{m\times m} \hspace{.3cm}
		z^{-1}\mathbf{I}_{m\times m}\hspace{.3cm}
		z^{-2}\mathbf{I}_{m\times m} \hspace{.3cm}
		\dots  \hspace{.3cm}
		z^{-T+1}\mathbf{I}_{m\times m} 
	\end{bmatrix}'.\nonumber
\end{align}
Suppose $\boldsymbol{\vartheta}_k=\boldsymbol{\vartheta}|_{z=z_k}$ and $\mathbf{W}_k=\mathbf{W}|_{z=z_k}$. The linear system of equations is obtained as follows:
\begin{align}
	\hat{\boldsymbol{\Theta}}_{t,T}\:\boldsymbol{\vartheta}_kq(z_k)=\mathbf{C}\mathbf{W}_k-q(z_k)\mathbf{E}_{z_k,T}, \hspace{1cm}\forall k\in \{1,\dots,n+pnm\},\label{eq:linsys}
\end{align}
where $\hat{\boldsymbol{\Theta}}_{t,T}\:\boldsymbol{\vartheta}_k$ is numerically computed by Algorithm \ref{al:sgd}, and $\mathbf{E}_{z_k,T}$ is treated as noise when $\norm{\mathbf{E}_{z_k,T}}_F$ is small enough. The unknowns are embedded in $q(z)$ and $\mathbf{C}$. 
In general, solving a linear system is easier when compared to the SVD-based methods in \cite{oymak2019non,sarkar2019finite,tsiamis2019finite}. In Appendix \ref{sec:lineq}, we guarantee a unique solution for \eqref{eq:linsys}. Let the vector of unknowns be denoted by $\boldsymbol{\varrho}=\{\{a_i\}_{i=1}^n,\{c_{i,j}\}_{i=1:p,j=1:mn}\}$. Then, one can rewrite \eqref{eq:linsys} in the standard form of linear system of equations easily, as explained in Appendix \ref{sec:lineq}, as follows:
	\begin{align}
		\boldsymbol{\Gamma}_{t,T}\boldsymbol{\varrho}_t=\boldsymbol{\varkappa}_t,\label{eq:lin1}
	\end{align}
	where $\boldsymbol{\Gamma}_{t,T}$ and $\boldsymbol{\varkappa}_t$ are calculated using $\hat{\boldsymbol{\Theta}}_{t,T}$. The above equation can be solved either by the pseudo-inverse method or iterative methods, e.g., \cite{razaviyayn2019linearly,liu2016accelerated,ma2015convergence}. Consider in each iteration of Algorithm \ref{al:sgd}, we solve  \eqref{eq:lin1} by the pseudo-inverse method as given in Algorithm \ref{al:comb}.  
	Theorem \ref{th:linsys} ensures the linear convergence  of the vector of unknowns  returned by Algorithm \ref{al:comb} upon the convergence of  $\hat{\boldsymbol{\Theta}}_{t,T}$.

		\begin{algorithm}[H]
			\SetAlgoLined
			\textbf{Input}: $\{\mathbf{x}_i,\mathbf{y}_i\}_{i=1}^t$, learning rate $\eta$, $t=1$\\
			\textbf{Output}: Estimation of $\mathbf{A}$, $\mathbf{C}$ and $\mathbf{D}$\\
			\If{a new input-output pair arrives}{
				$\hat{\boldsymbol{\Theta}}_{t,T}=\hat{\boldsymbol{\Theta}}_{t-1,T}-\eta(\hat{\boldsymbol{\Theta}}_{t-1,T}\mathbf{x}_t-\mathbf{y}_t)\mathbf{x}_t'$\\
				Find $\boldsymbol{\Gamma}_{t,T}$ and $\boldsymbol{\varkappa}_t$\\
				$\hat{\boldsymbol{\varrho}}_t=(\boldsymbol{\Gamma}_{t,T}^H\boldsymbol{\Gamma}_{t,T})^{-1}\boldsymbol{\Gamma}^H_{t,T}\boldsymbol{\varkappa}_t$\\
				$t=t+1$
			}
			\caption{Online SGD combined with the linear system}\label{al:comb}
		\end{algorithm}

\begin{theorem}\label{th:linsys}
	Suppose that $\hat{\boldsymbol{\Theta}}_{t,T}$, which is the $t^{\text{th}}$ iterate of Algorithm \ref{al:comb}  with $\eta \leq \frac{1}{m\:T\max(\boldsymbol{\sigma}^{\cdot 2})}$, is used to find $\hat{\boldsymbol{\varrho}}_t$ from \eqref{eq:lin1}, where $|z_k|=1$. Then, $\hat{\boldsymbol{\varrho}}_t$ iterates satisfy 
	\begin{align}
		&\mathbb{E}_{\mathbf{u}}\left[\mathbb{E}_{\boldsymbol{\zeta}}\left[\norm{\hat{\boldsymbol{\varrho}}_t-\boldsymbol{\varrho}}_2^2\right]\right]< \Upsilon+l_2(l_1/n+1)\Big(nm(n+nmp)(T-1)\Big)\norm{\boldsymbol{\omega}_0}_F^2\nonumber\\
		&\times\left(1-2\eta m\:T\min(\boldsymbol{\sigma}^{\cdot 2})+2\eta^2m^2\:T^2\min(\boldsymbol{\sigma}^{\cdot 2})\max(\boldsymbol{\sigma}^{\cdot 2})\right)^t,\label{eq:shir}
	\end{align}
	where $\Upsilon,l_1,l_2>0$ (given in Appendix \ref{sec:lineq}). Based on the above inequality, the computational complexity to reach $\epsilon$-neighborhood of $\boldsymbol{\varrho}$ is $
	\mathcal{O}\left(\frac{1}{2\eta mT\min(\boldsymbol{\sigma}^{\cdot 2})(1-\eta mT\max(\boldsymbol{\sigma}^{\cdot 2}))}\log(\frac{n^2\:m^2\:p\:T}{\epsilon})\right).$
\end{theorem}
Similar to the $\Delta_N$, $\Upsilon$ can be made as small as desired by increasing $T$ and decreasing $\eta$. Unknowns can be learned at the linear convergence rate when Algorithm \ref{al:sgdoffline} iterates $\{\hat{\boldsymbol{\Theta}}_{\tau,T}\}_\tau$ are used in \eqref{eq:linsys} instead of $\{\hat{\boldsymbol{\Theta}}_{t,T}\}_t$ (Algorithm \ref{al:comb1} in Appendix \ref{sec:batch}).
\section{Numerical Tests}	
\begin{figure}[t!]
	\centering
	\subfigure{\includegraphics[width=0.29\textwidth]{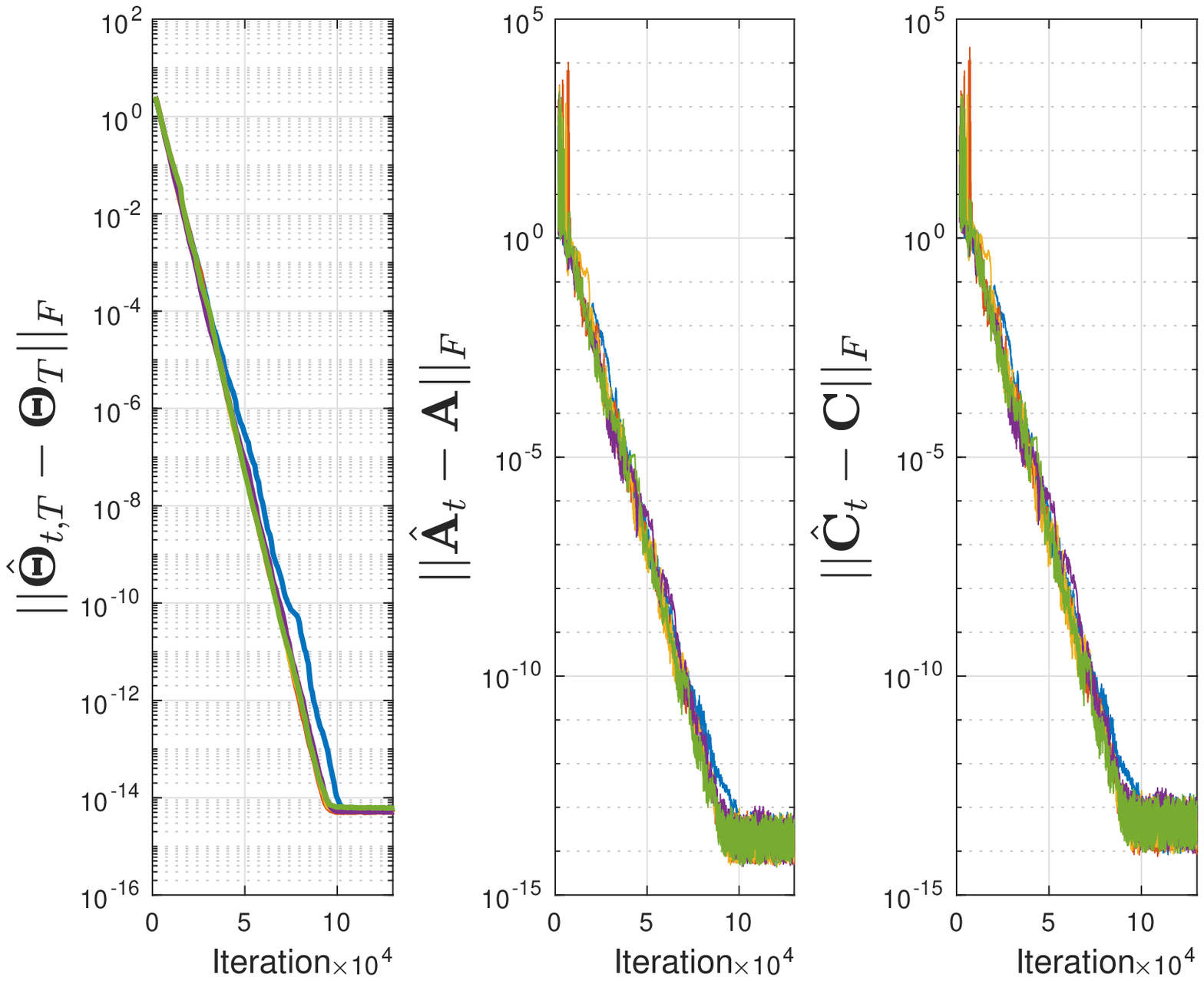}\label{fig:siso302}}
	\subfigure{\includegraphics[width=0.29\textwidth]{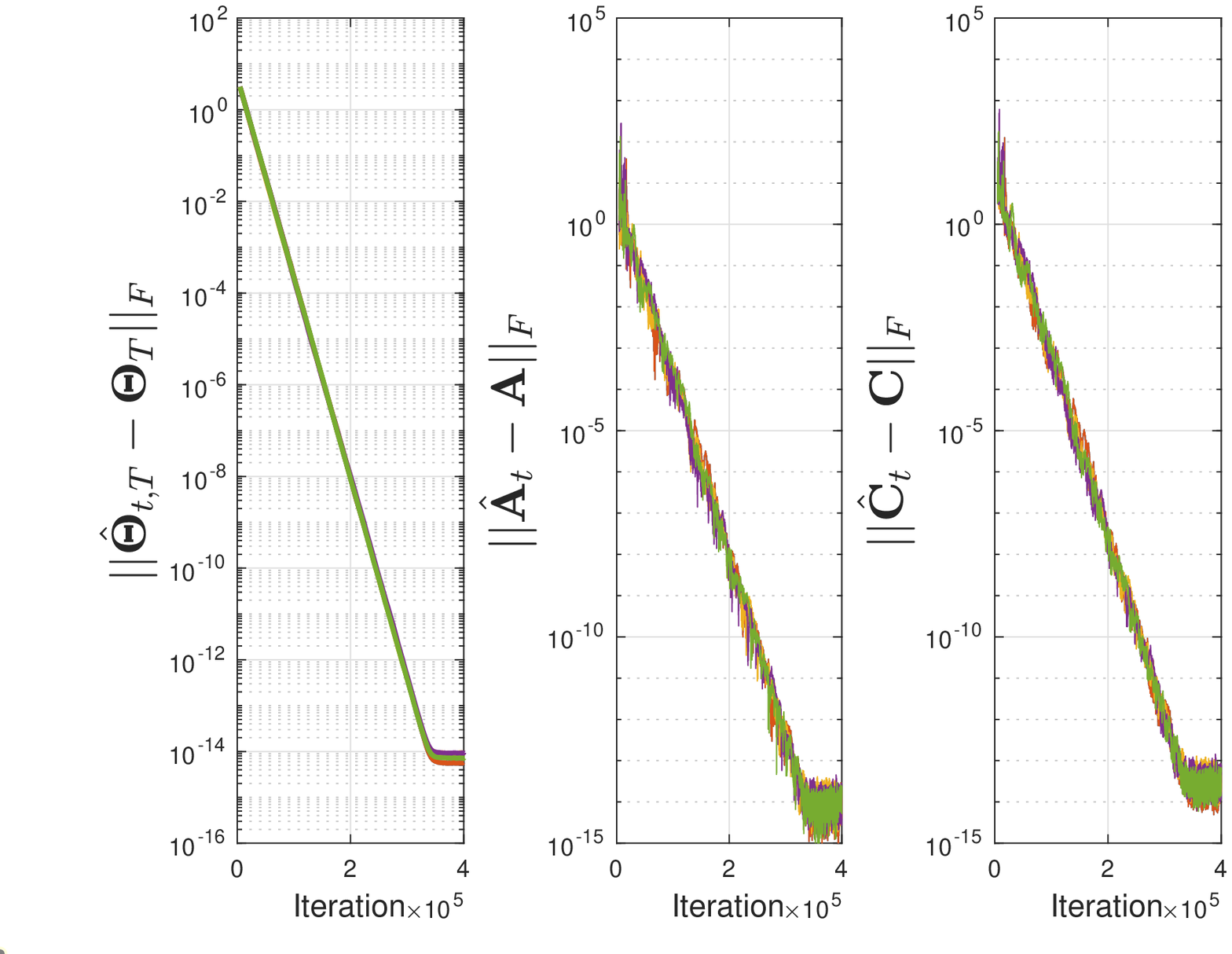}\label{fig:miso302}}
	\subfigure{\includegraphics[width=0.29\textwidth]{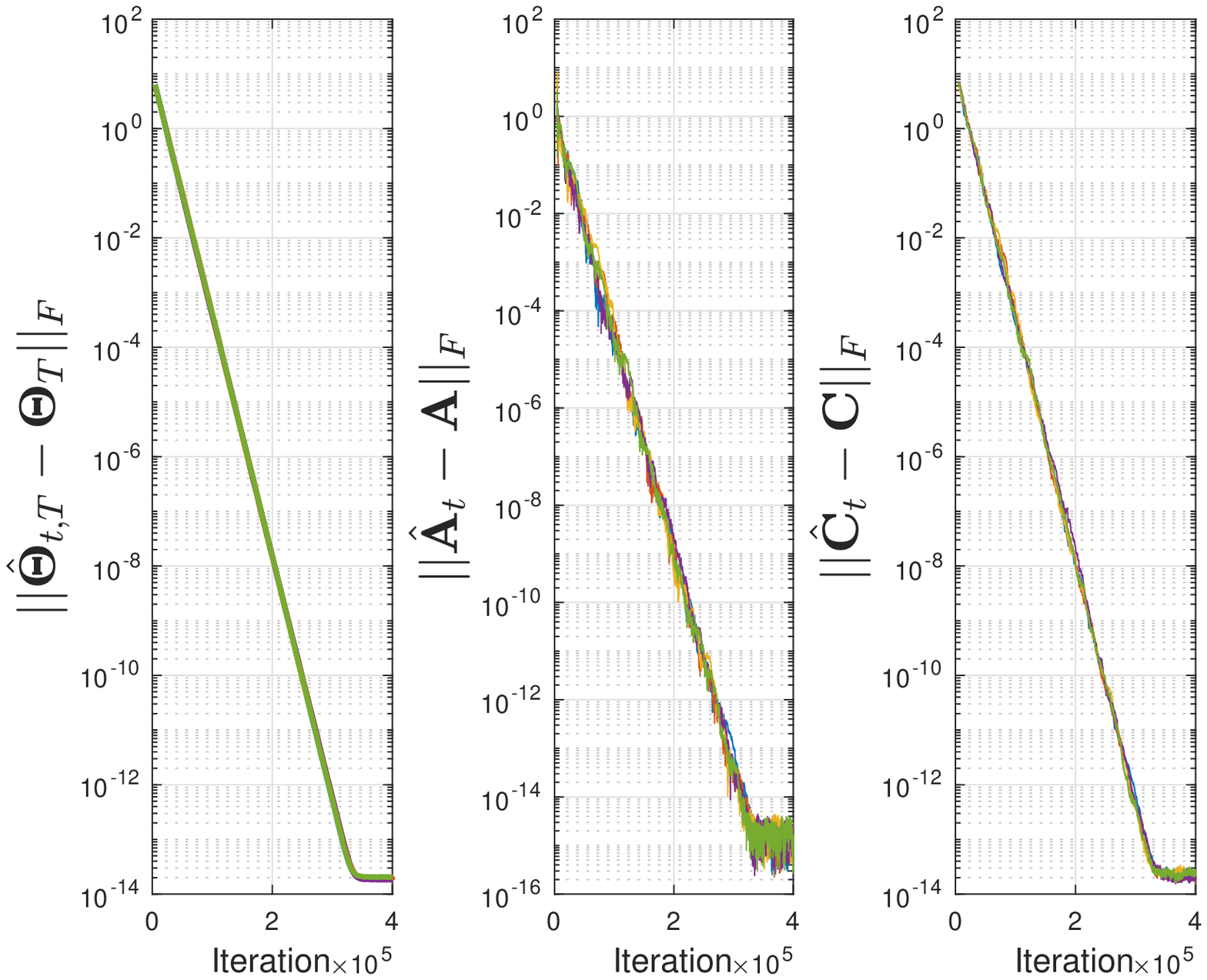}\label{fig:mimo302}}\\
	\subfigure{\includegraphics[width=0.29\textwidth]{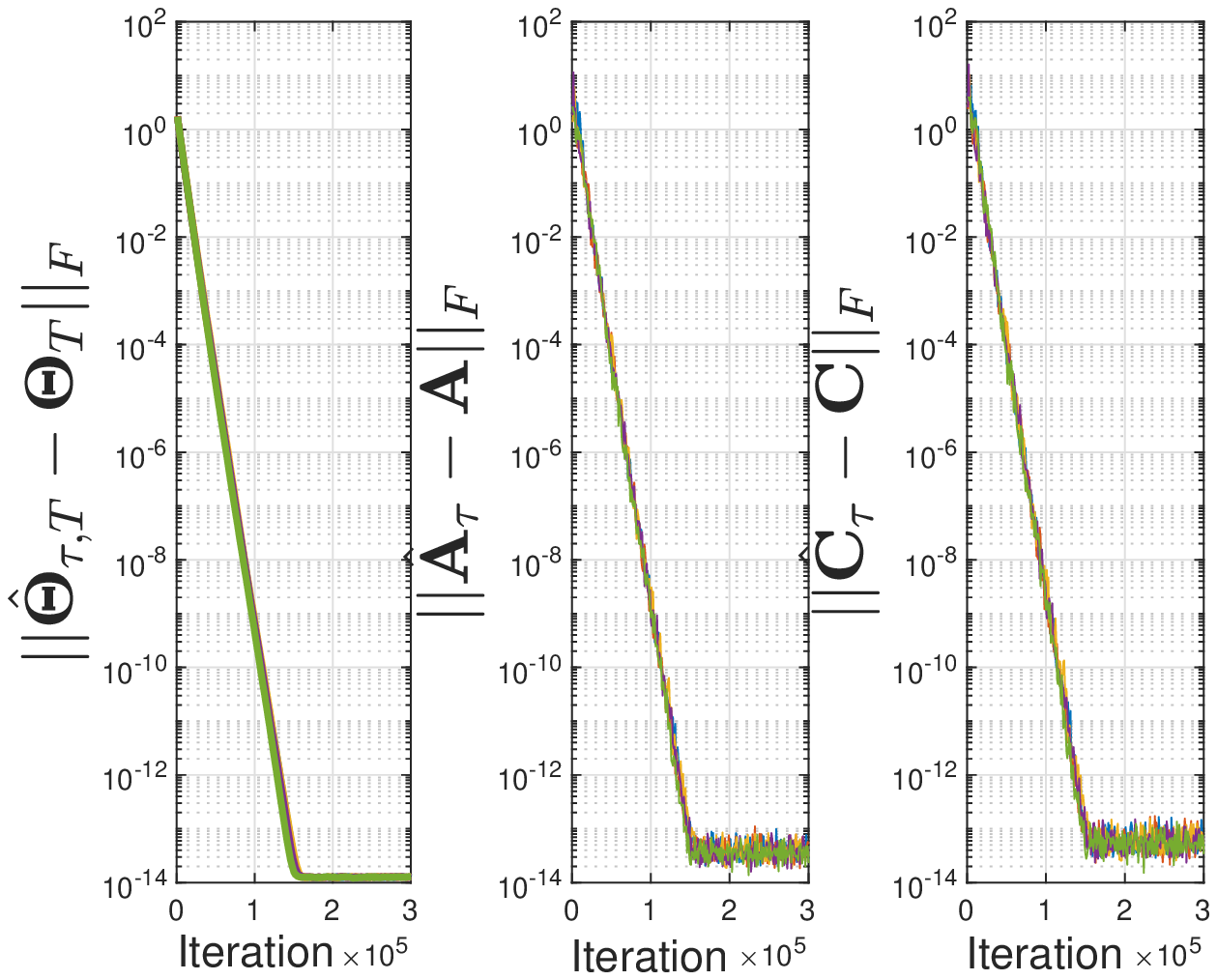}\label{fig:siso301}}
	\subfigure{\includegraphics[width=0.29\textwidth]{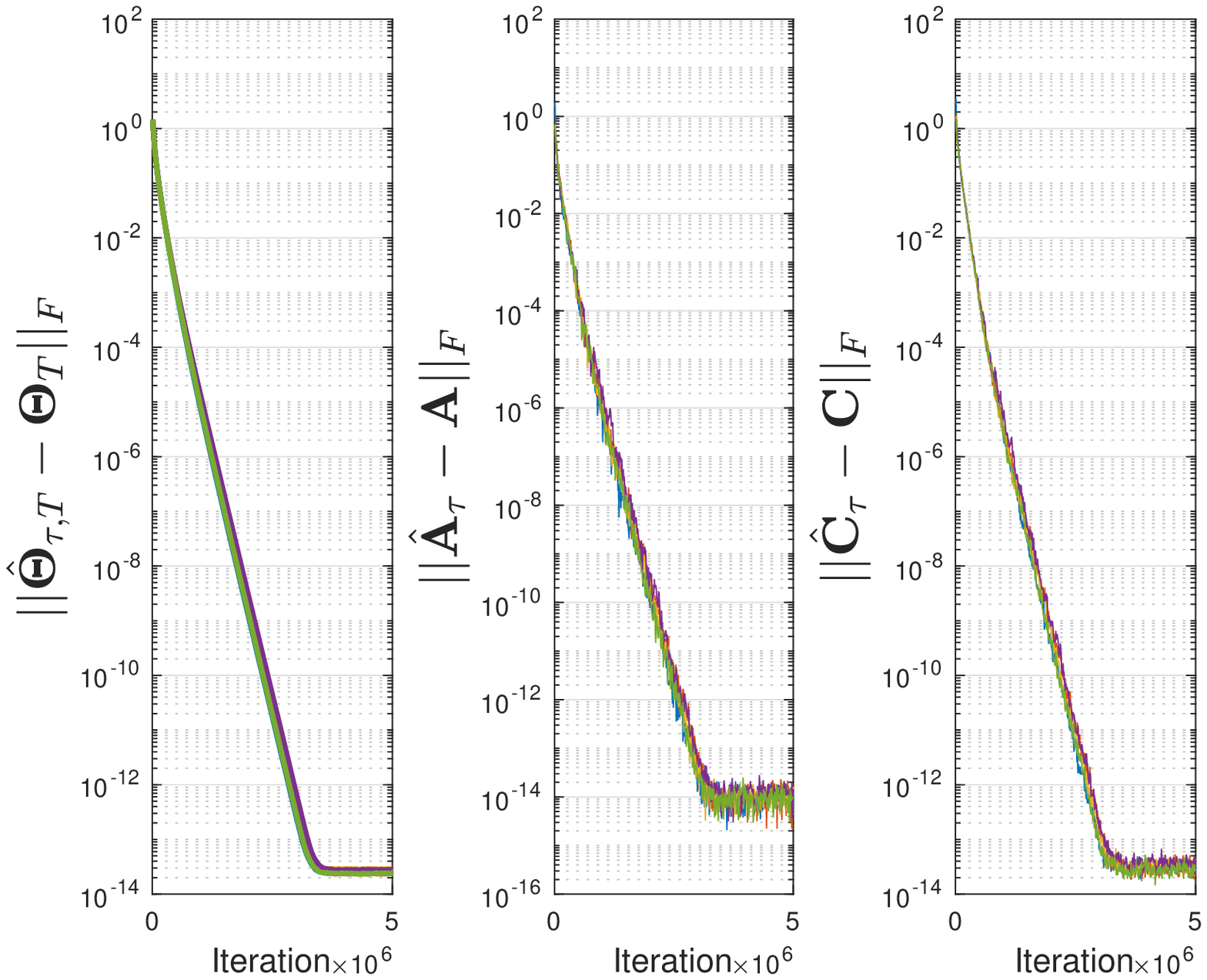}\label{fig:miso301}}
	\subfigure{\includegraphics[width=0.29\textwidth]{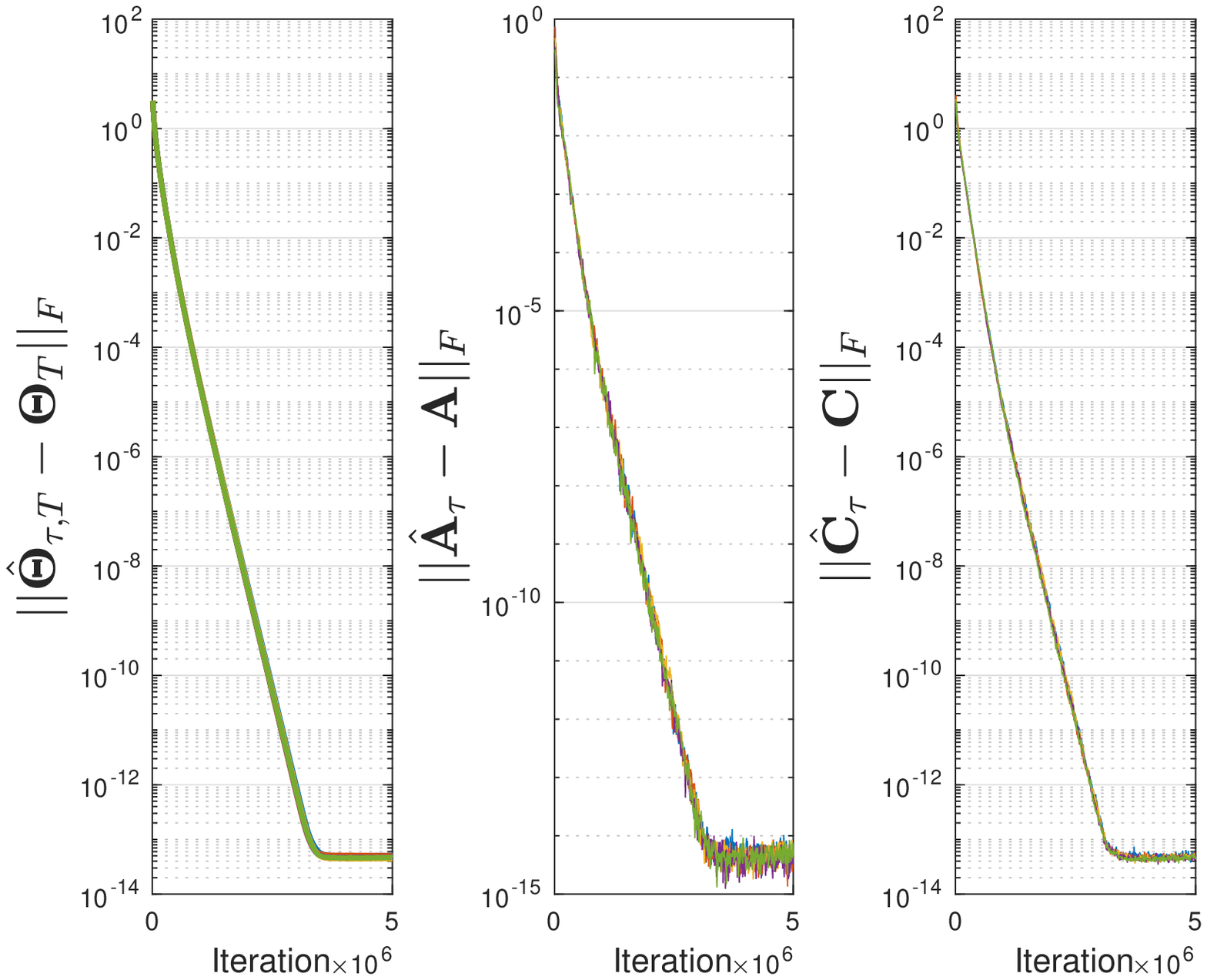}\label{fig:mimo301}}
	\caption{The top row depicts the convergence of Algorithm \ref{al:comb}. The bottom row depicts the convergence of Algorithm \ref{al:comb1}. In (a), (d) the underlying system is SISO with $n=30$, $m=1$, $p=1$. In (b), (e) the considered system is MISO with $n=5$, $m=6$, $p=1$. In (c), (f) the test system is MIMO with $n=5$, $m=6$, $p=4$.}
\end{figure}

In this section, we evaluate the performance of the proposed approaches. The matrix $\mathbf{A}$ is randomly generated by choosing the conjugate pairs of roots of the characteristic polynomial inside a circle with a maximum radius of $\rho(\mathbf{A})=0.975$. Elements of the matrices $\mathbf{C}$ and $\mathbf{D}$ are independently drawn from a standard Gaussian distribution $\mathcal{N}(0, 1)$ for each experiment. The initial state of the system is zero.
The performance measure in the experiments is the Frobenius norm distance between the estimated solution and the ground truth solution. We repeat each experiment 5 times and each curve corresponds to one independent realization. The spectral radius of $\mathbf{A}$ for SISO, multi-input single-output (MISO), and MIMO systems is $0.975$, $0.70$, and $0.64$, respectively. In experiments, as $\rho(\mathbf{A})$ is close to one, transfer function has a heavy tail and a large $T$ is required. 

The convergence of Algorithm \ref{al:comb} for SISO, MISO, and MIMO systems is depicted in Figs. \ref{fig:siso302}-\ref{fig:mimo302}, where the measurement noise is zero. The hidden state dimension for considered systems is $30$. 
The convergence of Algorithm \ref{al:comb1} for identical systems is depicted in Figs. \ref{fig:siso301}-\ref{fig:mimo301}, when the batch size is $10,000$. In each iteration of Algorithm \ref{al:comb1}, one input-output pair is chosen uniformly at random, and the gradient is implemented based on that sample. The step-size and truncation length are identical for both approaches in each test and are outlined in Appendix \ref{sec:exsim} along with results for noisy systems.
It is observed that the number of iterations required by Algorithm \ref{al:comb} is fewer compared to Algorithm  \ref{al:comb1}. The reason for this difference is that Algorithm \ref{al:comb} has access to a greater number of input-output pairs. The numerical tests confirm that the system identification error can be as small as desired via adjusting the learning rate and the truncation length.   

\begin{figure}
	\centering
	\subfigure{\includegraphics[width=0.4\textwidth]{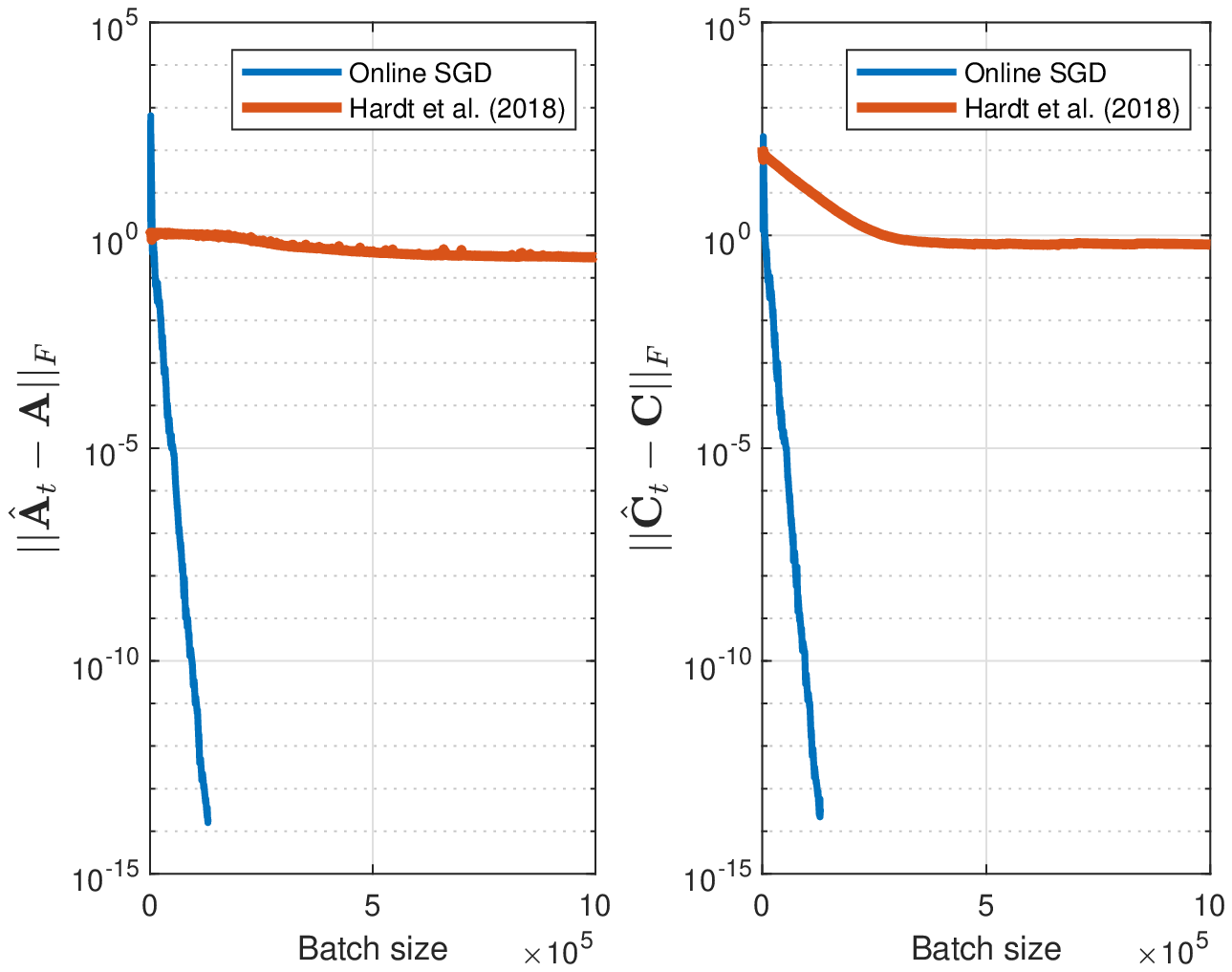}	\label{fig:hardt1}}
	\subfigure{\includegraphics[width=0.4\textwidth]{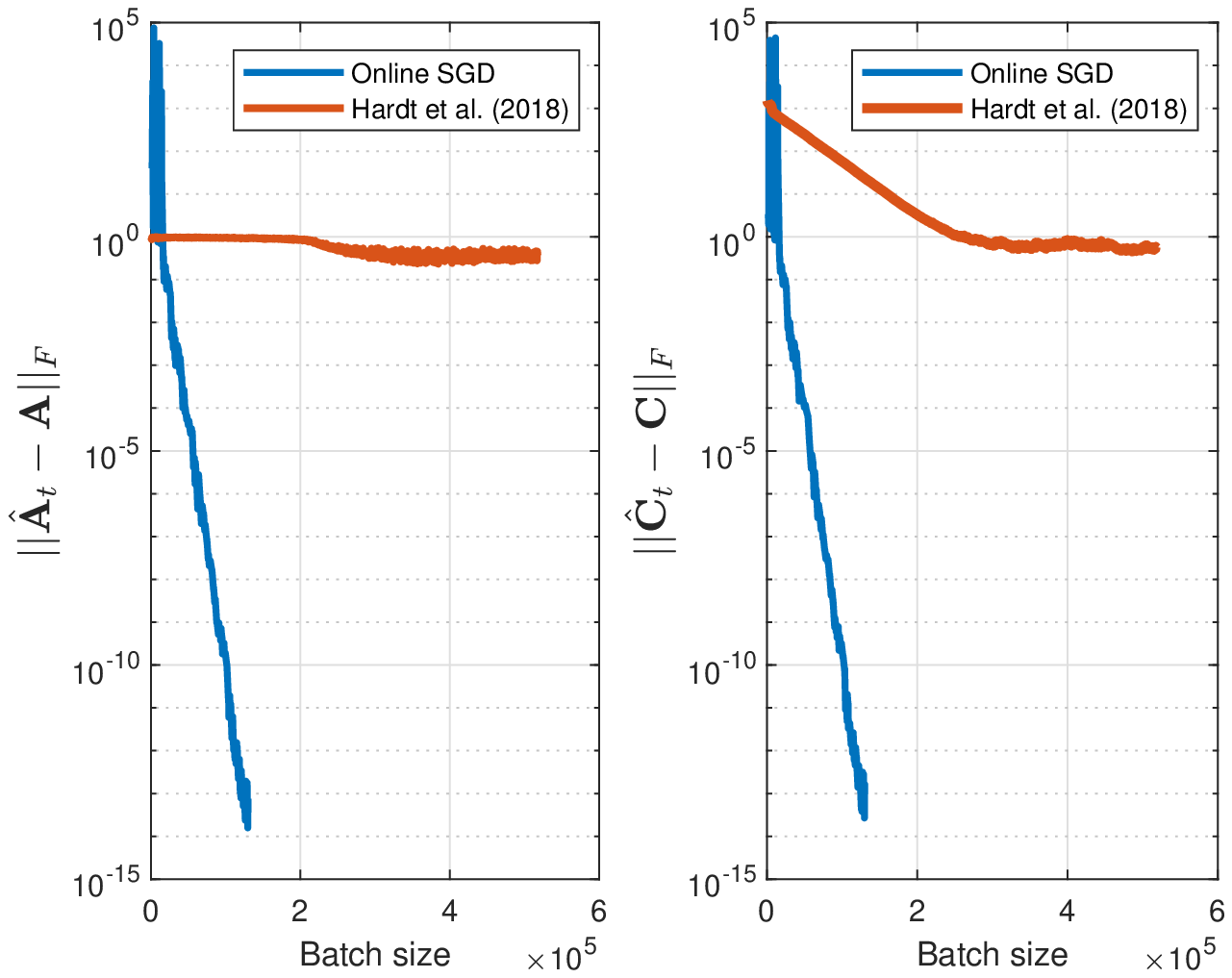}	\label{fig:hardt2}}
	\caption{Comparison of the performance of Algorithm \ref{al:comb} with \cite{hardt2018gradient}: (a) $m=1$, $n=20$, and $p=1$; and (b) $m=1$, $n=30$, and $p=1$.}
\end{figure}

We compare the performance of Algorithm \ref{al:comb} against \cite{hardt2018gradient}. The gradient projection algorithm in \cite{hardt2018gradient} implements gradient steps for $\mathbf{A}$ and $\mathbf{C}$ based on extracted information from a trajectory. After the gradients are implemented, the estimation of $\mathbf{A}$ is projected to a convex set. This set is characterized by $\Re(q(z)/z^n)>|\Im(q(z)/z^n)|$. The initialization of $\hat{\mathbf{A}}_0$ is critical for \cite{hardt2018gradient}. If the initial $\hat{\mathbf{A}}_0$ is unstable, the system blows up after one trajectory is fed to the system and the gradients cannot be computed. We compare the performance of the gradient projection algorithm in \cite{hardt2018gradient} against Algorithm \ref{al:comb} based on the number of input-output pairs that are fed to both approaches, where samples are discarded after the gradient implementations. In each iteration, the length of each trajectory fed to the gradient projection algorithm in \cite{hardt2018gradient} is $500$. For the comparisons, we consider two SISO systems. In the first system, we have $m=1$, $n=20$, and $p=1$. Moreover, for the second system, we set $m=1$, $n=30$, and $p=1$. We observe that Algorithm \ref{al:comb} outperforms the gradient projection algorithm in \cite{hardt2018gradient} given an identical number of input-output pairs. The reason is that the gradients for $\mathbf{A}$ and $\mathbf{C}$ are extracted from a non-linear non-convex regression in \cite{hardt2018gradient} and also the gradient implementation in \cite{hardt2018gradient} requires a greater number of input-output samples compared to Algorithm \ref{al:comb}. We compare the performance of our method against \cite{oymak2019non} in Appendix \ref{sec:exsim}. 

\vspace{-.3cm}
\section{Concluding Remarks and Future Directions}

This paper presents a novel approach to learn unknown transformation matrices of a certain class of stable linear dynamical systems from a single, noisy sequence of input-output pairs. We proposed online and offline SGD algorithms, proved that they efficiently learn the Markov parameters of the system at a linear convergence rate, and provide novel complexity bounds. When the unknown transformation matrices of the system have Brunovsky canonical form, we draw connections between Markov parameters and unknown transformation matrices using the transfer function of the system. We proved that the linear convergence of the Markov parameters enforces a linear convergence rate for unknown matrices to converge to their ground truth weights. We demonstrated the performance of our methods against state of the art methods through numerical simulations. It would be interesting to see whether our proposed approaches could be extended to the identification of periodic and Markov jump linear systems as well, as such systems are structurally more similar to multi-layer perceptron-type neural architectures. We defer such investigations to a future work.

\bibliography{ref_si}
\newpage
\appendix
\section{Bounding Markov Parameter Truncation Error}
In this section, we list the proofs that bound the error of Markov parameter truncation.
\allowdisplaybreaks
\subsection{Proof of Lemma \ref{le:trun}}
First, let us bound $\norm{\mathbf{\mathbf{h}}_{t}}_2^2$. Since $\mathbf{\mathbf{h}}_{t}$ is a linear combination of system inputs, we have
\begin{align}
	\mathbf{h}_{t}=\mathbf{A}^{t}\mathbf{h}_{0}+\sum_{i=1}^{t}\mathbf{A}^{i-1}\mathbf{B}\mathbf{u}_{t-i}.\label{eq:rec1}
\end{align}
We bound $\norm{\mathbf{\mathbf{h}_{t+1}}}_2^2$ using recursion as follows:
\begin{align}
	& \mathbb{E}_{\mathbf{u}}[\norm{\mathbf{h}_{t+1}}_2^2]= \mathbb{E}_{\mathbf{u}}[\norm{\mathbf{A}\mathbf{h}_{t}+\mathbf{B}\mathbf{u}_t}_2^2]\overset{(a)}{=}\mathbb{E}_{\mathbf{u}}[\norm{\mathbf{A}\mathbf{h}_{t}}_2^2]+\mathbb{E}_{\mathbf{u}}[\norm{\mathbf{B}\mathbf{u}_t}_2^2]\nonumber\\
	&\leq\mathbb{E}_{\mathbf{u}}[\norm{\mathbf{A}\mathbf{h}_{t}}_2^2]+\mathbb{E}_{\mathbf{u}}[\norm{\mathbf{u}_t}_2^2]\norm{\mathbf{B}}_2^2\overset{(b)}{\leq} \mathbb{E}_{\mathbf{u}}[\norm{\mathbf{A}\mathbf{h}_{t}}_2^2]+m(\max(\boldsymbol{\sigma}^{\cdot 2}))\norm{\mathbf{B}}_F^2,\label{eq:rec2}
\end{align}
where (a) follows due to the fact that $\mathbf{\mathbf{h}}_{t}$ is made of $\{\mathbf{u}_i\}_{i=0}^{t-1}$ and the initial state $\mathbf{\mathbf{h}}_{0}$, which are independent from $\mathbf{u}_t$. We have (b) since $\mathbb{E}_{\mathbf{u}}[\norm{\mathbf{u}_t}_2^2]=\sum_{i=1}^{m}\boldsymbol{\sigma}^{\cdot 2}(i) \leq m(\max(\boldsymbol{\sigma}^{\cdot 2}))$. 

Based on the update rule for the hidden state, one can expand $\mathbb{E}_{\mathbf{u}}[\norm{\mathbf{A}\mathbf{h}_{t}}_2^2]$ recursively as  $\mathbb{E}_{\mathbf{u}}[\norm{\mathbf{A}\mathbf{h}_{t}}_2^2]=\mathbb{E}_{\mathbf{u}}[\norm{\mathbf{A}(\mathbf{A}\mathbf{h}_{t-1}+\mathbf{B}\mathbf{u}_{t-1})}_2^2]$. Consider eigendecomposition for $\mathbf{A}$ as $\mathbf{A}=\mathbf{V}\boldsymbol{\Lambda}\mathbf{V}^{-1}$. Using recursions \eqref {eq:rec1} and \eqref {eq:rec2}, we bound the norm of the hidden state as follows:
\begin{align}
	&\mathbb{E}_{\mathbf{u}}[\norm{\mathbf{h}_{t+1}}_2^2]\leq \mathbb{E}_{\mathbf{u}}\left[\norm{\mathbf{A}^{t+1}\mathbf{h}_0+\sum_{i=1}^{t}\mathbf{A}^i\mathbf{B}\mathbf{u}_{t-i}}_2^2\right]+m\max(\boldsymbol{\sigma}^{\cdot 2})\norm{\mathbf{B}}_F^2\nonumber\\
	&\overset{(a)}{=}\mathbb{E}_{\mathbf{u}}\left[\norm{\mathbf{V}\boldsymbol{\Lambda}^{t+1}\mathbf{V}^{-1}\mathbf{h}_0+\sum_{i=1}^{t}\mathbf{V}\boldsymbol{\Lambda}^{i}\mathbf{V}^{-1}\mathbf{B}\mathbf{u}_{t-i}}_2^2\right]+m\max(\boldsymbol{\sigma}^{\cdot 2})\norm{\mathbf{B}}_F^2\nonumber\\
	&=\mathbb{E}_{\mathbf{u}}\left[\norm{\mathbf{V}\boldsymbol{\Lambda}^{t+1}\mathbf{V}^{-1}\mathbf{h}_0+\sum_{i=1}^{t}\mathbf{V}\boldsymbol{\Lambda}^{i}\mathbf{V}^{-1}\mathbf{B}\mathbf{u}_{t-i}}_2^2\right]+m\max(\boldsymbol{\sigma}^{\cdot 2})\norm{\mathbf{B}}_F^2
	\nonumber\\
	&=\mathbb{E}_{\mathbf{u}}\left[\norm{\mathbf{V}\boldsymbol{\Lambda}^{t+1}\mathbf{V}^{-1}\mathbf{h}_0}_2^2+\sum_{i=1}^{t}\norm{\mathbf{V}\boldsymbol{\Lambda}^{i}\mathbf{V}^{-1}\mathbf{B}\mathbf{u}_{t-i}}_2^2\right]+m\max(\boldsymbol{\sigma}^{\cdot 2})\norm{\mathbf{B}}_F^2\nonumber\\
	&\overset{(b)}{\leq} n^2 \ell \rho(\mathbf{A})^{2(t+1)}\norm{\mathbf{h}_0}_2^2+n^2\ell m\max(\boldsymbol{\sigma}^{\cdot 2})\norm{\mathbf{B}}_F^2\sum_{i=1}^{t}\rho(\mathbf{A})^{2i}+m\max(\boldsymbol{\sigma}^{\cdot 2})\norm{\mathbf{B}}_F^2
	\nonumber\\
	&=n^2 \ell \rho(\mathbf{A})^{2(t+1)}\norm{\mathbf{h}_0}_2^2+\frac{n^2\ell m\max(\boldsymbol{\sigma}^{\cdot 2})\rho(\mathbf{A})^2(1-\rho(\mathbf{A})^{2t})\norm{\mathbf{B}}_F^2}{1-\rho(\mathbf{A})^{2}}+m\max(\boldsymbol{\sigma}^{\cdot 2})\norm{\mathbf{B}}_F^2\nonumber\\
	&\leq n^2 \ell \rho(\mathbf{A})^2\norm{\mathbf{h}_0}_2^2+\frac{n^2\ell m\max(\boldsymbol{\sigma}^{\cdot 2})\rho(\mathbf{A})^2\norm{\mathbf{B}}_F^2}{1-\rho(\mathbf{A})^{2}}+m\max(\boldsymbol{\sigma}^{\cdot 2})\norm{\mathbf{B}}_F^2,\label{eq:hidd}
\end{align}
where (a) follows from eigendecomposition for $\mathbf{A}$. In (b), $\ell=\norm{\mathbf{V}^{-1}}_F^2$.
Using the above upper-bound for $\mathbb{E}_{\mathbf{u}}[\norm{\mathbf{h}_{t+1}}_2^2]$, we bound $\mathbb{E}_{\mathbf{u}}[\norm{\mathbf{C}\mathbf{A}^{T-1}\mathbf{h}_{t-T+1}}_2^2]$ as follows:
\begin{align}
	&\mathbb{E}_{\mathbf{u}}[\norm{\mathbf{C}\mathbf{A}^{T-1}\mathbf{h}_{t-T+1}}_2^2]\leq \norm{\mathbf{C}}_F^2\norm{\mathbf{\mathbf{A}}^{T-1}}_F^2\mathbb{E}_{\mathbf{u}}[\norm{\mathbf{h}_{t-T+1}}_2^2]\nonumber\\
	&\leq \norm{\mathbf{C}}_F^2\norm{\mathbf{V}\boldsymbol{\Lambda}^{T-1}\mathbf{V}^{-1}}_F^2\mathbb{E}_{\mathbf{u}}[\norm{\mathbf{h}_{t-T+1}}_2^2]\nonumber\\
	&\leq n^2\ell\norm{\mathbf{C}}_F^2\rho(\mathbf{A})^{2(T-1)}\mathbb{E}_{\mathbf{u}}[\norm{\mathbf{h}_{t-T+1}}_2^2]\nonumber\\
	&\leq n^2\ell\norm{\mathbf{C}}_F^2\rho(\mathbf{A})^{2(T-1)}\nonumber\\
	&\times\underbrace{\left[n^2 \ell \rho(\mathbf{A})^2\norm{\mathbf{h}_0}_2^2+\frac{n^2\ell m\max(\boldsymbol{\sigma}^{\cdot 2})\rho(\mathbf{A})^2\norm{\mathbf{B}}_F^2}{1-\rho(\mathbf{A})^{2}}+m\max(\boldsymbol{\sigma}^{\cdot 2})\norm{\mathbf{B}}_F^2\right]}_{\gamma}.\nonumber
\end{align}
When the truncation length $T$ is large enough, $\lim_{T\rightarrow\infty}\rho(\mathbf{A})^{2(T-1)}=0$ since $\rho(\mathbf{A})<1$. For a large $T$, we find $\mathbb{E}_{\mathbf{u}}\left[\norm{\mathbf{C}\mathbf{A}^{T-1}\mathbf{h}_{t-T+1}}_2^2\right]\approx 0$, or equivalently, $\lim_{T\rightarrow\infty}\mathbb{E}_{\mathbf{u}}\left[\norm{\mathbf{C}\mathbf{A}^{T-1}\mathbf{h}_{t-T+1}}_2^2\right]=0$.

\subsection{Proof of Proposition \ref{pro:noise}}\label{sec:proofprop}
From the least squares problem, we have:
\begin{align}
	\hat{\boldsymbol{\Theta}}_T=&\arg\min_{\hat{\boldsymbol{\Theta}}_T}\lim_{N\rightarrow\infty}\frac{1}{2N}\sum_{t=1}^N\norm{\mathbf{y}_t- \hat{\boldsymbol{\Theta}}_T\mathbf{x}_t}_2^2\nonumber\\
	&=\arg\min_{\hat{\boldsymbol{\Theta}}_T}\lim_{N\rightarrow\infty}\frac{1}{2N}\sum_{t=1}^N\norm{\boldsymbol{\Theta}_T\mathbf{x}_t+\mathbf{C}\mathbf{A}^{T-1}\mathbf{h}_{t-T+1}+\boldsymbol{\zeta}_t- \hat{\boldsymbol{\Theta}}_T\mathbf{x}_t}_2^2\nonumber\\
	&=\arg\min_{\hat{\boldsymbol{\Theta}}_T}\lim_{N\rightarrow\infty}\frac{1}{2N}\sum_{t=1}^N\Bigg[\norm{\boldsymbol{\Theta}_T\mathbf{x}_t- \hat{\boldsymbol{\Theta}}_T\mathbf{x}_t}_2^2+\norm{\mathbf{C}\mathbf{A}^{T-1}\mathbf{h}_{t-T+1}+\boldsymbol{\zeta}_t}_2^2
	\nonumber\\
	&+2\text{Tr}\left(\mathbf{x}_t'(\boldsymbol{\Theta}_T- \hat{\boldsymbol{\Theta}}_T)'(\mathbf{C}\mathbf{A}^{T-1}\mathbf{h}_{t-T+1}+\boldsymbol{\zeta}_t)\right)\Bigg]\nonumber\\
	&\overset{(a)}{=}\arg\min_{\hat{\boldsymbol{\Theta}}_T}\lim_{N\rightarrow\infty}\frac{1}{2N}\sum_{t=1}^N\norm{(\boldsymbol{\Theta}_T- \hat{\boldsymbol{\Theta}}_T)\mathbf{x}_t}_2^2.\label{eq:min}
\end{align}
Before justifying (a), we notice that: 
\begin{enumerate}
	\item From Lemma \ref{le:trun}, we know that $\norm{\mathbf{C}\mathbf{A}^{T-1}\mathbf{h}_{t-T+1}}_2^2$ tends to zero when $T$ increases.
	\item We have $\lim_{N\rightarrow\infty}\frac{1}{2N}\sum_{t=1}^N\boldsymbol{\zeta}_t=\mathbf{0}$ with probability one due to  Chebyshev's inequality \cite{saw1984chebyshev}.
\end{enumerate}
In light of the above arguments, $\lim_{N\rightarrow\infty}\frac{1}{2N}\sum_{t=1}^N\text{Tr}\left(\mathbf{x}_t'(\boldsymbol{\Theta}_T- \hat{\boldsymbol{\Theta}}_T)'(\mathbf{C}\mathbf{A}^{T-1}\mathbf{h}_{t-T+1}+\boldsymbol{\zeta}_t)\right)$ can be made as small as desired.

The Hessian matrix for $\lim_{N\rightarrow\infty}\frac{1}{2N}\sum_{t=1}^N\norm{(\boldsymbol{\Theta}_T- \hat{\boldsymbol{\Theta}}_T)\mathbf{x}_t}_2^2$  is $\lim_{N\rightarrow\infty}\frac{1}{2N}\sum_{t=1}^N\mathbf{x}_t\mathbf{x}_t'$, which is positive definite with probability one due to Chebyshev's inequality, and thus, the solution for $\boldsymbol{\Theta}_T$ in \eqref{eq:min} is unique. Based on the above arguments, we observe that $\hat{\boldsymbol{\Theta}}_T=\boldsymbol{\Theta}_T$.
\subsection{Proof of Remark \ref{le:noise}}
Instead of \eqref{eq:state}, consider the following dynamics:
\begin{align}\textstyle
	&\mathbf{h}_{t+1}=\mathbf{A}\mathbf{h}_{t}+\mathbf{B}\mathbf{u}_t+\boldsymbol{\varsigma}_t,\nonumber\\&
	\mathbf{y}_t=\mathbf{C}\mathbf{h}_t+\mathbf{D}\mathbf{u}_t+\boldsymbol{\zeta}_t,\nonumber
\end{align}
where $\boldsymbol{\varsigma}_t$ is the process noise at time instance $t$. Similar to \ref{sec:proofprop}, we have:
\begin{align}
	\hat{\boldsymbol{\Theta}}_T=&\arg\min_{\hat{\boldsymbol{\Theta}}_T}\lim_{N\rightarrow\infty}\frac{1}{2N}\sum_{t=1}^N\norm{\mathbf{y}_t- \hat{\boldsymbol{\Theta}}_T\mathbf{x}_t}_2^2\nonumber\\
	&=\arg\min_{\hat{\boldsymbol{\Theta}}_T}\lim_{N\rightarrow\infty}\frac{1}{2N}\sum_{t=1}^N\norm{\boldsymbol{\Theta}_T\mathbf{x}_t+\mathbf{C}\mathbf{A}^{T-1}\mathbf{h}_{t-T+1}+\boldsymbol{\zeta}_t+\mathbf{C}\sum_{i=1}^{T-1}\mathbf{A}^{i-1}\boldsymbol{\varsigma}_{t-i}- \hat{\boldsymbol{\Theta}}_T\mathbf{x}_t}_2^2\nonumber\\
	&=\arg\min_{\hat{\boldsymbol{\Theta}}_T}\lim_{N\rightarrow\infty}\frac{1}{2N}\sum_{t=1}^N\Bigg[\norm{\boldsymbol{\Theta}_T\mathbf{x}_t- \hat{\boldsymbol{\Theta}}_T\mathbf{x}_t}_2^2+\norm{\mathbf{C}\mathbf{A}^{T-1}\mathbf{h}_{t-T+1}+\boldsymbol{\zeta}_t+\mathbf{C}\sum_{i=1}^{T-1}\mathbf{A}^{i-1}\boldsymbol{\varsigma}_{t-i}}_2^2
	\nonumber\\
	&+2\text{Tr}\left(\mathbf{x}_t'(\boldsymbol{\Theta}_T- \hat{\boldsymbol{\Theta}}_T)'(\mathbf{C}\mathbf{A}^{T-1}\mathbf{h}_{t-T+1}+\boldsymbol{\zeta}_t+\mathbf{C}\sum_{i=1}^{T-1}\mathbf{A}^{i-1}\boldsymbol{\varsigma}_{t-i})\right)\Bigg]\nonumber\\
	&=\arg\min_{\hat{\boldsymbol{\Theta}}_T}\lim_{N\rightarrow\infty}\frac{1}{2N}\sum_{t=1}^N\norm{(\boldsymbol{\Theta}_T- \hat{\boldsymbol{\Theta}}_T)\mathbf{x}_t}_2^2.\label{eq:min1}
\end{align}
Since the process noise is independent of
the inputs, from \eqref{eq:min1}, we observe that one can learn $\hat{\boldsymbol{\Theta}}_T$ through a regression given in \eqref{opt:reg}. Although $\mathbf{C}\sum_{i=1}^{T-1}\mathbf{A}^{i-1}\boldsymbol{\varsigma}_{t-i}$ accumulates in the hidden state of the system, the norm of $\mathbf{C}\mathbf{A}^{i-1}\boldsymbol{\varsigma}_{t-i}$ becomes small for a large $i$. The reason is that $\rho(\mathbf{A})<1$ (similar to \eqref{eq:hidd}).

\subsection{Proof of Lemma \ref{le:te}}
We bound the Frobenius norm of the aggregate transfer function terms which are truncated and ignored in the transfer function approximation as follows: 
\begin{align}
	\norm{\mathbf{E}_{z,T}}_F^2=\norm{\sum_{t=T}^\infty z^{-t}\mathbf{C}\mathbf{A}^{t-1}\mathbf{B}}_F^2\hspace{-.2cm}\leq \sum_{t=T}^\infty\norm{ z^{-t}\mathbf{C}\mathbf{A}^{t-1}\mathbf{B}}_F^2\overset{(a)}{\leq} \sum_{t=T}^\infty n^2\ell \norm{\mathbf{C}}_F^2\norm{\mathbf{B}}_F^2\rho(\mathbf{A})^{2(t-1)}|z|^{-2t}\hspace{-.2cm},\nonumber
\end{align}
where (a) follows due to the eigendecomposition $\mathbf{A}=\mathbf{V}\boldsymbol{\Lambda}\mathbf{V}^{-1}$ for $\mathbf{A}$.
Assuming that $|\rho(\mathbf{A})|<|z|$ (stability region of the system), we have
\begin{align}
	&\norm{\mathbf{E}_{z,T}}_F^2=\norm{\sum_{t=T}^\infty z^{-t}\mathbf{C}\mathbf{A}^{t-1}\mathbf{B}}_F^2\leq \frac{n^2\ell \norm{\mathbf{C}}_F^2\norm{\mathbf{B}}_F^2\rho(\mathbf{A})^{2(T-1)}}{1-\rho(\mathbf{A})^2|z|^{-2}}\nonumber.
\end{align}
\section{Batch Methods}\label{sec:batch}
In this section, we introduce two batch algorithms.
\subsection{Offline SGD Combined with The Linear System}
The offline SGD algorithm can also be used to estimate the unknown weight matrices, where we find $\boldsymbol{\Gamma}_{\tau,T}$ and $\boldsymbol{\varkappa}_\tau$ using offline SGD iterates $\hat{\boldsymbol{\Theta}}_{\tau,T}$. This approach is summarized in Algorithm \ref{al:comb1}.
\begin{algorithm}
	\SetAlgoLined
	\textbf{Input}: $\{\mathbf{x}_t,\mathbf{y}_t\}_{t=1}^N$, learning rate $\eta$, $\tau=1$\\
	\textbf{Output}: Estimation of $\mathbf{A}$, $\mathbf{C}$ and $\mathbf{D}$\\
	\For{$\tau$ from $1$ to $\text{END}$}{
		Uniformly at random choose $t\in\{T,T+1,\dots,N\}$;\\
		$\hat{\boldsymbol{\Theta}}_{\tau,T}=\hat{\boldsymbol{\Theta}}_{\tau-1,T}-\eta(\hat{\boldsymbol{\Theta}}_{\tau-1,T}\mathbf{x}_t-\mathbf{y}_t)\mathbf{x}_t'$\\
		Find $\boldsymbol{\Gamma}_{\tau,T}$ and $\boldsymbol{\varkappa}_\tau$\\
		$\hat{\boldsymbol{\varrho}}_\tau=(\boldsymbol{\Gamma}_{\tau,T}^H\boldsymbol{\Gamma}_{\tau,T})^{-1}\boldsymbol{\Gamma}^H_{\tau,T}\boldsymbol{\varkappa}_\tau$
	}
	\caption{Offline SGD combined with the linear system}\label{al:comb1}
\end{algorithm}
\subsection{Online pseudo-inverse based method}
If the computation cost of pseudo-inversion to solve \eqref{opt:lim} is not high, one can solve \eqref{opt:lim} by the pseudo-inverse method when each new input-output pair arrives. We solve \eqref{opt:lim} with $N=t$, where $t$ is the number observed input-output pairs. We extract the estimated weight matrices from the linear system of equations \eqref{eq:linsys} if $\hat{\boldsymbol{\Theta}}_T$ is used instead of $\hat{\boldsymbol{\Theta}}_{t,T}$. This approach is summarized in Algorithm \ref{al:inv}. Similar to \eqref{eq:shir}, it is easy to show that the following complexity bound holds for Algorithm \ref{al:inv}:
\begin{align}
	\mathbb{E}_{\mathbf{u}}\left[\mathbb{E}_{\boldsymbol{\zeta}}\left[\norm{\hat{\boldsymbol{\varrho}}_t-\boldsymbol{\varrho}}_2^2\right]\right]\leq s_1+l_2\Big(nm(n+nmp)(T-1)\Big)\chi_t^2.\nonumber
\end{align}
where $s_1>0$ (given in Appendix \ref{sec:lineq}) can be made as small as desired by increasing $T$.
Although Algorithm \ref{al:inv} is more computationally expensive compared to Algorithm \ref{al:comb}, it is more robust to noise. The reason is that it simultaneously uses all the available samples to estimate the Markov parameters, and due to averaging, the aggregate noise  is attenuated.
\begin{algorithm}
	\SetAlgoLined
	\textbf{Input}: $\{\mathbf{x}_i,\mathbf{y}_i\}_{i=1}^t$, learning rate $\eta$, $t=1$\\
	\textbf{Output}: Estimation of $\mathbf{A}$, $\mathbf{C}$ and $\mathbf{D}$\\
	\If{a new input-output pair arrives}{
		Update \eqref{opt:lim}\\
		Solve \eqref{opt:lim} by the pseudo-inverse method\\
		Find $\boldsymbol{\Gamma}_{t,T}$ and $\boldsymbol{\varkappa}_t$\\
		$\hat{\boldsymbol{\varrho}}_t=(\boldsymbol{\Gamma}_{t,T}^H\boldsymbol{\Gamma}_{t,T})^{-1}\boldsymbol{\Gamma}^H_{t,T}\boldsymbol{\varkappa}_t$\\
		$t=t+1$
	}
	\caption{Online pseudo-inverse based method}\label{al:inv}
\end{algorithm}
\section{Convergence Analysis}
In this section, we list proofs that are related to the regression problem \eqref{opt:lim} and the  convergence of the proposed algorithms to solve it. 
\subsection{Proof of Theorem \ref{th:prob}}
Consider the following problem:
\begin{align}
	\hat{\boldsymbol{\Theta}}_T=\arg\min_{\hat{\boldsymbol{\Theta}}_T}\frac{1}{2(N-T+1)}\sum_{t=T}^{N}\norm{\mathbf{y}_t- \hat{\boldsymbol{\Theta}}_T\mathbf{x}_t}_2^2.\nonumber
\end{align}
The starting point in the above summation is $t=T$. The reason is that at this point, all the elements of the vector $\mathbf{x}_t$ are filled with random numbers (see \eqref{eq:x}). Once we bound the distance between $\hat{\boldsymbol{\Theta}}_T$ for the above problem and $\boldsymbol{\Theta}_T$, it will an upper-bound for $\hat{\boldsymbol{\Theta}}_T$ in \eqref{opt:lim}. The Hessian matrix for the above problem is $\frac{1}{N-T+1}\sum_{t=T}^{N}\mathbf{x}_t\mathbf{x}_t'$. When $N\geq 2\:T$, the Hessian matrix is full rank \cite{tsiligkaridis2013covariance}. Since the Hessian matrix is positive definite, we can solve the above problem using the first order optimality condition:
\begin{align}
	&	\nabla_{\hat{\boldsymbol{\Theta}}_T} \frac{1}{2(N-T+1)}\sum_{t=T}^{N}\norm{\mathbf{y}_t- \hat{\boldsymbol{\Theta}}_T\mathbf{x}_t}_2^2=\frac{1}{N-T+1}\sum_{t=T}^{N}\left(\mathbf{y}_t- \hat{\boldsymbol{\Theta}}_T\mathbf{x}_t\right)\mathbf{x}_t'\nonumber\\
	&=\frac{1}{N-T+1}\sum_{t=T}^{N}\left(\boldsymbol{\Theta}_T\mathbf{x}_t+\boldsymbol{\zeta}_t+\mathbf{C}\mathbf{A}^{T-1}\mathbf{h}_{t-T+1}- \hat{\boldsymbol{\Theta}}_{T}\mathbf{x}_t\right)\mathbf{x}_t'=\mathbf{0}.\nonumber
\end{align}
Using the above equation, we find:
\begin{align}
	& \left(\hat{\boldsymbol{\Theta}}_T-\boldsymbol{\Theta}_T\right)\left(\sum_{t=T}^{N}\mathbf{x}_t\mathbf{x}_t'\right)=\sum_{t=T}^{N}\boldsymbol{\zeta}_t\mathbf{x}_t'+\sum_{t=T}^{N}\mathbf{C}\mathbf{A}^{T-1}\mathbf{h}_{t-T+1}\mathbf{x}_t'.\nonumber
\end{align}
Since when $N\geq 2\:T$, the scatter matrix $\frac{1}{2(N-T+1)}\sum_{t=T}^{N}\mathbf{x}_t\mathbf{x}_t'$ is full rank and invertible. We find the difference between the optimal solution and the current point as follows:
\begin{align}
	\hat{\boldsymbol{\Theta}}_T-\boldsymbol{\Theta}_T=\left(\sum_{t=T}^{N}\boldsymbol{\zeta}_t\mathbf{x}_t'\right)\left(\sum_{t=T}^{N}\mathbf{x}_t\mathbf{x}_t'\right)^{-1}+\left(\sum_{t=T}^{N}\mathbf{C}\mathbf{A}^{T-1}\mathbf{h}_{t-T+1}\mathbf{x}_t'\right)\left(\sum_{t=T}^{N}\mathbf{x}_t\mathbf{x}_t'\right)^{-1}.
\end{align}
We bound the expected Frobenius norm distance to the global optimal solution as follows:
\begin{align}
	\mathbb{E}&_{\boldsymbol{\zeta}}\left[\mathbb{E}_{\mathbf{u}}\left[\norm{\hat{\boldsymbol{\Theta}}_T-\boldsymbol{\Theta}_T}_F^2\right]\right]\nonumber\\
	&=\mathbb{E}_{\boldsymbol{\zeta}}\left[\mathbb{E}_{\mathbf{u}}\left[\norm{\left(\sum_{t=T}^{N}\boldsymbol{\zeta}_t\mathbf{x}_t'\right)\left(\sum_{t=T}^{N}\mathbf{x}_t\mathbf{x}_t'\right)^{-1}+\left(\sum_{t=T}^{N}\mathbf{C}\mathbf{A}^{T-1}\mathbf{h}_{t-T+1}\mathbf{x}_t'\right)\left(\sum_{t=T}^{N}\mathbf{x}_t\mathbf{x}_t'\right)^{-1}}_F^2\right]\right]\nonumber\\
	&\leq\mathbb{E}_{\boldsymbol{\zeta}}\left[\mathbb{E}_{\mathbf{u}}\left[\norm{\left(\sum_{t=T}^{N}\boldsymbol{\zeta}_t\mathbf{x}_t'\right)\left(\sum_{t=T}^{N}\mathbf{x}_t\mathbf{x}_t'\right)^{-1}}_F^2\right]\right]\nonumber\\
	&+\mathbb{E}_{\boldsymbol{\zeta}}\left[\mathbb{E}_{\mathbf{u}}\left[\norm{\left(\sum_{t=T}^{N}\mathbf{C}\mathbf{A}^{T-1}\mathbf{h}_{t-T+1}\mathbf{x}_t'\right)\left(\sum_{t=T}^{N}\mathbf{x}_t\mathbf{x}_t'\right)^{-1}}_F^2\right]\right]\nonumber\\
	&\overset{(a)}{\leq} \mathbb{E}_{\boldsymbol{\zeta}}\left[\mathbb{E}_{\mathbf{u}}\left[\norm{\sum_{t=T}^{N}\boldsymbol{\zeta}_t\mathbf{x}_t'}_F^2\norm{\left(\sum_{t=T}^{N}\mathbf{x}_t\mathbf{x}_t'\right)^{-1}}_F^2\right]\right]
	\nonumber\\
	&+\mathbb{E}_{\mathbf{u}}\left[\norm{\sum_{t=T}^{N}\mathbf{C}\mathbf{A}^{T-1}\mathbf{h}_{t-T+1}\mathbf{x}_t'}_F^2\norm{\left(\sum_{t=T}^{N}\mathbf{x}_t\mathbf{x}_t'\right)^{-1}}_F^2\right],
	\nonumber\\
	&\leq \mathbb{E}_{\boldsymbol{\zeta}}\left[\mathbb{E}_{\mathbf{u}}\left[\norm{\sum_{t=T}^{N}\boldsymbol{\zeta}_t\mathbf{x}_t'}_F^2\right]\right]\mathbb{E}_{\mathbf{u}}\left[\norm{\left(\sum_{t=T}^{N}\mathbf{x}_t\mathbf{x}_t'\right)^{-1}}_F^2\right]
	\nonumber\\
	&+\mathbb{E}_{\mathbf{u}}\left[\norm{\sum_{t=T}^{N}\mathbf{C}\mathbf{A}^{T-1}\mathbf{h}_{t-T+1}\mathbf{x}_t'}_F^2\right]\mathbb{E}_{\mathbf{u}}\left[\norm{\left(\sum_{t=T}^{N}\mathbf{x}_t\mathbf{x}_t'\right)^{-1}}_F^2\right],
	\label{eq:distance}
\end{align}
where (a) follows because of the Cauchy-Schwarz inequality. Before simplifying the above inequality, let us bound the norm of $\left(\frac{1}{N-T+1}\sum_{t=T}^{N}\mathbf{x}_t\mathbf{x}_t'\right)^{-1}$ as follows:
\begin{align}
	& 	\mathbb{E}_{\mathbf{u}}\left[\norm{\left(\frac{1}{N-T+1}\sum_{t=T}^{N}\mathbf{x}_t\mathbf{x}_t'\right)^{-1}}_F^2\right] \nonumber\\
	&\leq mT	 \mathbb{E}_{\mathbf{u}}\left[\norm{\left(\frac{1}{N-T+1}\sum_{t=T}^{N}\mathbf{x}_t\mathbf{x}_t'\right)^{-1}}_2^2\right]=\frac{mT}{\left(\min(\boldsymbol{\sigma}^{\cdot 2})\right)^2}.\nonumber
\end{align}
We know that $\frac{1}{(N-T+1)}\left(\sum_{t=T}^{N}\mathbf{x}_t\mathbf{x}_t'\right)^{-1}$ is the unbiased estimator of the covariance matrix \cite{tsiligkaridis2013covariance}. Therefore, we have 
\begin{align}
	& 	\mathbb{E}_{\mathbf{u}}\left[\norm{\left(\sum_{t=T}^{N}\mathbf{x}_t\mathbf{x}_t'\right)^{-1}}_F^2\right] \leq mT	\mathbb{E}_{\mathbf{u}}\left[ \norm{\left(\sum_{t=T}^{N}\mathbf{x}_t\mathbf{x}_t'\right)^{-1}}_2^2\right]=\frac{mT}{(N-T+1)^2\left(\min(\boldsymbol{\sigma}^{\cdot 2})\right)^2}.\nonumber
\end{align}
We bound $\mathbb{E}_{\boldsymbol{\zeta}}\left[\mathbb{E}_{\mathbf{u}}\left[\norm{\sum_{t=T}^{N}\boldsymbol{\zeta}_t\mathbf{x}_t'}_F^2\right]\right]$ as follows:
\begin{align}
	\mathbb{E}_{\boldsymbol{\zeta}}\left[\mathbb{E}_{\mathbf{u}}\left[\norm{\sum_{t=T}^{N}\boldsymbol{\zeta}_t\mathbf{x}_t'}_F^2\right]\right]\leq \sum_{t=T}^{N}\mathbb{E}_{\boldsymbol{\zeta}}\left[\mathbb{E}_{\mathbf{u}}\left[\norm{\boldsymbol{\zeta}_t}_2^2\norm{\mathbf{x}_t'}_2^2\right]\right]\leq\left(N-T+1\right)pmT\max(\boldsymbol{\sigma}_\zeta^{\cdot 2})\max(\boldsymbol{\sigma}^{\cdot 2}).\nonumber
\end{align}
We bound $\mathbb{E}_{\mathbf{u}}\left[\norm{\sum_{t=T}^{N}\mathbf{C}\mathbf{A}^{T-1}\mathbf{h}_{t-T+1}\mathbf{x}_t'}_F^2\right]$ as follows:
\begin{align}
	&\mathbb{E}_{\mathbf{u}}\left[\norm{\sum_{t=T}^{N}\mathbf{C}\mathbf{A}^{T-1}\mathbf{h}_{t-T+1}\mathbf{x}_t'}_F^2\right] \leq \sum_{t=T}^{N}\norm{\mathbf{C}}_F^2\norm{\mathbf{A}^{T-1}}_F^2 \mathbb{E}_{\mathbf{u}}\left[\norm{\mathbf{x}_t'}_2^2\right]\mathbb{E}_{\mathbf{u}}\left[\norm{\mathbf{h}_{t-T+1}}_2^2\right]\nonumber\\
	&\leq \sum_{t=T}^{N}n^2\ell\norm{\mathbf{C}}_F^2\rho(\mathbf{A})^{2(T-1)} \mathbb{E}_{\mathbf{u}}\left[\norm{\mathbf{x}_t'}_2^2\right]\mathbb{E}_{\mathbf{u}}\left[\norm{\mathbf{h}_{t-T+1}}_2^2\right]\nonumber\\
	&\overset{(a)}{\leq} n^2\ell mT\norm{\mathbf{C}}_F^2\rho(\mathbf{A})^{2(T-1)}\left(\max(\boldsymbol{\sigma}^{\cdot 2})\right)^2\sum_{t=T}^{N} \left(m\max(\boldsymbol{\sigma}^{\cdot 2})\norm{\mathbf{B}}_F^2\right)\nonumber\\
	&+n^2\ell mT\norm{\mathbf{C}}_F^2\rho(\mathbf{A})^{2(T-1)}\left(\max(\boldsymbol{\sigma}^{\cdot 2})\right)^2\sum_{t=T}^{N} \left(n^2 \ell \rho(\mathbf{A})^2 \norm{\mathbf{h}_0}_2^2+\frac{n^2\ell m\max(\boldsymbol{\sigma}^{\cdot 2})\rho(\mathbf{A})^2\norm{\mathbf{B}}_F^2}{1-\rho(\mathbf{A})^{2}}\right)
	\nonumber\\
	&\leq n^2\ell \norm{\mathbf{C}}_F^2\rho(\mathbf{A})^{2(T-1)}m^2T(\max(\boldsymbol{\sigma}^{\cdot 2}))^3(N-T+1)\norm{\mathbf{B}}_F^2\nonumber\\
	&+n^4\ell^2mT \norm{\mathbf{C}}_F^2\rho(\mathbf{A})^{2T}\left(\max(\boldsymbol{\sigma}^{\cdot 2})\right)^2(N-T+1)\times\underbrace{\left[\norm{\mathbf{h}_0}_2^2+\frac{ m\max(\boldsymbol{\sigma}^{\cdot 2})\norm{\mathbf{B}}_F^2}{1-\rho(\mathbf{A})^{2}}\right]}_{\iota},
\end{align}
where in (a) we use \eqref{eq:hidd} to bound $\norm{\mathbf{h}_{t-T+1}}_2^2$.
We use the above inequalities to simplify \eqref{eq:distance} as follows:
\begin{align}
	&\mathbb{E}_{\boldsymbol{\zeta}}\left[\mathbb{E}_{\mathbf{u}}\left[\norm{\hat{\boldsymbol{\Theta}}_T-\boldsymbol{\Theta}_T}_F^2\right]\right]\nonumber\\
	&\leq  \mathbb{E}_{\boldsymbol{\zeta}}\left[\mathbb{E}_{\mathbf{u}}\left[\norm{\sum_{t=T}^{N}\boldsymbol{\zeta}_t\mathbf{x}_t'}_F^2\right]\right]\mathbb{E}_{\boldsymbol{\zeta}}\left[\mathbb{E}_{\mathbf{u}}\left[\norm{\left(\sum_{t=T}^{N}\mathbf{x}_t\mathbf{x}_t'\right)^{-1}}_F^2\right]\right]
	\nonumber\\
	&+\mathbb{E}_{\mathbf{u}}\left[\norm{\sum_{t=T}^{N}\mathbf{C}\mathbf{A}^{T-1}\mathbf{h}_{t-T+1}\mathbf{x}_t'}_F^2\right]\mathbb{E}_{\mathbf{u}}\left[\norm{\left(\sum_{t=T}^{N}\mathbf{x}_t\mathbf{x}_t'\right)^{-1}}_F^2\right]\nonumber\\
	&\leq \frac{p m^2T^2 \max(\boldsymbol{\sigma}_\zeta^{\cdot 2})\max(\boldsymbol{\sigma}^{\cdot 2})}{(N-T+1)\left(\min(\boldsymbol{\sigma}^{\cdot 2})\right)^2} + \frac{n^2\ell \norm{\mathbf{C}}_F^2\rho(\mathbf{A})^{2(T-1)}m^3T^2(\max(\boldsymbol{\sigma}^{\cdot 2}))^3\norm{\mathbf{B}}_F^2 }{(N-T+1)\left(\min(\boldsymbol{\sigma}^{\cdot 2})\right)^2}\nonumber\\
	&+ \frac{n^4\ell^2 m^2T^2\rho(\mathbf{A})^{2T}\norm{\mathbf{C}}_F^2\left(\max(\boldsymbol{\sigma}^{\cdot 2})\right)^2\iota }{(N-T+1)\left(\min(\boldsymbol{\sigma}^{\cdot 2})\right)^2}=\chi_N^2.\nonumber
\end{align}

\begin{lemma}
	For an arbitrary $\mu$-strongly convex function $ f(\mathbf{x})$ with an $L$-Lipschitz continuous gradient, we have 
	\begin{align}
		&\langle \nabla f(\mathbf{x})-\nabla f(\mathbf{y}), \mathbf{x}-\mathbf{y} \rangle \geq \mu \norm{\mathbf{x}-\mathbf{y}}_2^2,\label{eq:str}\\
		& \langle \nabla f(\mathbf{x})-\nabla f(\mathbf{y}), \mathbf{x}-\mathbf{y} \rangle \geq \frac{1}{L} \norm{\nabla f(\mathbf{x})-\nabla f(\mathbf{y})}_2^2,\label{eq:lip}
	\end{align}
	where \eqref{eq:str} follows from \cite[eq. 2.1.11]{nesterov1998introductory} and \eqref{eq:lip} follows from \cite[eq. 2.1.8]{nesterov1998introductory}. From the Bergstrom's inequality, we have
	\begin{align}
		& 2\norm{\mathbf{x}}_2^2+2\norm{\mathbf{y}}_2^2\geq \norm{\mathbf{x}+\mathbf{y}}_2^2.\label{eq:sum}
	\end{align}

\end{lemma}

\subsection{Proof of Theorem \ref{th:offline}}\label{pr:offline}
Let us assume that the initial state of the system is denoted by $\mathbf{h}_0$. Each iteration of the offline SGD is as follows: 
\begin{align}
	\hat{\boldsymbol{\Theta}}_{\tau,T}=\hat{\boldsymbol{\Theta}}_{\tau-1,T}-\eta(\hat{\boldsymbol{\Theta}}_{\tau-1,T}\mathbf{x}_t-\mathbf{y}_t)\mathbf{x}_t',\nonumber
\end{align}
where $t$ is randomly chosen from $\{T,T+1,\dots,N\}$ with uniform probability. We let the first-order solution obtained from \eqref{opt:lim} be denoted by $\hat{\boldsymbol{\Theta}}_{T}$ and the ground truth solution is represented by $\boldsymbol{\Theta}_{T}$. The difference between $\hat{\boldsymbol{\Theta}}_{T}$ and the ground truth solution is denoted by $\boldsymbol{\nu}$ and defined as $\boldsymbol{\nu}=\hat{\boldsymbol{\Theta}}_{T}-\boldsymbol{\Theta}_{T}$. In Theorem \ref{th:prob}, the Frobenius norm of the difference is bounded as $\mathbb{E}_{\mathbf{u}}[\mathbb{E}_{\boldsymbol{\zeta}}[\norm{\boldsymbol{\nu}}]]_F^2\leq \chi_N^2$, where $N$ is the batch size. Let $\boldsymbol{\omega}_\tau$ denote the difference between $\hat{\boldsymbol{\Theta}}_{\tau,T}$ (in $\tau^{\text{th}}$ iteration) and $\hat{\boldsymbol{\Theta}}_T$ as $\boldsymbol{\omega}_\tau=\hat{\boldsymbol{\Theta}}_{\tau,T}-\hat{\boldsymbol{\Theta}}_{T}$. Based on the definition of $\boldsymbol{\omega}_\tau$, the update rule for $\boldsymbol{\omega}_\tau$ is $\boldsymbol{\omega}_{\tau+1}=\boldsymbol{\omega}_\tau-\eta(\hat{\boldsymbol{\Theta}}_{\tau,T}\mathbf{x}_t-\mathbf{y}_t)\mathbf{x}_t'$.  We have
\begin{align}
	&\hat{\boldsymbol{\Theta}}_{\tau,T}\mathbf{x}_t-\mathbf{y}_t=\hat{\boldsymbol{\Theta}}_{\tau,T}\mathbf{x}_t-\boldsymbol{\Theta}_{T}\mathbf{x}_t-\boldsymbol{\zeta}_t-\mathbf{C}\mathbf{A}^{T-1}\mathbf{h}_{t-T+1}\nonumber\\
	&=\hat{\boldsymbol{\Theta}}_{\tau,T}\mathbf{x}_t-\hat{\boldsymbol{\Theta}}_{T}\mathbf{x}_t+\boldsymbol{\nu}\mathbf{x}_t-\boldsymbol{\zeta}_t-\mathbf{C}\mathbf{A}^{T-1}\mathbf{h}_{t-T+1}\nonumber\\
	&=(\hat{\boldsymbol{\Theta}}_{\tau,T}-\hat{\boldsymbol{\Theta}}_{T})\mathbf{x}_t+\boldsymbol{\nu}\mathbf{x}_t-\boldsymbol{\zeta}_t-\mathbf{C}\mathbf{A}^{T-1}\mathbf{h}_{t-T+1}=\boldsymbol{\omega}_\tau\mathbf{x}_t+\boldsymbol{\nu}\mathbf{x}_t-\boldsymbol{\zeta}_t-\mathbf{C}\mathbf{A}^{T-1}\mathbf{h}_{t-T+1}.\nonumber
\end{align}
We bound the optimality gap as follows:
\begin{align}
	&\norm{\boldsymbol{\omega}_{\tau+1}}_F^2={\norm{\boldsymbol{\omega}_\tau-\eta(\hat{\boldsymbol{\Theta}}_{\tau,T}\mathbf{x}_t-\mathbf{y}_t)\mathbf{x}_t'}_F^2}=\norm{\boldsymbol{\omega}_\tau}_F^2-2\eta \mathbf{x}_t'\boldsymbol{\omega}_\tau'(\hat{\boldsymbol{\Theta}}_{\tau,T}\mathbf{x}_t-\mathbf{y}_t)\nonumber\\
	&+\eta^2\norm{(\hat{\boldsymbol{\Theta}}_{\tau,T}\mathbf{x}_t-\mathbf{y}_t)\mathbf{x}_t'}_F^2.\nonumber
\end{align}
We take expectations with respect to $t$ and $\mathbf{u}$ and obtain
\begin{align}
	&\mathbb{E}_{\mathbf{u}}[\mathbb{E}_{t}[\norm{\boldsymbol{\omega}_{\tau+1}}_F^2]]=\norm{\boldsymbol{\omega}_{\tau}}_F^2-2\eta \mathbb{E}_{\mathbf{u}}[\mathbb{E}_{t}[\mathbf{x}_t'\boldsymbol{\omega}_\tau'(\hat{\boldsymbol{\Theta}}_{\tau,T}\mathbf{x}_t-\mathbf{y}_t)]]+\eta^2\mathbb{E}_{\mathbf{u}}[\mathbb{E}_{t}[\norm{(\hat{\boldsymbol{\Theta}}_{\tau,T}\mathbf{x}_t-\mathbf{y}_t)\mathbf{x}_t'}_F^2]].\label{eq:bes}
\end{align}
To simplify \eqref{eq:bes}, we obtain a lower-bound for $\mathbb{E}_{\mathbf{u}}[\mathbb{E}_{t}[\mathbf{x}_t'\boldsymbol{\omega}_\tau'(\hat{\boldsymbol{\Theta}}_{\tau,T}\mathbf{x}_t-\mathbf{y}_t)]]$. Using \eqref{eq:str}, we have:
\begin{align}
	&\text{Tr}\left[\mathbb{E}_{\mathbf{u}}\left[\mathbb{E}_{t}[(\hat{\boldsymbol{\Theta}}_{\tau,T}\mathbf{x}_t-\mathbf{y}_t)\mathbf{x}_t'-(\hat{\boldsymbol{\Theta}}_{T}\mathbf{x}_t-\mathbf{y}_t)\mathbf{x}_t']\right]\boldsymbol{\omega}_{\tau}'\right]\geq m\:T\min(\boldsymbol{\sigma}^{\cdot 2})\norm{\hat{\boldsymbol{\Theta}}_{\tau,T}-\hat{\boldsymbol{\Theta}}_{T}}_F^2,\nonumber\\
	\Rightarrow&\text{Tr}\left[\mathbb{E}_{\mathbf{u}}\left[\mathbb{E}_{t}[(\hat{\boldsymbol{\Theta}}_{\tau,T}\mathbf{x}_t-\mathbf{y}_t)\mathbf{x}_t']\boldsymbol{\omega}_{\tau}'\right]\right]\nonumber\\
	&-\text{Tr}\left[\mathbb{E}_{\mathbf{u}}\left[\mathbb{E}_{t}[(\hat{\boldsymbol{\Theta}}_{T}\mathbf{x}_t-(\hat{\boldsymbol{\Theta}}_{T}-\boldsymbol{\nu})\mathbf{x}_t-\boldsymbol{\zeta}_t-\mathbf{C}\mathbf{A}^{T-1}\mathbf{h}_{t-T+1})\mathbf{x}_t']\boldsymbol{\omega}_{\tau}'\right]\right]\nonumber\\
	&\geq m\:T\min(\boldsymbol{\sigma}^{\cdot 2})\norm{\hat{\boldsymbol{\Theta}}_{\tau,T}-\hat{\boldsymbol{\Theta}}_{T}}_F^2,\nonumber\\
	\Rightarrow&\text{Tr}\left[\mathbb{E}_{\mathbf{u}}\left[\mathbb{E}_{t}[(\hat{\boldsymbol{\Theta}}_{\tau,T}\mathbf{x}_t-\mathbf{y}_t)\mathbf{x}_t']\boldsymbol{\omega}_{\tau}'\right]\right]-\text{Tr}\left[\mathbb{E}_{\mathbf{u}}\left[\mathbb{E}_{t}[\boldsymbol{\nu}\mathbf{x}_t\mathbf{x}_t']\boldsymbol{\omega}_{\tau}'\right]\right]\nonumber\\
	&-\frac{1}{N-T+1}\sum_{t=2T}^{N}\text{Tr}[\mathbb{E}_{\mathbf{u}}[\mathbf{C}\mathbf{A}^{T-1}\mathbf{h}_{t-T+1}\mathbf{x}_{t}'\boldsymbol{\omega}_{\tau}']]\geq m\:T\min(\boldsymbol{\sigma}^{\cdot 2})\norm{\hat{\boldsymbol{\Theta}}_{\tau,T}-\hat{\boldsymbol{\Theta}}_{T}}_F^2,\nonumber\\
	\Rightarrow&\mathbb{E}_{\mathbf{u}}\left[\mathbb{E}_{t}[\mathbf{x}_t'\boldsymbol{\omega}_\tau'(\hat{\boldsymbol{\Theta}}_{\tau,T}\mathbf{x}_t-\mathbf{y}_t)]\right]\geq \text{Tr}\left[\mathbb{E}_{\mathbf{u}}\left[\mathbb{E}_{t}[\boldsymbol{\nu}\mathbf{x}_t\mathbf{x}_t']\boldsymbol{\omega}_{\tau}'\right]\right]+m\:T\min(\boldsymbol{\sigma}^{\cdot 2})\norm{\hat{\boldsymbol{\Theta}}_{\tau,T}-\hat{\boldsymbol{\Theta}}_{T}}_F^2,\label{eq:bes1}
\end{align}
In the above chain of inequalities, we have $\frac{1}{N-T+1}\sum_{t=2T}^{N}\text{Tr}[\mathbb{E}_{\mathbf{u}}[\mathbf{C}\mathbf{A}^{T-1}\mathbf{h}_{t-T+1}\mathbf{x}_{t}'\boldsymbol{\omega}_{\tau}']]=0$ similar to \eqref{eq:min}.
Before we bound $\mathbb{E}_{\mathbf{u}}[\mathbb{E}_{t}[\norm{(\hat{\boldsymbol{\Theta}}_{\tau,T}\mathbf{x}_t-\mathbf{y}_t)\mathbf{x}_t'}_F^2]]$, first, we demonstrate that $\mathbb{E}_{\mathbf{u}}[\norm{\hat{\boldsymbol{\Theta}}_{\tau,T}\mathbf{x}_t-\mathbf{y}_t}_2^2]$ has a Lipschitz continuous gradient as follows:
\begin{align}
	&\mathbb{E}_{\mathbf{u}}\left[\norm{(\hat{\boldsymbol{\Theta}}_{\tau,T}\mathbf{x}_t-\mathbf{y}_t)\mathbf{x}_t'-(\hat{\boldsymbol{\Theta}}_{T}\mathbf{x}_t-\mathbf{y}_t)\mathbf{x}_t'}_F^2\right] \leq \mathbb{E}_{\mathbf{u}}\left[ \norm{(\hat{\boldsymbol{\Theta}}_{\tau,T}\mathbf{x}_t-\hat{\boldsymbol{\Theta}}_{T}\mathbf{x}_t)\mathbf{x}'_t}_F^2 \right]     \nonumber\\
	& \leq\mathbb{E}_{\mathbf{u}}\left[ \norm{\hat{\boldsymbol{\Theta}}_{\tau,T}-\hat{\boldsymbol{\Theta}}_{T}}_F^2 \right] \mathbb{E}_{\mathbf{u}}\left[ \norm{\mathbf{x}_t\mathbf{x}'_t}_F^2 \right] \leq m^2\:T^2(\max(\boldsymbol{\sigma}^{\cdot 2}))^2\mathbb{E}_{\mathbf{u}}\left[ \norm{\hat{\boldsymbol{\Theta}}_{\tau,T}-\hat{\boldsymbol{\Theta}}_{T}}_F^2 \right],\label{eq:shir2}
\end{align}
where $m\:T\max(\boldsymbol{\sigma}^{\cdot 2})$ is the Lipschitz constant. We have
\begin{align}
	&\mathbb{E}_{\mathbf{u}}[\mathbb{E}_{t}[\norm{(\hat{\boldsymbol{\Theta}}_{\tau,T}\mathbf{x}_t-\mathbf{y}_t)\mathbf{x}_t'}_F^2]]\leq \mathbb{E}_{\mathbf{u}}[\mathbb{E}_{t}[\norm{(\hat{\boldsymbol{\Theta}}_{\tau,T}\mathbf{x}_t-\mathbf{y}_t)\mathbf{x}_t'-(\hat{\boldsymbol{\Theta}}_{T}\mathbf{x}_t-\mathbf{y}_t)\mathbf{x}_t'+(\hat{\boldsymbol{\Theta}}_{T}\mathbf{x}_t-\mathbf{y}_t)\mathbf{x}_t'}_F^2]]\nonumber\\
	&\overset{(a)}{\leq} 2\mathbb{E}_{\mathbf{u}}[\mathbb{E}_{t}[\norm{(\hat{\boldsymbol{\Theta}}_{T}\mathbf{x}_t-\mathbf{y}_t)\mathbf{x}_t'-(\hat{\boldsymbol{\Theta}}_{\tau,T}\mathbf{x}_t-\mathbf{y}_t)\mathbf{x}_t'}_F^2]]+2\mathbb{E}_{\mathbf{u}}[\mathbb{E}_{t}[\norm{(\hat{\boldsymbol{\Theta}}_{T}\mathbf{x}_t-\mathbf{y}_t)\mathbf{x}_t'}_F^2]]\nonumber\\
	&\overset{(b)}{\leq} 2m\:T\max(\boldsymbol{\sigma}^{\cdot 2}) \mathbb{E}_{\mathbf{u}}\left[\mathbb{E}_{t}[\mathbf{x}_t'\boldsymbol{\omega}_\tau'(\hat{\boldsymbol{\Theta}}_{\tau,T}\mathbf{x}_t-\mathbf{y}_t-\hat{\boldsymbol{\Theta}}_{T}\mathbf{x}_t+\mathbf{y}_t)]\right]
	\nonumber\\
	&  +2\mathbb{E}_{\mathbf{u}}\left[\mathbb{E}_t[\norm{(\boldsymbol{\nu}\mathbf{x}_t-\mathbf{C}\mathbf{A}^{T-1}\mathbf{h}_{t-T+1}-\boldsymbol{\zeta}_t)\mathbf{x}_t'}_F^2]\right]\nonumber\\
	&\leq 2m\:T\max(\boldsymbol{\sigma}^{\cdot 2}) \mathbb{E}_{\mathbf{u}}\left[\mathbb{E}_{t}[\mathbf{x}_t'\boldsymbol{\omega}_\tau'(\hat{\boldsymbol{\Theta}}_{\tau,T}\mathbf{x}_t-\mathbf{y}_t)]\right]-2m\:T\max(\boldsymbol{\sigma}^{\cdot 2}) \mathbb{E}_{\mathbf{u}}\left[\mathbb{E}_{t}[\mathbf{x}_t'\boldsymbol{\omega}_\tau'(\hat{\boldsymbol{\Theta}}_{T}\mathbf{x}_t-\mathbf{y}_t)]\right]
	\nonumber\\
	&  +2\mathbb{E}_{\mathbf{u}}\left[\mathbb{E}_t[\norm{(\boldsymbol{\nu}\mathbf{x}_t-\mathbf{C}\mathbf{A}^{T-1}\mathbf{h}_{t-T+1}-\boldsymbol{\zeta}_t)\mathbf{x}_t'}_F^2]\right]\nonumber\\
	&\leq 2m\:T\max(\boldsymbol{\sigma}^{\cdot 2}) \mathbb{E}_{\mathbf{u}}\left[\mathbb{E}_{t}[\mathbf{x}_t'\boldsymbol{\omega}_\tau'(\hat{\boldsymbol{\Theta}}_{\tau,T}\mathbf{x}_t-\mathbf{y}_t)]\right]\nonumber\\
	& -2m\:T\max(\boldsymbol{\sigma}^{\cdot 2}) \mathbb{E}_{\mathbf{u}}\left[\mathbb{E}_{t}[\mathbf{x}_t'\boldsymbol{\omega}_\tau'(\boldsymbol{\nu}\mathbf{x}_t-\mathbf{C}\mathbf{A}^{T-1}\mathbf{h}_{t-T+1}-\boldsymbol{\zeta}_t)]\right]
	\nonumber\\
	&  +2\mathbb{E}_{\mathbf{u}}\left[\mathbb{E}_t[\norm{(\boldsymbol{\nu}\mathbf{x}_t-\mathbf{C}\mathbf{A}^{T-1}\mathbf{h}_{t-T+1}-\boldsymbol{\zeta}_t)\mathbf{x}_t'}_F^2]\right]\nonumber\\
	& \leq 2m\:T\max(\boldsymbol{\sigma}^{\cdot 2}) \mathbb{E}_{\mathbf{u}}\left[\mathbb{E}_{t}[\mathbf{x}_t'\boldsymbol{\omega}_\tau'(\hat{\boldsymbol{\Theta}}_{\tau,T}\mathbf{x}_t-\mathbf{y}_t)]\right]\nonumber\\
	& -2m\:T\max(\boldsymbol{\sigma}^{\cdot 2}) \mathbb{E}_{\mathbf{u}}\left[\mathbb{E}_{t}[\mathbf{x}_t'\boldsymbol{\omega}_\tau'\boldsymbol{\nu}\mathbf{x}_t]\right]
	\nonumber\\
	&  +2\mathbb{E}_{\mathbf{u}}\left[\mathbb{E}_t[\norm{(\boldsymbol{\nu}\mathbf{x}_t-\mathbf{C}\mathbf{A}^{T-1}\mathbf{h}_{t-T+1}-\boldsymbol{\zeta}_t)\mathbf{x}_t'}_F^2]\right],\nonumber
\end{align}
where (a) follows from \eqref{eq:sum} and (b) follows from \eqref{eq:lip}. We simplify \eqref{eq:bes} using \eqref{eq:bes1} and the above inequality as follows:
\begin{align}
	& \mathbb{E}_{\mathbf{u}}[\mathbb{E}_{t}[\norm{\boldsymbol{\omega}_{\tau+1}}_F^2]]\leq \norm{\boldsymbol{\omega}_{\tau}}_F^2-2\eta \mathbb{E}_{\mathbf{u}}[\mathbb{E}_{t}[\mathbf{x}_t'\boldsymbol{\omega}_\tau'(\hat{\boldsymbol{\Theta}}_{\tau,T}\mathbf{x}_t-\mathbf{y}_t)]]\nonumber\\
	&+2\eta^2m\:T\max(\boldsymbol{\sigma}^{\cdot 2}) \mathbb{E}_{\mathbf{u}}\left[\mathbb{E}_{t}[\mathbf{x}_t'\boldsymbol{\omega}_\tau'(\hat{\boldsymbol{\Theta}}_{\tau,T}\mathbf{x}_t-\mathbf{y}_t)]\right]\nonumber\\
	&+2\eta^2 \mathbb{E}_{\mathbf{u}}\left[\mathbb{E}_t[\norm{(\boldsymbol{\nu}\mathbf{x}_t-\mathbf{C}\mathbf{A}^{T-1}\mathbf{h}_{t-T+1}-\boldsymbol{\zeta}_t)\mathbf{x}_t'}_F^2]\right] -2\eta^2m\:T\max(\boldsymbol{\sigma}^{\cdot 2}) \mathbb{E}_{\mathbf{u}}\left[\mathbb{E}_{t}[\mathbf{x}_t'\boldsymbol{\omega}_\tau'\boldsymbol{\nu}\mathbf{x}_t]\right]
	\nonumber\\
	&\leq \norm{\boldsymbol{\omega}_{\tau}}_F^2+(-2\eta+2\eta^2m\:T\max(\boldsymbol{\sigma}^{\cdot 2}))\mathbb{E}_{\mathbf{u}}[\mathbb{E}_{t}[\mathbf{x}_t'\boldsymbol{\omega}_\tau'(\hat{\boldsymbol{\Theta}}_{\tau,T}\mathbf{x}_t-\mathbf{y}_t)]]\nonumber\\
	&+2\eta^2 \mathbb{E}_{\mathbf{u}}\left[\mathbb{E}_t[\norm{(\boldsymbol{\nu}\mathbf{x}_t-\mathbf{C}\mathbf{A}^{T-1}\mathbf{h}_{t-T+1}-\boldsymbol{\zeta}_t)\mathbf{x}_t'}_F^2]\right]-2\eta^2m\:T\max(\boldsymbol{\sigma}^{\cdot 2}) \mathbb{E}_{\mathbf{u}}\left[\mathbb{E}_{t}[\mathbf{x}_t'\boldsymbol{\omega}_\tau'\boldsymbol{\nu}\mathbf{x}_t]\right]\nonumber\\
	& \overset{(a)}{\leq} \norm{\boldsymbol{\omega}_{\tau}}_F^2(1-2\eta m\:T\min(\boldsymbol{\sigma}^{\cdot 2})+2\eta^2m^2\:T^2\min(\boldsymbol{\sigma}^{\cdot 2})\max(\boldsymbol{\sigma}^{\cdot 2}) )\nonumber   \\
	&+(-2\eta+2\eta^2m\:T\max(\boldsymbol{\sigma}^{\cdot 2}))\text{Tr}\left[\mathbb{E}_{\mathbf{u}}\left[\mathbb{E}_{t}[\boldsymbol{\nu}\mathbf{x}_t\mathbf{x}_t']\boldsymbol{\omega}_{\tau}'\right]\right]\nonumber   \\
	&-2\eta^2m\:T\max(\boldsymbol{\sigma}^{\cdot 2}) \mathbb{E}_{\mathbf{u}}\left[\mathbb{E}_{t}[\mathbf{x}_t'\boldsymbol{\omega}_\tau'\boldsymbol{\nu}\mathbf{x}_t]\right]\nonumber   \\
	&+2\eta^2 \mathbb{E}_{\mathbf{u}}\left[\mathbb{E}_t[\norm{(\boldsymbol{\nu}\mathbf{x}_t-\mathbf{C}\mathbf{A}^{T-1}\mathbf{h}_{t-T+1}-\boldsymbol{\zeta}_t)\mathbf{x}_t'}_F^2]\right]\nonumber\\
	& \overset{(b)}{\leq}  \norm{\boldsymbol{\omega}_{\tau}}_F^2(1-2\eta m\:T\min(\boldsymbol{\sigma}^{\cdot 2})+2\eta^2m^2\:T^2\min(\boldsymbol{\sigma}^{\cdot 2})\max(\boldsymbol{\sigma}^{\cdot 2}))\nonumber   \\
	&+2\eta mT\max(\boldsymbol{\sigma}^{\cdot 2})\norm{\boldsymbol{\nu}}_F\norm{\boldsymbol{\omega}_{0}}_F +2\eta^2 \mathbb{E}_{\mathbf{u}}\left[\mathbb{E}_t[\norm{(\boldsymbol{\nu}\mathbf{x}_t-\mathbf{C}\mathbf{A}^{T-1}\mathbf{h}_{t-T+1}-\boldsymbol{\zeta}_t)\mathbf{x}_t'}_F^2]\right]\nonumber\\
	&+2\eta^2m^2\:T^2(\max(\boldsymbol{\sigma}^{\cdot 2}))^2 |\text{Tr}[\boldsymbol{\nu}\boldsymbol{\omega}_{\tau}']|\nonumber\\
	& \leq  \norm{\boldsymbol{\omega}_{\tau}}_F^2(1-2\eta m\:T\min(\boldsymbol{\sigma}^{\cdot 2})+2\eta^2m^2\:T^2\min(\boldsymbol{\sigma}^{\cdot 2})\max(\boldsymbol{\sigma}^{\cdot 2}))\nonumber   \\
	&+(\eta mT\max(\boldsymbol{\sigma}^{\cdot 2})+\eta^2m^2\:T^2(\max(\boldsymbol{\sigma}^{\cdot 2}))^2)(\norm{\boldsymbol{\nu}}_F^2+\norm{\boldsymbol{\omega}_{0}}_F^2) \nonumber   \\
	&+2\eta^2 \mathbb{E}_{\mathbf{u}}\left[\mathbb{E}_t[\norm{(\boldsymbol{\nu}\mathbf{x}_t-\mathbf{C}\mathbf{A}^{T-1}\mathbf{h}_{t-T+1}-\boldsymbol{\zeta}_t)\mathbf{x}_t'}_F^2]\right],\label{eq:bes2}
\end{align}
where (a) follows from \eqref{eq:bes1} and the fact that $\mathbb{E}_{\mathbf{u}}\left[\mathbb{E}_{t}[\mathbf{x}_t'\boldsymbol{\omega}_\tau'\boldsymbol{\nu}\mathbf{x}_t]\right]\leq mT\max(\boldsymbol{\sigma}^{\cdot 2})\text{Tr}[\boldsymbol{\nu}\boldsymbol{\omega}_\tau]$. Furthermore, (b) follows from Von Neumann's trace inequality and also the assumption that $-2\eta+2\eta^2m\:T\max(\boldsymbol{\sigma}^{\cdot 2})<0$. To simplify \eqref{eq:bes2}, we consider the following bound:
\begin{align}
	&      \mathbb{E}_{\boldsymbol{\zeta}}\left[\mathbb{E}_{\mathbf{u}}\left[\mathbb{E}_t[\norm{(\boldsymbol{\nu}\mathbf{x}_t-\mathbf{C}\mathbf{A}^{T-1}\mathbf{h}_{t-T+1}-\boldsymbol{\zeta}_t)\mathbf{x}_t'}_F^2]\right] \right]  \nonumber\\
	& \leq \norm{\boldsymbol{\nu}}_F^2\mathbb{E}_{\mathbf{u}}[\norm{\mathbf{x}_t\mathbf{x}_t'}_F^2]+\norm{\mathbf{C}}_F^2\norm{\mathbf{A}^{T-1}}_F^2\mathbb{E}_{\mathbf{u}}[\norm{\mathbf{h}_{t-T+1}}_2^2]\mathbb{E}_{\mathbf{u}}[\norm{\mathbf{x}_{t}}_2^2]+\mathbb{E}_{\mathbf{u}}[\norm{\mathbf{x}_{t}}_2^2]\mathbb{E}_{\boldsymbol{\zeta}}[\norm{\boldsymbol{\zeta}_t}_2^2]\nonumber\\
	&\leq m^2\:T^2(\max(\boldsymbol{\sigma}^{\cdot 2}))^2\chi_N^2+n^2mT\max(\boldsymbol{\sigma}^{\cdot 2})\ell \rho(\mathbf{A})^{2(T-1)}\gamma\norm{\mathbf{C}}_F^2+pmT\max(\boldsymbol{\sigma}^{\cdot 2})\max(\boldsymbol{\sigma}_\zeta^{\cdot 2}).\nonumber
\end{align}
Using the above inequality, we continue \eqref{eq:bes2} as follow:
\begin{align}
	&\mathbb{E}_{\boldsymbol{\zeta}}\left[\mathbb{E}_{\mathbf{u}}[\mathbb{E}_{t}[\norm{\boldsymbol{\omega}_{\tau+1}}_F^2]]\right]\leq  \norm{\boldsymbol{\omega}_{\tau}}_F^2(1-2\eta m\:T\min(\boldsymbol{\sigma}^{\cdot 2})+2\eta^2m^2\:T^2\min(\boldsymbol{\sigma}^{\cdot 2})\max(\boldsymbol{\sigma}^{\cdot 2}))\nonumber\\
	& +(\eta mT\max(\boldsymbol{\sigma}^{\cdot 2})+\eta^2m^2\:T^2(\max(\boldsymbol{\sigma}^{\cdot 2}))^2)(\norm{\boldsymbol{\nu}}_F^2+\norm{\boldsymbol{\omega}_{0}}_F^2)\nonumber   \\
	& +2\eta^2\:\Big[m^2\:T^2(\max(\boldsymbol{\sigma}^{\cdot 2}))^2\chi_N^2+n^2mT\max(\boldsymbol{\sigma}^{\cdot 2})\ell \rho(\mathbf{A})^{2(T-1)}\gamma\norm{\mathbf{C}}_F^2\nonumber\\
	& +pmT\max(\boldsymbol{\sigma}^{\cdot 2})\max(\boldsymbol{\sigma}_\zeta^{\cdot 2})\Big].\label{eq:tir}
\end{align}
We observe if $1-2\eta m\:T\min(\boldsymbol{\sigma}^{\cdot 2})+2\eta^2m^2\:T^2\min(\boldsymbol{\sigma}^{\cdot 2})\max(\boldsymbol{\sigma}^{\cdot 2})\leq 1$, we obtain $\mathbb{E}_{t}[\mathbb{E}_{\mathbf{u}}[\norm{\boldsymbol{\omega}_{\tau+1}}_F^2]]\leq \norm{\boldsymbol{\omega}_{\tau}}_F^2$. Therefore, $\eta$ should satisfy
\begin{align}
	\eta \leq \frac{1}{m\:T\max(\boldsymbol{\sigma}^{\cdot 2})}.\label{eq:step}
\end{align}
To make the additive constant terms in \eqref{eq:tir} small enough, we can choose $\eta$ close to zero although very small $\eta$ makes the coefficient of $\mathbb{E}_{\mathbf{u}}\left[\mathbb{E}_{\boldsymbol{\zeta}}\left[\norm{\boldsymbol{\omega}_\tau}_F^2\right]\right]$ close to one and makes the convergence rate slow. Therefore, more iterations are required to reach a certain error.

From recursion, we bound the Frobenius norm distance between the solution of \eqref{opt:lim} and the initial point as follows:
\begin{align}
	&\mathbb{E}_{\mathbf{t}}\left[\mathbb{E}_{\mathbf{u}}\left[\mathbb{E}_{\boldsymbol{\zeta}}\left[{\norm{\boldsymbol{\omega}_{\tau}}_F^2}\right]\right]\right]<\norm{\boldsymbol{\omega}_0}_F^2\left(1-2\eta m\:T\min(\boldsymbol{\sigma}^{\cdot 2})+2\eta^2m^2\:T^2\min(\boldsymbol{\sigma}^{\cdot 2})\max(\boldsymbol{\sigma}^{\cdot 2})\right)^{\tau}\nonumber\\
	&+\:\sum_{i=0}^{\tau-1}\:\left(1-2\eta m\:T\min(\boldsymbol{\sigma}^{\cdot 2})+2\eta^2m^2\:T^2\min(\boldsymbol{\sigma}^{\cdot 2})\max(\boldsymbol{\sigma}^{\cdot 2})\right)^i\:\Big[2\eta^2m^2\:T^2(\max(\boldsymbol{\sigma}^{\cdot 2}))^2\chi_N^2\nonumber\\
	&+2n^2\eta^2mT\max(\boldsymbol{\sigma}^{\cdot 2})\ell \rho(\mathbf{A})^{2(T-1)}\gamma\norm{\mathbf{C}}_F^2+2\eta^2pmT\max(\boldsymbol{\sigma}^{\cdot 2})\max(\boldsymbol{\sigma}_\zeta^{\cdot 2})\nonumber\\
	&+(\eta mT\max(\boldsymbol{\sigma}^{\cdot 2})+\eta^2m^2\:T^2(\max(\boldsymbol{\sigma}^{\cdot 2}))^2)(\norm{\boldsymbol{\nu}}_F^2+\norm{\boldsymbol{\omega}_{0}}_F^2)\Big]\nonumber\\
	&<\norm{\boldsymbol{\omega}_0}_F^2\left(1-2\eta m\:T\min(\boldsymbol{\sigma}^{\cdot 2})+2\eta^2m^2\:T^2\min(\boldsymbol{\sigma}^{\cdot 2})\max(\boldsymbol{\sigma}^{\cdot 2})\right)^\tau\nonumber\\
	&+\frac{2\eta^2m^2\:T^2(\max(\boldsymbol{\sigma}^{\cdot 2}))^2\chi_N^2+(\eta mT\max(\boldsymbol{\sigma}^{\cdot 2})+\eta^2m^2\:T^2(\max(\boldsymbol{\sigma}^{\cdot 2}))^2)(\chi_N^2+\norm{\boldsymbol{\omega}_{0}}_F^2)}{1-2\eta m\:T\min(\boldsymbol{\sigma}^{\cdot 2})+2\eta^2m^2\:T^2\min(\boldsymbol{\sigma}^{\cdot 2})\max(\boldsymbol{\sigma}^{\cdot 2})}\nonumber\\
	&+\frac{2n^2\eta^2mT\max(\boldsymbol{\sigma}^{\cdot 2})\ell \rho(\mathbf{A})^{2(T-1)}\gamma\norm{\mathbf{C}}_F^2+2\eta^2pmT\max(\boldsymbol{\sigma}^{\cdot 2})\max(\boldsymbol{\sigma}_\zeta^{\cdot 2})}{1-2\eta m\:T\min(\boldsymbol{\sigma}^{\cdot 2})+2\eta^2m^2\:T^2\min(\boldsymbol{\sigma}^{\cdot 2})\max(\boldsymbol{\sigma}^{\cdot 2})}.\nonumber
\end{align}
The additive terms can become as small as desired by adjusting $\eta$. With smaller step-size, the proposed offline SGD requires additional iterations to reach a certain neighborhood of the ground truth solution. Suppose that the distance between $\hat{\boldsymbol{\Theta}}_{\tau,T}$ and the ground truth $\boldsymbol{\Theta}_T$ is denoted by $\boldsymbol{\phi}_\tau$. Then, we bound $\mathbb{E}_{\mathbf{t}}\left[\mathbb{E}_{\mathbf{u}}\left[\mathbb{E}_{\boldsymbol{\zeta}}[\boldsymbol{\phi}_\tau]\right]\right]$ as follows:
\begin{align}
	&\mathbb{E}_{\mathbf{t}}\left[\mathbb{E}_{\mathbf{u}}\left[\mathbb{E}_{\boldsymbol{\zeta}}\left[\norm{\boldsymbol{\phi}_\tau}_F^2\right]\right]\right]=\mathbb{E}_{\mathbf{t}}\left[\mathbb{E}_{\mathbf{u}}\left[\mathbb{E}_{\boldsymbol{\zeta}}\left[\norm{\hat{\boldsymbol{\Theta}}_{\tau,T}-\boldsymbol{\Theta}_{T}}_F^2\right]\right]\right]\nonumber\\
	&=\mathbb{E}_{\mathbf{t}}\left[\mathbb{E}_{\mathbf{u}}\left[\mathbb{E}_{\boldsymbol{\zeta}}\left[\norm{\hat{\boldsymbol{\Theta}}_{\tau,T}-\hat{\boldsymbol{\Theta}}_{T}+\hat{\boldsymbol{\Theta}}_{T}-\boldsymbol{\Theta}_{T}}_F^2\right]\right]\right]\nonumber\\
	&=\mathbb{E}_{\mathbf{t}}\left[\mathbb{E}_{\mathbf{u}}\left[\mathbb{E}_{\boldsymbol{\zeta}}\left[\norm{\boldsymbol{\omega}_\tau+\boldsymbol{\nu}}_F^2\right]\right]\right]\leq \mathbb{E}_{\mathbf{t}}\left[\mathbb{E}_{\mathbf{u}}\left[\mathbb{E}_{\boldsymbol{\zeta}}\left[\norm{\boldsymbol{\omega}_\tau}_F^2\right]\right]\right]+\mathbb{E}_{\mathbf{u}}\left[\mathbb{E}_{\boldsymbol{\zeta}}\left[\norm{\boldsymbol{\nu}}_F^2\right]\right]\nonumber\\
	&<\norm{\boldsymbol{\omega}_0}_F^2\left(1-2\eta m\:T\min(\boldsymbol{\sigma}^{\cdot 2})+2\eta^2m^2\:T^2\min(\boldsymbol{\sigma}^{\cdot 2})\max(\boldsymbol{\sigma}^{\cdot 2})\right)^\tau\nonumber\\
	&+\underbracea{\frac{2\eta^2m^2\:T^2(\max(\boldsymbol{\sigma}^{\cdot 2}))^2\chi_N^2+(\eta mT\max(\boldsymbol{\sigma}^{\cdot 2})+\eta^2m^2\:T^2(\max(\boldsymbol{\sigma}^{\cdot 2}))^2)(\chi_N^2+\norm{\boldsymbol{\omega}_{0}}_F^2)}{1-2\eta m\:T\min(\boldsymbol{\sigma}^{\cdot 2})+2\eta^2m^2\:T^2\min(\boldsymbol{\sigma}^{\cdot 2})\max(\boldsymbol{\sigma}^{\cdot 2})}}\nonumber\\
	&\underbracebd{+\frac{2n^2\eta^2mT\max(\boldsymbol{\sigma}^{\cdot 2})\ell \rho(\mathbf{A})^{2(T-1)}\gamma\norm{\mathbf{C}}_F^2+2\eta^2pmT\max(\boldsymbol{\sigma}^{\cdot 2})\max(\boldsymbol{\sigma}_\zeta^{\cdot 2})}{1-2\eta m\:T\min(\boldsymbol{\sigma}^{\cdot 2})+2\eta^2m^2\:T^2\min(\boldsymbol{\sigma}^{\cdot 2})\max(\boldsymbol{\sigma}^{\cdot 2})}}_{\Delta_N}+\chi_N^2.\label{eq:mash}
\end{align}
\subsection{Proof of Theorem \ref{th:one}}
The convergence analysis of the online SGD is similar to that for the offline SGD. The reason is that the expectation of gradients in \eqref{eq:offlineup} and \eqref{eq:onlineup} are equal. We assume that the initial state of the system $\mathbf{h}_0$ is known. Each iteration of the online SGD iteration is implemented as follows: 
\begin{align}
	\hat{\boldsymbol{\Theta}}_{t,T}=\hat{\boldsymbol{\Theta}}_{t-1,T}-\eta(\hat{\boldsymbol{\Theta}}_{t-1,T}\mathbf{x}_t-\mathbf{y}_t)\mathbf{x}_t',\nonumber
\end{align}
where $t$ corresponds to the newest input-output pair and $t\geq 2T$. Similar to the proof given in Appendix \ref{pr:offline}, $\boldsymbol{\omega}_t$ denotes the difference between $\hat{\boldsymbol{\Theta}}_{t,T}$ and $\hat{\boldsymbol{\Theta}}_T$, i.e., $\boldsymbol{\omega}_t=\hat{\boldsymbol{\Theta}}_{t,T}-\hat{\boldsymbol{\Theta}}_T$, where $\hat{\boldsymbol{\Theta}}_T$ is the minimizer of $\frac{1}{2t}\sum_{i=1}^{t}\norm{\mathbf{y}_i- \hat{\boldsymbol{\Theta}}_T\mathbf{x}_i}_2^2$. The major difference between the online SGD and the offline version is that the online SGD has access to  $t$ samples. 
Therefore, the  gap due to limited batch size characterized in Theorem \ref{th:prob} is bounded as $\norm{\boldsymbol{\nu}}_F\leq \chi_t$, where $\boldsymbol{\nu}=\hat{\boldsymbol{\Theta}}_T-\boldsymbol{\Theta}_T$. Based on this, $\hat{\boldsymbol{\Theta}}_{t,T}\mathbf{x}_t-\mathbf{y}_t=\hat{\boldsymbol{\Theta}}_{t,T}\mathbf{x}_t-\boldsymbol{\Theta}_{T}\mathbf{x}_t-\boldsymbol{\zeta}_t-\mathbf{C}\mathbf{A}^{T-1}\mathbf{h}_{t-T+1}=\hat{\boldsymbol{\Theta}}_{t,T}\mathbf{x}_t-\hat{\boldsymbol{\Theta}}_T\mathbf{x}_t+\boldsymbol{\nu}\mathbf{x}_t-\boldsymbol{\zeta}_t-\mathbf{C}\mathbf{A}^{T-1}\mathbf{h}_{t-T+1}=(\hat{\boldsymbol{\Theta}}_{t,T}-\hat{\boldsymbol{\Theta}}_{T})\mathbf{x}_t+\boldsymbol{\nu}\mathbf{x}_t-\boldsymbol{\zeta}_t-\mathbf{C}\mathbf{A}^{T-1}\mathbf{h}_{t-T+1}=\boldsymbol{\omega}_t\mathbf{x}_t+\boldsymbol{\nu}\mathbf{x}_t-\boldsymbol{\zeta}_t-\mathbf{C}\mathbf{A}^{T-1}\mathbf{h}_{t-T+1}$. As the expected gradient of the online SGD is equal to that of the offline SGD, with the step size \eqref{eq:step}, we bound the optimality gap reduction in iteration $t+1$ as follows:
\begin{align}
	&\mathbb{E}_{\mathbf{u}}\left[\mathbb{E}_{\boldsymbol{\zeta}}\left[\norm{\boldsymbol{\phi}_t}_F^2\right]\right]
	<\norm{\boldsymbol{\omega}_0}_F^2\left(1-2\eta m\:T\min(\boldsymbol{\sigma}^{\cdot 2})+2\eta^2m^2\:T^2\min(\boldsymbol{\sigma}^{\cdot 2})\max(\boldsymbol{\sigma}^{\cdot 2})\right)^\tau\nonumber\\
	&+\underbracea{\frac{2\eta^2m^2\:T^2(\max(\boldsymbol{\sigma}^{\cdot 2}))^2\chi_t^2+(\eta mT\max(\boldsymbol{\sigma}^{\cdot 2})+\eta^2m^2\:T^2(\max(\boldsymbol{\sigma}^{\cdot 2}))^2)(\chi_t^2+\norm{\boldsymbol{\omega}_{0}}_F^2)}{1-2\eta m\:T\min(\boldsymbol{\sigma}^{\cdot 2})+2\eta^2m^2\:T^2\min(\boldsymbol{\sigma}^{\cdot 2})\max(\boldsymbol{\sigma}^{\cdot 2})}}\nonumber\\
	&\underbracebd{+\frac{2n^2\eta^2mT\max(\boldsymbol{\sigma}^{\cdot 2})\ell \rho(\mathbf{A})^{2(T-1)}\gamma\norm{\mathbf{C}}_F^2+2\eta^2pmT\max(\boldsymbol{\sigma}^{\cdot 2})\max(\boldsymbol{\sigma}_\zeta^{\cdot 2})}{1-2\eta m\:T\min(\boldsymbol{\sigma}^{\cdot 2})+2\eta^2m^2\:T^2\min(\boldsymbol{\sigma}^{\cdot 2})\max(\boldsymbol{\sigma}^{\cdot 2})}}_{\Delta_t}+\chi_t^2.\nonumber
\end{align}
\subsection{\blue{Proof of Theorem \ref{th:linsys}}}\label{sec:lineq}
When $\mathbf{A}$, $\mathbf{B}$ and $\mathbf{C}$ are in Brunovsky canonical form, the transfer function \eqref{eq:trans1} is uniquely realized by the state-space representation \cite{hardt2018gradient}. Due to \cite[Lemma B.1]{hardt2018gradient}, one can observe that the LHS of \eqref{eq:trans1} uniquely realizes the RHS of \eqref{eq:trans1} and vice versa. 

We notice that the linear system in \eqref{eq:linsys} is consistent. The reason is that \eqref{eq:equality} holds in all frequencies including the chosen $z_k$. Using $\mathbf{W}_{k}$, we rewrite \eqref{eq:linsys} in the form of a linear system of equations as follows:\allowdisplaybreaks
\begin{align}
	&\underbrace{\begin{bmatrix}
			-(\boldsymbol{\Theta}_T\:\boldsymbol{\vartheta}_1+\mathbf{E}_{z_1,T})'z_1^{n-1} &    \cdots & -(\boldsymbol{\Theta}_T\:\boldsymbol{\vartheta}_1+\mathbf{E}_{z_1,T})' & \mathbf{W}_1'\\
			-(\boldsymbol{\Theta}_T\:\boldsymbol{\vartheta}_2+\mathbf{E}_{z_2,T})'z_2^{n-1} &   \cdots & -(\boldsymbol{\Theta}_T\:\boldsymbol{\vartheta}_2+\mathbf{E}_{z_2,T})' & \mathbf{W}_2'\\
			\vdots  &\ddots & \vdots & \vdots \\
			-(\boldsymbol{\Theta}_T\:\boldsymbol{\vartheta}_{n+pnm}+\mathbf{E}_{z_{n+pnm},T})'z_{n+pnm}^{n-1} &  \cdots & -(\boldsymbol{\Theta}_T\:\boldsymbol{\vartheta}_{n+pnm}+\mathbf{E}_{z_{n+pnm},T})' & \mathbf{W}_{n+pnm}'
	\end{bmatrix}}_{\boldsymbol{\Psi}_T}
	\nonumber\\&\underbrace{\begin{bmatrix}
			a_1\\
			\vdots\\
			a_n\\
			\mathbf{C}'
	\end{bmatrix}}_{\boldsymbol{\varphi}}
	=\underbrace{\begin{bmatrix}
			(	\boldsymbol{\Theta}_T\:\boldsymbol{\vartheta}_1z_1^n)'\\
			(\boldsymbol{\Theta}_T\:\boldsymbol{\vartheta}_2z_2^n)'\\
			\vdots\\
			(\boldsymbol{\Theta}_T\:\boldsymbol{\vartheta}_{n+pnm}z_{n+pnm}^n)'\\
	\end{bmatrix}}_{\boldsymbol{\kappa}}\label{eq:lin}.
\end{align}
In $\boldsymbol{\Psi}_T$, in each block row, all $\{(\boldsymbol{\Theta}_T\:\boldsymbol{\vartheta}_k+\mathbf{E}_{z_k,T})'z_k^{n-v}\}_{v=1}^n$ blocks are $m\times p$ matrix blocks except the last one, $\mathbf{W}_k$, which is an $m\times nm$ matrix block. The first $n$ blocks, $\{(\boldsymbol{\Theta}_T\:\boldsymbol{\vartheta}_k+\mathbf{E}_{z_k,T})'z_k^{n-v}\}_{v=1}^n$, in each row are multiplied by $\{a_v\}_{v=1}^n$ elements in $\boldsymbol{\varphi}$, and $\mathbf{W}_k$ is multiplied by $\mathbf{C}'$. In the above equation, $\boldsymbol{\kappa}$ is made of $n+pnm$ block matrices stacked vertically, each with the dimension $m\times p$. We note that \eqref{eq:lin} is made of $n+pnm$ block rows. The $k^{\text{th}}$ block row is as follows:
\begin{align}
	-\sum_{v=1}^{n}a_v(\boldsymbol{\Theta}_T\:\boldsymbol{\vartheta}_k+\mathbf{E}_{z_k,T})'z_k^{n-v}+\mathbf{W}_k'\mathbf{C}'=(\boldsymbol{\Theta}_T\:\boldsymbol{\vartheta}_kz_k^n)',\hspace{1cm}\forall k \in \{1,n+nmp\}.\label{eq:kth}
\end{align}
The above block equation is $m\times p$. In the $(i,j)^{\text{th}}$ equation of block equation \eqref{eq:kth}, the variables are $\{a_i\}_{i=1}^{n}$ and $\{c_{j,i+vm}\}_{v=0}^{n-1}$.  The $(i,j)^{\text{th}}$ equation from the above block is as follows:
\begin{align}
	-\sum_{v=1}^{n}a_v[(\boldsymbol{\Theta}_T\:\boldsymbol{\vartheta}_k+\mathbf{E}_{z_k,T})']_{ij}z_k^{n-v}+\sum_{v=0}^{n-1}c_{j,i+vm}z_k^{v}=[(\boldsymbol{\Theta}_T\:\boldsymbol{\vartheta}_kz_k^n)']_{ij},\label{eq:one}
\end{align}
We stack \eqref{eq:one} for all frequencies, $\{z_k\}_{k=1}^{n+nmp}$, and form the following linear system:
\allowdisplaybreaks
\begin{equation}
	\!\begin{aligned}
		&
		\underbracea{\Bigggg[\begin{matrix}
				-[(\boldsymbol{\Theta}_T\:\boldsymbol{\vartheta}_1+\mathbf{E}_{z_1,T})']_{ij}z_1^{n-1}&-[(\boldsymbol{\Theta}_T\:\boldsymbol{\vartheta}_1+\mathbf{E}_{z_1,T})']_{ij}z_1^{n-2}&	\dots\\
				-[(\boldsymbol{\Theta}_T\:\boldsymbol{\vartheta}_2+\mathbf{E}_{z_2,T})']_{ij}z_2^{n-1}&-[(\boldsymbol{\Theta}_T\:\boldsymbol{\vartheta}_2+\mathbf{E}_{z_2,T})']_{ij}z_2^{n-2}&	\dots\\
				\vdots&	\vdots& \vdots \\
				-[(\boldsymbol{\Theta}_T\:\boldsymbol{\vartheta}_{n+pnm}+\mathbf{E}_{z_{n+pnm},T})']_{ij}z_{n+nmp}^{n-1}&-[(\boldsymbol{\Theta}_T\:\boldsymbol{\vartheta}_{n+pnm}+\mathbf{E}_{z_{n+pnm},T})']_{ij}z_{n+nmp}^{n-2}&	\dots
		\end{matrix}}\\
		&\qquad\qquad
		\underbracebd{\begin{matrix}
				-[(\boldsymbol{\Theta}_T\:\boldsymbol{\vartheta}_1+\mathbf{E}_{z_1,T})']_{ij}& 1&z_1&	\dots& z_1^{n-1}\\
				-[(\boldsymbol{\Theta}_T\:\boldsymbol{\vartheta}_2+\mathbf{E}_{z_2,T})']_{ij}& 1&z_2&	\dots& z_2^{n-1}\\
				\vdots&\vdots&\vdots&\vdots&\vdots\\
				-[(\boldsymbol{\Theta}_T\:\boldsymbol{\vartheta}_{n+nmp}+\mathbf{E}_{z_{n+nmp,T}})']_{ij}& 1&z_{n+nmp}&	\dots& z_{n+nmp}^{n-1}
			\end{matrix}\Bigggg]}_{\mathbf{R}_{ij}}\left[\begin{matrix}\{a_v\}_{v=1}^{n}\\	\{c_{j,i+vm}\}_{v=0}^{n-1} \end{matrix}\right]\\
		&=\underbrace{\left[\begin{matrix}[(\boldsymbol{\Theta}_T\:\boldsymbol{\vartheta}_1z_1^n)']_{ij}\\ [(\boldsymbol{\Theta}_T\:\boldsymbol{\vartheta}_2z_2^n)']_{ij}\\\vdots\\ [(\boldsymbol{\Theta}_T\:\boldsymbol{\vartheta}_{n+nmp}z_{n+nmp}^n)']_{ij} \end{matrix}\right]}_{\mathbf{r}_{ij}}.\label{eq:decom1}
	\end{aligned}
\end{equation}
We notice that the exponential of a frequency, denoted by $z^v$, is different from the frequency $z$. We investigate the linear dependency of the columns in the above matrix, where the $k^{\text{th}}$ row corresponds to the $k^{\text{th}}$ frequency. We note that in $\mathbf{R}_{ij}$, each element in a row incorporates a particular set of frequencies. The set of frequencies embedded in $\boldsymbol{\vartheta}_k$ is $\{z_k^{-1},\dots,z_k^{-T+1}\}$. Given that $\mathbf{E}_{z_k,T}$ is negligible, we list the sets of frequencies incorporated in the coefficients of each unknown in a row of \eqref{eq:decom1} as follows:
\begin{subequations}
	\begin{align}
		&	c_{j,i+vm} \longrightarrow z_k^{v},\hspace{6.1cm} v\in\{0,\dots,n-1\},\label{eq:tir1}\\
		&	a_v \longrightarrow \{z_k^{n-v-1},z_k^{n-v-2},\dots,z_k^{n-v-T+1}\},\hspace{1cm} v\in\{1,\dots,n\}\label{eq:tir2}.
	\end{align}
\end{subequations}
The set of frequencies in the coefficient of $a_1$ is $\{z_k^{n-2},z_k^{n-3},\dots,z_k^{n-T}\}$. The set of frequencies in the coefficient of $a_2$ is $\{z_k^{n-3},z_k^{n-4},\dots,z_k^{n-1-T}\}$. We observe that the frequency $z_k^{n-1-T}$ does not exist in the set of frequencies in the coefficient of $a_1$. Similarly, the coefficient of $a_v$ incorporates frequency $z_k^{n-v-T+1}$ that does not exist in the coefficients of $\{a_1,\dots,a_{v-1}\}$. Furthermore, we observe from \eqref{eq:tir1} that the coefficient of $c_{j,i+vm}$ is $z_k^{v}$. The frequencies in the coefficients of $\{c_{j,i+vm}\}_{v=0}^{n-1}$ are separate and do not overlap. Given that $T \geq n+1$, we obtain $n-v-T+1<0$. There is at least one frequency $z^{n-v-T+1}$ in the coefficient of $a_v$ that do not appear in the frequencies incorporated in the coefficients of $\{c_{j,i+vm}\}_{v=0}^{n-1}$ since $n-v-T+1<0$. 

Let us suppose that the columns of the coefficient matrix in \eqref{eq:decom1} are linearly dependent. In this case, we have:
\begin{align}
	-\sum_{v=1}^{n}[(\boldsymbol{\Theta}_T\:\boldsymbol{\vartheta}_k+\mathbf{E}_{z_k,T})']_{ij}z_k^{n-v}\alpha_a^v+\sum_{v=1}^{n}z_k^{v-1}\alpha_c^v=0,\label{eq:red}
\end{align}
where $\alpha_a^v$ and $\alpha_c^v$ are given coefficients to the $v^\text{th}$ and $v+n^\text{th},v\in\{1,\dots,n\}$, columns of $\mathbf{R}_{ij}$, respectively, to ensure the linear dependency of columns. We consider two possibilities:
\begin{enumerate}
	\item Suppose that each $\alpha_a^v$ is zero, i.e., $\{\alpha_a^v\}_{v=1}^n=\{0\}$. Then, \eqref{eq:red} implies that $\sum_{v=1}^{n}z_k^{v-1}\alpha_c^v=0$. When at least two different $\alpha_c^v$ are non-zero, $\sum_{v=1}^{n}z_k^{v-1}\alpha_c^v=0$ yields a polynomial of $z$, in which $z_k$ is a root. However, we note that $z_k$ is arbitrarily chosen by us. It is impossible that a polynomial of $z$ with a finite degree has infinite roots. Therefore, by contradiction, we conclude that the columns $v+1$ to $v+n$ of $\mathbf{R}_{ij}$ are linearly independent of each other.
	\item Suppose that at least one $\alpha_a^v$, $\forall v\in\{1,\dots,n\}$, is not zero. Then, we consider  $[(\boldsymbol{\Theta}_T\:\boldsymbol{\vartheta}_k+\mathbf{E}_{z_k,T})']_{ij}z_k^{n-v}\alpha_a^v$. We choose the largest $v$ such that $\alpha_a^v\neq 0$. From \eqref{eq:tir2}, we see that the frequency $z_k^{n-v-T+1}$ appears only in $[(\boldsymbol{\Theta}_T\:\boldsymbol{\vartheta}_k+\mathbf{E}_{z_k,T})']_{ij}z_k^{n-v}\alpha_a^v$ and it does not exist in the other elements of the $k^{\text{th}}$ row, which are given in \eqref{eq:red}. The reason is that we picked the largest $v$ and $n-v-T+1$ is the least exponent for $z_k$ in \eqref{eq:red} . Hence, due to its uniqueness, $z_k^{n-v-T+1}$ cannot be removed by the linear combination of different elements in \eqref{eq:red}. Based on this fact, \eqref{eq:red} is always at least a polynomial of $z^{n-v-T+1}$, in which an arbitrary $z_k$ is a root. It is impossible that a polynomial of $z$ with a finite degree has infinite roots. Therefore, by contradiction, we conclude that the columns of $\mathbf{R}_{ij}$ are linearly independent of each other.
\end{enumerate}

\begin{algorithm}[H]
	\SetAlgoLined
	\textbf{Initialization}: $\boldsymbol{\Gamma}_T=\mathbf{0}_{mp(n+nmp)\times (n+nmp)}$, $\boldsymbol{\varkappa}=\mathbf{0}_{mp(n+nmp)\times 1}$\\
	\textbf{Input}: \eqref{eq:linsys} for all $(i,j)\in \{(i,j)\mid i\in\{1,\dots,m\},j\in\{1,\dots,p\}\}$\\
	\textbf{Output}: $\boldsymbol{\Gamma}_T$, $\boldsymbol{\varkappa}$\\
	\For{all $(i,j)\in \{(i,j)\mid i\in\{1,\dots,m\},j\in\{1,\dots,p\}\}$}{
		Find $\mathbf{R}_{ij}$ and $\mathbf{r}_{ij}$ from \eqref{eq:decom1}\\
		$\boldsymbol{\Gamma}_T((n+nmp)((i-1)p+j-1)+1:(n+nmp)((i-1)p+j),1:n)=\mathbf{R}_{ij}(:,1:n)$\\
		\For{$v\in\{0,\dots,n-1\}$} {
			$\boldsymbol{\Gamma}_T((n+nmp)((i-1)p+j-1)+1:(n+nmp)((i-1)p+j),n+(j-1)nm+i+vm)=\mathbf{R}_{ij}(:,n+v+1)$
		}
		$\boldsymbol{\varkappa}((n+nmp)((i-1)p+j-1)+1:(n+nmp)((i-1)p+j),1)=\mathbf{r}_{ij}$
	}
	\textbf{Return}: $\boldsymbol{\Gamma}_T$, $\boldsymbol{\varkappa}$
	\caption{Transforming \eqref{eq:linsys} to the standard form of linear system of equations} 	\label{al:stand}
\end{algorithm}
Since the coefficient matrix for the linear system \eqref{eq:decom1} is full-rank, one can identify both blocks of variables, i.e., $\{a_v\}_{v=1}^{n}$ and $\{c_{j,i+vm}\}_{v=0}^{n-1}$. By changing $i$ in the range $\{1,\dots, m\}$ and $j$ in the range $\{1,\dots, p\}$, one can identify all elements of $\mathbf{C}$. We can rewrite \eqref{eq:lin} in the standard form of a linear system of equations $\boldsymbol{\Gamma}_T\boldsymbol{\varrho}=\boldsymbol{\varkappa}$ by using Algorithm \ref{al:stand}. Algorithm \ref{al:stand} stacks \eqref{eq:lin} for different $i$ and $j$ one after the other, while it includes all coefficients for all elements of $\mathbf{C}$.

In general, $\boldsymbol{\Psi}_T$ is a tall matrix. The system $\boldsymbol{\Psi}_T\boldsymbol{\varrho}=\boldsymbol{\kappa}$ can be solved by different numerical approaches (e.g., \cite{razaviyayn2019linearly,liu2016accelerated,ma2015convergence}). 

We observe that both $\boldsymbol{\Psi}_T$ and $\boldsymbol{\kappa}$ are linearly parameterized by $\boldsymbol{\Theta}_T\:\boldsymbol{\vartheta}_k$. In addition, $\mathbf{E}_{z_k,T}=\sum_{t=T}^{\infty} z_k^{-t}\mathbf{C}\mathbf{A}^{t-1}\mathbf{B}$ appears in $\boldsymbol{\Psi}_T$. 
Let us represent the matrix $\boldsymbol{\Psi}_T$ by $\boldsymbol{\Psi}_{t,T}$ when 1) $\boldsymbol{\Psi}_T$ is parameterized by $\hat{\boldsymbol{\Theta}}_{t,T}\boldsymbol{\vartheta}_k$; and 2) $\mathbf{E}_{z_k,T}=\mathbf{0}$. Moreover, $\boldsymbol{\kappa}_t$ is parameterized by  $\hat{\boldsymbol{\Theta}}_{t,T}$. Based on this, we demonstrate that the linear convergence of $\hat{\boldsymbol{\Theta}}_{t,T}$ enforces the solution of $\boldsymbol{\Psi}_{t,T}\hat{\boldsymbol{\varphi}}_t=\boldsymbol{\kappa}_t$ to linearly converge to the ground truth values. We assume $|z_k|=1$. We have:
\begin{equation*}
	\hspace{-3.2cm}
	\begin{rcases}
		&\boldsymbol{\Psi}_{t,T}\hat{\boldsymbol{\varphi}}_t=\boldsymbol{\kappa}_t \\
		&\boldsymbol{\Psi}_{T}\boldsymbol{\varphi}=\boldsymbol{\kappa}
	\end{rcases}
	\rightarrow\boldsymbol{\Psi}_{t,T}\hat{\boldsymbol{\varphi}}_t-\boldsymbol{\Psi}_{T}\boldsymbol{\varphi}=\boldsymbol{\kappa}_t-\boldsymbol{\kappa}.\nonumber
\end{equation*}
We expand the above equation as follows:
\begin{align}
	&	\boldsymbol{\Psi}_{t,T}\hat{\boldsymbol{\varphi}}_t-\boldsymbol{\Psi}_{t,T}\boldsymbol{\varphi}+\boldsymbol{\Psi}_{t,T}\boldsymbol{\varphi}-\boldsymbol{\Psi}_{T}\boldsymbol{\varphi}=\boldsymbol{\kappa}_t-\boldsymbol{\kappa},\nonumber\\
	\Rightarrow	&\norm{\boldsymbol{\Psi}_{t,T}(\hat{\boldsymbol{\varphi}}_t-\boldsymbol{\varphi})+(\boldsymbol{\Psi}_{t,T}-\boldsymbol{\Psi}_{T})\boldsymbol{\varphi}}_F^2=\norm{\boldsymbol{\kappa}_t-\boldsymbol{\kappa}}_F^2,\nonumber\\
	\overset{(a)}{\Rightarrow}	&\norm{\boldsymbol{\Psi}_{t,T}(\hat{\boldsymbol{\varphi}}_t-\boldsymbol{\varphi})}_F^2-\norm{(\boldsymbol{\Psi}_{t,T}-\boldsymbol{\Psi}_{T})\boldsymbol{\varphi}}_F^2\leq \norm{\boldsymbol{\kappa}_t-\boldsymbol{\kappa}}_F^2,\nonumber\\
	\overset{(b)}{\Rightarrow}	&(\min(\delta(\boldsymbol{\Psi}_{t,T})))^2\norm{\hat{\boldsymbol{\varphi}}_t-\boldsymbol{\varphi}}_F^2\leq \norm{\boldsymbol{\kappa}_t-\boldsymbol{\kappa}}_F^2+\norm{(\boldsymbol{\Psi}_{t,T}-\boldsymbol{\Psi}_{T})}_F^2\norm{\boldsymbol{\varphi}}_F^2,\nonumber\\
	\overset{(c)}{\Rightarrow}	
	&\mathbb{E}_{\mathbf{u}}[\mathbb{E}_{\boldsymbol{\zeta}}[\norm{\hat{\boldsymbol{\varphi}}_t-\boldsymbol{\varphi}}_F^2]]\nonumber\\
	&\leq \mathbb{E}_{\mathbf{u}}\left[\mathbb{E}_{\boldsymbol{\zeta}}\left[\frac{m(T-1)(n+nmp)\norm{\hat{\boldsymbol{\Theta}}_{t,T}-\hat{\boldsymbol{\Theta}}_T}_F^2+nm(n+nmp)(T-1)\norm{\hat{\boldsymbol{\Theta}}_{t,T}-\hat{\boldsymbol{\Theta}}_T}_F^2\norm{\boldsymbol{\varphi}}_F^2}{(\min\left(\delta(\boldsymbol{\Psi}_{t,T}))\right)^2}\right]\right]\nonumber\\
	& + \frac{n^3(n+nmp)\ell \norm{\mathbf{C}}_F^2\norm{\mathbf{B}}_F^2\rho(\mathbf{A})^{2(T-1)}\norm{\boldsymbol{\varphi}}_F^2}{(1-\rho(\mathbf{A})^2)(\min\left(\delta(\boldsymbol{\Psi}_{t,T}))\right)^2}\nonumber\\
	&=\frac{(m(T-1)(n+nmp)+nm(n+nmp)(T-1)\norm{\boldsymbol{\varphi}}_F^2)}{(\min(\delta(\boldsymbol{\Psi}_{t,T})))^2}\mathbb{E}_{\mathbf{u}}\left[\mathbb{E}_{\boldsymbol{\zeta}}\left[\norm{\hat{\boldsymbol{\Theta}}_{t,T}-\hat{\boldsymbol{\Theta}}_T}_F^2\right]\right]\nonumber\\
	&+\underbrace{\frac{n^3(n+nmp)\ell \norm{\mathbf{C}}_F^2\norm{\mathbf{B}}_F^2\rho(\mathbf{A})^{2(T-1)}\norm{\boldsymbol{\varphi}}_F^2}{(1-\rho(\mathbf{A})^2)(\min\left(\delta(\boldsymbol{\Psi}_{t,T}))\right)^2}}_{s_1}.\label{eq:ineq}
\end{align}
In the above series of inequalities, we have $(a)$ due to the triangle difference inequality. 
Moreover, (b) follows due to the fact that 1) $\norm{(\boldsymbol{\Psi}_{t,T}-\boldsymbol{\Psi}_{T})\boldsymbol{\varphi}}_F^2\leq \norm{(\boldsymbol{\Psi}_{t,T}-\boldsymbol{\Psi}_{T})}_F^2\norm{\boldsymbol{\varphi}}_F^2$; and 2) $(\min(\delta(\boldsymbol{\Psi}_{t,T})))^2\norm{\hat{\boldsymbol{\varphi}}_t-\boldsymbol{\varphi}}_F^2\leq \norm{\boldsymbol{\Psi}_{t,T}(\hat{\boldsymbol{\varphi}}_t-\boldsymbol{\varphi})}_F^2$, where $\min(\delta(\boldsymbol{\Psi}_{t,T}))$ is the minimum non-zero singular value of $\boldsymbol{\Psi}_{t,T}$. In (c), we notice that $\boldsymbol{\Psi}_{t,T}-\boldsymbol{\Psi}_{T}$ is only a function of $\hat{\boldsymbol{\Theta}}_{t,T}-\boldsymbol{\Theta}_T$ and $\{\sum_{t=T}^{\infty} z_k^{-t}\mathbf{C}\mathbf{A}^{t-1}\mathbf{B}\}_{k=1}^{n+nmp}$ as follows:
\begin{align}
	\boldsymbol{\Psi}_{t,T}&-\boldsymbol{\Psi}_{T}=\underbrace{\Bigggg[\begin{matrix}
			-((\hat{\boldsymbol{\Theta}}_{t,T}-\boldsymbol{\Theta}_T)\:\boldsymbol{\vartheta}_1)'z_1^{n-1} &    \cdots & -((\hat{\boldsymbol{\Theta}}_{t,T}-\boldsymbol{\Theta}_T)\:\boldsymbol{\vartheta}_1)' & \boldsymbol{0}\\
			-((\hat{\boldsymbol{\Theta}}_{t,T}-\boldsymbol{\Theta}_T)\:\boldsymbol{\vartheta}_2)'z_2^{n-1} &   \cdots & -((\hat{\boldsymbol{\Theta}}_{t,T}-\boldsymbol{\Theta}_T)\:\boldsymbol{\vartheta}_2)' & \boldsymbol{0}\\
			\vdots  &\ddots & \vdots & \vdots \\
			-((\hat{\boldsymbol{\Theta}}_{t,T}-\boldsymbol{\Theta}_T)\:\boldsymbol{\vartheta}_{n+pnm})'z_{n+pnm}^{n-1} &  \cdots & -((\hat{\boldsymbol{\Theta}}_{t,T}-\boldsymbol{\Theta}_T)\:\boldsymbol{\vartheta}_{n+pnm})' & \boldsymbol{0}
		\end{matrix}\Bigggg]}_{\diamondsuit}\nonumber\\
	&+\underbrace{\Bigggg[\begin{matrix}
			(\sum_{t=T}^{\infty} z_1^{-t}\mathbf{C}\mathbf{A}^{t-1}\mathbf{B})'z_1^{n-1} &    \cdots & (\sum_{t=T}^{\infty} z_1^{-t}\mathbf{C}\mathbf{A}^{t-1}\mathbf{B})' & \boldsymbol{0}\\
			(\sum_{t=T}^{\infty} z_2^{-t}\mathbf{C}\mathbf{A}^{t-1}\mathbf{B})'z_2^{n-1} &   \cdots & (\sum_{t=T}^{\infty} z_2^{-t}\mathbf{C}\mathbf{A}^{t-1}\mathbf{B})' & \boldsymbol{0}\\
			\vdots  &\ddots & \vdots & \vdots \\
			(\sum_{t=T}^{\infty} z_{n+pnm}^{-t}\mathbf{C}\mathbf{A}^{t-1}\mathbf{B})'z_{n+pnm}^{n-1} &  \cdots & (\sum_{t=T}^{\infty} z_{n+pnm}^{-t}\mathbf{C}\mathbf{A}^{t-1}\mathbf{B})' & \boldsymbol{0}
		\end{matrix}\Bigggg]}_{\clubsuit},\nonumber
\end{align}
where 
$\diamondsuit$ can be factorized into $\hat{\boldsymbol{\Theta}}_{t,T}-\boldsymbol{\Theta}_T$ and a constant matrix whose norm is denoted by $nm(n+nmp)(T-1)$. When $|z|=1$, the norm of $\clubsuit$ is bounded using Lemma \ref{le:te} as follows:
\begin{align}
	\norm{\clubsuit}_F^2\leq \frac{n^3(n+nmp)\ell \norm{\mathbf{C}}_F^2\norm{\mathbf{B}}_F^2\rho(\mathbf{A})^{2(T-1)}}{1-\rho(\mathbf{A})^2}.\label{eq:res}
\end{align}
Due to its structure, $\boldsymbol{\kappa}_t-\boldsymbol{\kappa}$ can be factorized into $\hat{\boldsymbol{\Theta}}_{t,T}-\boldsymbol{\Theta}_T$ and a constant matrix whose norm is denoted by $m(T-1)(n+nmp)$. In (c), we decompose $\boldsymbol{\Psi}_{t,T}-\boldsymbol{\Psi}_t$ and $\boldsymbol{\kappa}_t-\boldsymbol{\kappa}$ and use the Cauchy–Schwarz inequality. 

Since we have already shown that by Algorithm \ref{al:comb1} (or Algorithm \ref{al:comb}) decreases $\norm{\hat{\boldsymbol{\Theta}}_{t,T}-\boldsymbol{\Theta}_T}_F^2$ exponentially, we observe from \eqref{eq:ineq} that  $\norm{\hat{\boldsymbol{\varphi}}_t-\boldsymbol{\varphi}}_2^2$ is enforced to be decreased at least  exponentially when $T$ is large enough to make $\norm{\triangle}_F^2$ very small. This concludes the linear convergence in expectation for the unknown parameters $\{a_i\}_{i=1}^n$ and $\{c_{i,j}\}_{i=1:p,j=1:mn}$ in $\hat{\boldsymbol{\varphi}}_t$ to $\boldsymbol{\varphi}$ when $\boldsymbol{\Psi}_{t,T}\hat{\boldsymbol{\varphi}}_t=\boldsymbol{\kappa}_{t}$ is solved in each iteration of Algorithm \ref{al:comb1} or Algorithm \ref{al:comb}. Since   $\norm{\hat{\boldsymbol{\varphi}}_{t}}_F=\norm{\hat{\boldsymbol{\varrho}}_{t}}_2$, Algorithm \ref{al:comb} linearly converges in expectation.

We substitute the upper-bound in \eqref{eq:samoff} for $\mathbb{E}_{\mathbf{u}}\Big[\mathbb{E}_{\boldsymbol{\zeta}}[\norm{\hat{\boldsymbol{\Theta}}_{t,T}-\hat{\boldsymbol{\Theta}}_T}_2^2]\Big]$, we find
\begin{align}\textstyle
	&\mathbb{E}_{\mathbf{u}}\left[\mathbb{E}_{\boldsymbol{\zeta}}\left[\norm{\hat{\boldsymbol{\varrho}}_t-\boldsymbol{\varrho}}_2^2\right]\right]\leq \underbrace{\frac{n^3(n+nmp)\ell \norm{\mathbf{C}}_F^2\norm{\mathbf{B}}_F^2\rho(\mathbf{A})^{2(T-1)}\norm{\boldsymbol{\varrho}}_F^2}{(1-\rho(\mathbf{A})^2)(\min(\delta(\boldsymbol{\Psi}_{t,T})))^2}}_{s_1}
	\nonumber\\
	&+\underbracea{\frac{(m(T-1)(n+nmp)+nm(n+nmp)(T-1)\norm{\boldsymbol{\varrho}}_F^2)}{(\min(\delta(\boldsymbol{\Psi}_{t,T})))^2}}\nonumber\\
	&\underbraceb{
		\Big[\frac{2\eta^2m^2\:T^2(\max(\boldsymbol{\sigma}^{\cdot 2}))^2\chi_t^2+2n^2\eta^2mT\max(\boldsymbol{\sigma}^{\cdot 2})\ell \rho(\mathbf{A})^{2(T-1)}\gamma\norm{\mathbf{C}}_F^2}{1-2\eta m\:T\min(\boldsymbol{\sigma}^{\cdot 2})+2\eta^2m^2\:T^2\min(\boldsymbol{\sigma}^{\cdot 2})\max(\boldsymbol{\sigma}^{\cdot 2})}}\nonumber\\
	&+\underbracebd{\frac{2\eta^2pmT\max(\boldsymbol{\sigma}^{\cdot 2})\max(\boldsymbol{\sigma}_\zeta^{\cdot 2})+(\eta mT\max(\boldsymbol{\sigma}^{\cdot 2})+\eta^2m^2\:T^2(\max(\boldsymbol{\sigma}^{\cdot 2}))^2)(\chi_t^2+\norm{\boldsymbol{\omega}_{0}}_F^2)}{1-2\eta m\:T\min(\boldsymbol{\sigma}^{\cdot 2})+2\eta^2m^2\:T^2\min(\boldsymbol{\sigma}^{\cdot 2})\max(\boldsymbol{\sigma}^{\cdot 2})}+\chi_t^2\Big]}_{s_2}\nonumber\\
	&+\frac{(m(T-1)(n+nmp)+nm(n+nmp)(T-1)\norm{\boldsymbol{\varrho}}_F^2)}{(\min(\delta(\boldsymbol{\Psi}_{t,T})))^2}\norm{\hat{\boldsymbol{\Theta}}_{0,T}-\hat{\boldsymbol{\Theta}}_{T}}_F^2\nonumber\\
	&\left(1-2\eta m\:T\min(\boldsymbol{\sigma}^{\cdot 2})+2\eta^2m^2\:T^2\min(\boldsymbol{\sigma}^{\cdot 2})\max(\boldsymbol{\sigma}^{\cdot 2})\right)^t.\label{eq:overal}
\end{align}
We conclude that the estimated unknown parameters $\{a_i\}_{i=1}^n$ and $\{c_{i,j}\}_{i=1:p,j=1:mn}$ linearly converge in expectation to $\boldsymbol{\varrho}$ as the obtained $\hat{\boldsymbol{\Theta}}_{t,T}$ by Algorithm \ref{al:comb} converges. As explained in Appendix \ref{pr:offline}, by adjusting the step-size, the batch size and the truncation length, we can make the neighborhood that Algorithm \ref{al:comb} converges to as small as desired. Therefore, the error of learning unknown parameters $\{a_i\}_{i=1}^n$ and $\{c_{i,j}\}_{i=1:p,j=1:mn}$, which is $\Upsilon=s_1+s_2$, can be made as small as desired by increasing $T$ and decreasing $\eta$. Similar to the above derivations, one can conclude that unknown parameters can be learned at a linear convergence rate (in expectation) if Algorithm \ref{al:comb1} is used instead of Algorithm \ref{al:comb}.

One can rewrite \eqref{eq:overal} as follows:
\begin{align}
	&\mathbb{E}_{\mathbf{u}}\left[\mathbb{E}_{\boldsymbol{\zeta}}\left[\norm{\hat{\boldsymbol{\varrho}}_t-\boldsymbol{\varrho}}_2^2\right]\right]\leq \Upsilon+l_2(l_1/n+1)\Big(nm(n+nmp)(T-1)\Big)\nonumber\\
	&\times\norm{\boldsymbol{\omega}_0}_F^2\left(1-2\eta m\:T\min(\boldsymbol{\sigma}^{\cdot 2})+2\eta^2m^2\:T^2\min(\boldsymbol{\sigma}^{\cdot 2})\max(\boldsymbol{\sigma}^{\cdot 2})\right)^t,\nonumber
\end{align}
where
\begin{align}
	l_1=\frac{1}{\norm{\boldsymbol{\varrho}}_F^2}, \hspace{1cm}l_2=\frac{\norm{\boldsymbol{\varrho}}_F^2}{(\min(\delta(\boldsymbol{\Psi}_{t,T})))^2},
\end{align}
and $\boldsymbol{\omega}_0=\hat{\boldsymbol{\Theta}}_{0,T}-\hat{\boldsymbol{\Theta}}_T$. Here, we analyze the iteration complexity when the step-size is as given in \eqref{eq:step}. Suppose $\epsilon>0$ such that $s_1+s_2 \leq \frac{\epsilon}{2}$. We take logarithm from \eqref{eq:overal} and rearrange as follows:
\begin{align}
	&\log\left(\frac{2\norm{\hat{\boldsymbol{\Theta}}_{0,T}-\hat{\boldsymbol{\Theta}}_T}_F^2(m(T-1)(n+nmp)+nm(n+nmp)(T-1)\norm{\boldsymbol{\varrho}}_F^2)}{\epsilon(\min(\delta(\boldsymbol{\Psi}_{t+1,T})))^2}\right)\nonumber\\
	&\leq t\log(\frac{1}{1-2\eta m\:T\min(\boldsymbol{\sigma}^{\cdot 2})+2\eta^2m^2\:T^2\min(\boldsymbol{\sigma}^{\cdot 2})\max(\boldsymbol{\sigma}^{\cdot 2})}).\nonumber
\end{align} 
Since $\log(\frac{1}{x})\geq 1- x$ when $0<x\leq 1$, we find
\begin{align}
	&\frac{1}{2\eta m\:T\min(\boldsymbol{\sigma}^{\cdot 2})-2\eta^2m^2\:T^2\min(\boldsymbol{\sigma}^{\cdot 2})\max(\boldsymbol{\sigma}^{\cdot 2})}\nonumber\\
	&\log\left(\frac{2\norm{\hat{\boldsymbol{\Theta}}_{0,T}-\hat{\boldsymbol{\Theta}}_T}_F^2(m(T-1)(n+nmp)+nm(n+nmp)(T-1)\norm{\boldsymbol{\varrho}}_F^2)}{\epsilon(\min(\delta(\boldsymbol{\Psi}_{t+1,T})))^2}\right)\leq t.\nonumber
\end{align}
Based on the above inequality, we find the computational complexity as follows: 
\begin{align}
	\mathcal{O}\left(\frac{1}{2\eta m\:T\min(\boldsymbol{\sigma}^{\cdot 2})-2\eta^2m^2\:T^2\min(\boldsymbol{\sigma}^{\cdot 2})\max(\boldsymbol{\sigma}^{\cdot 2})}\log(\frac{n^2\:m^2\:p\:T}{\epsilon})\right).\nonumber
\end{align}
\section{Extended Numerical Tests}\label{sec:exsim}

\subsection{Comparisons with Ho-Kalman Algorithm in \cite{oymak2019non}}
The Ho-Kalman algorithm in \cite{oymak2019non} estimates the Hankel matrix $\mathbf{H}$, which is built using the Markov parameters of the system. Consider that the last $mT/2$ columns of the Hankel matrix are denoted by $\mathbf{H}^+$, where $T=200$. The Ho-Kalman algorithm finds the rank-$n$-approximation of the Hankel matrix. Next, the rank-$n$-approximation, denoted by $\mathbf{L}$, is decomposed into the observability and controllability matrices. This decomposition is carried out using SVD. Therefore, if the rank-$n$-approximation of the Hankel matrix has an SVD decomposition like $\mathbf{L}=\mathbf{U}\boldsymbol{\Sigma}\mathbf{V}'$, the observability matrix is $\mathbf{O}=\mathbf{U}\boldsymbol{\Sigma}^{1/2}$ and the controllability matrix is $\mathbf{Q}=\boldsymbol{\Sigma}^{1/2}\mathbf{V}'$. Then, the estimated $\hat{\mathbf{C}}$ matrix is the first $p$ rows of the observability matrix. Furthermore, the estimated $\hat{\mathbf{A}}$ matrix is $(\hat{\mathbf{O}}'\hat{\mathbf{O}})^{-1}\hat{\mathbf{O}}\hat{\mathbf{H}}^+(\hat{\mathbf{Q}}'\hat{\mathbf{Q}})^{-1}\hat{\mathbf{Q}}$. We consider two MIMO systems for the comparisons. In the first system, the hidden state dimension is $20$, $m=4$, $n=5$, and $p=4$. In the second system, the hidden state dimension is $30$, $m=6$, $n=5$, and $p=6$.  Our numerical simulations confirm that if the ground truth Hankel matrix is given to the Ho-Kalman algorithm, the estimated matrices are not identical to the ground truth weight matrices. To help the Ho-Kalman algorithm to find the underlying weight matrices, we give the optimal transformation $\mathcal{T}$ to the Ho-Kalman algorithm such that $\mathbf{U}\boldsymbol{\Sigma}^{1/2}\mathcal{T}$ becomes the ground truth observability matrix, and $\mathcal{T}^{-1}\boldsymbol{\Sigma}^{1/2}\mathbf{V}'$ becomes the ground truth controllability matrix of the underlying system. Furthermore, we consider that the standard deviation of measurement noise is $0.1$.
\begin{figure}
	\centering
	\subfigure{\includegraphics[width=0.42\textwidth]{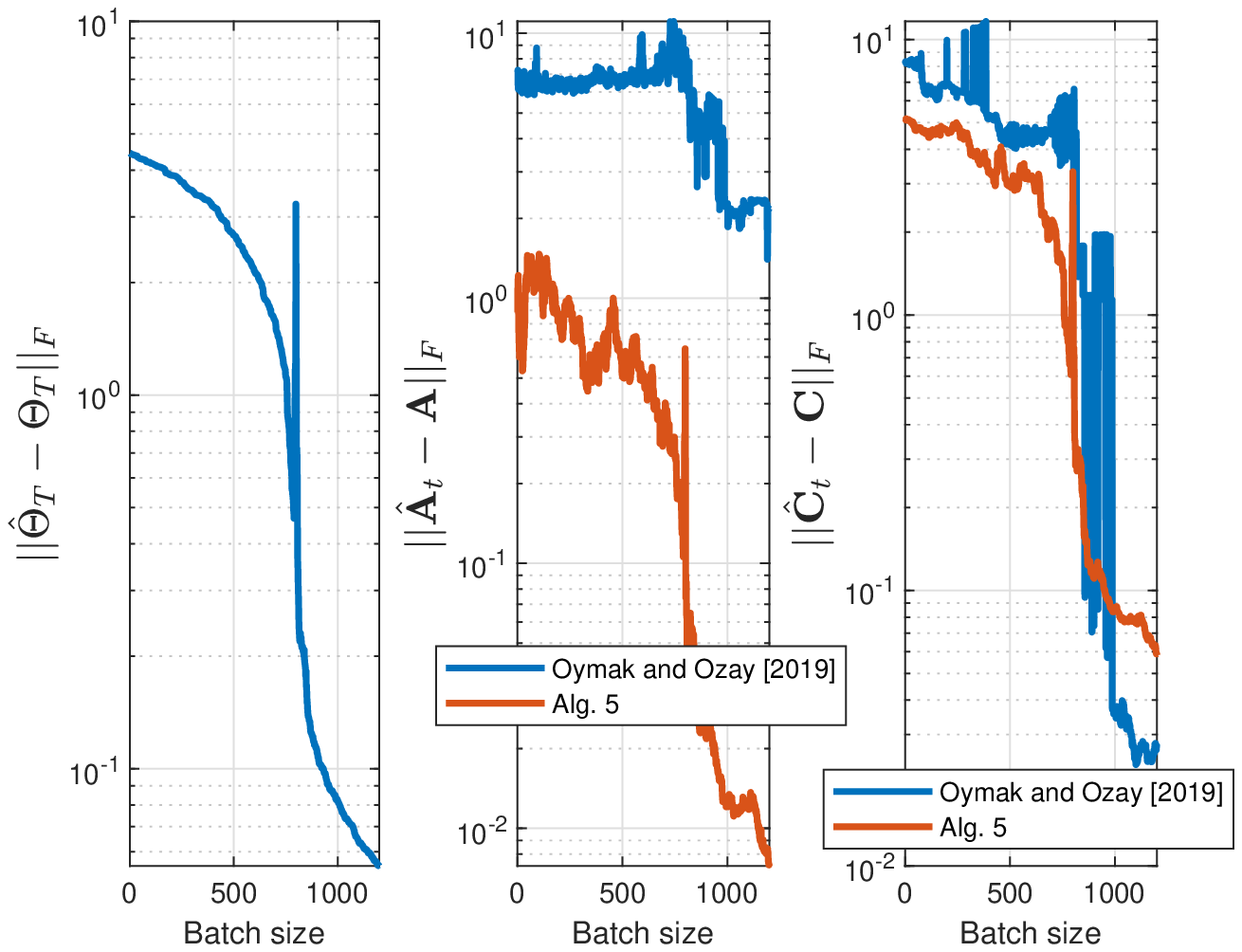}			\label{fig:oymak1}}
	\subfigure{\includegraphics[width=0.42\textwidth]{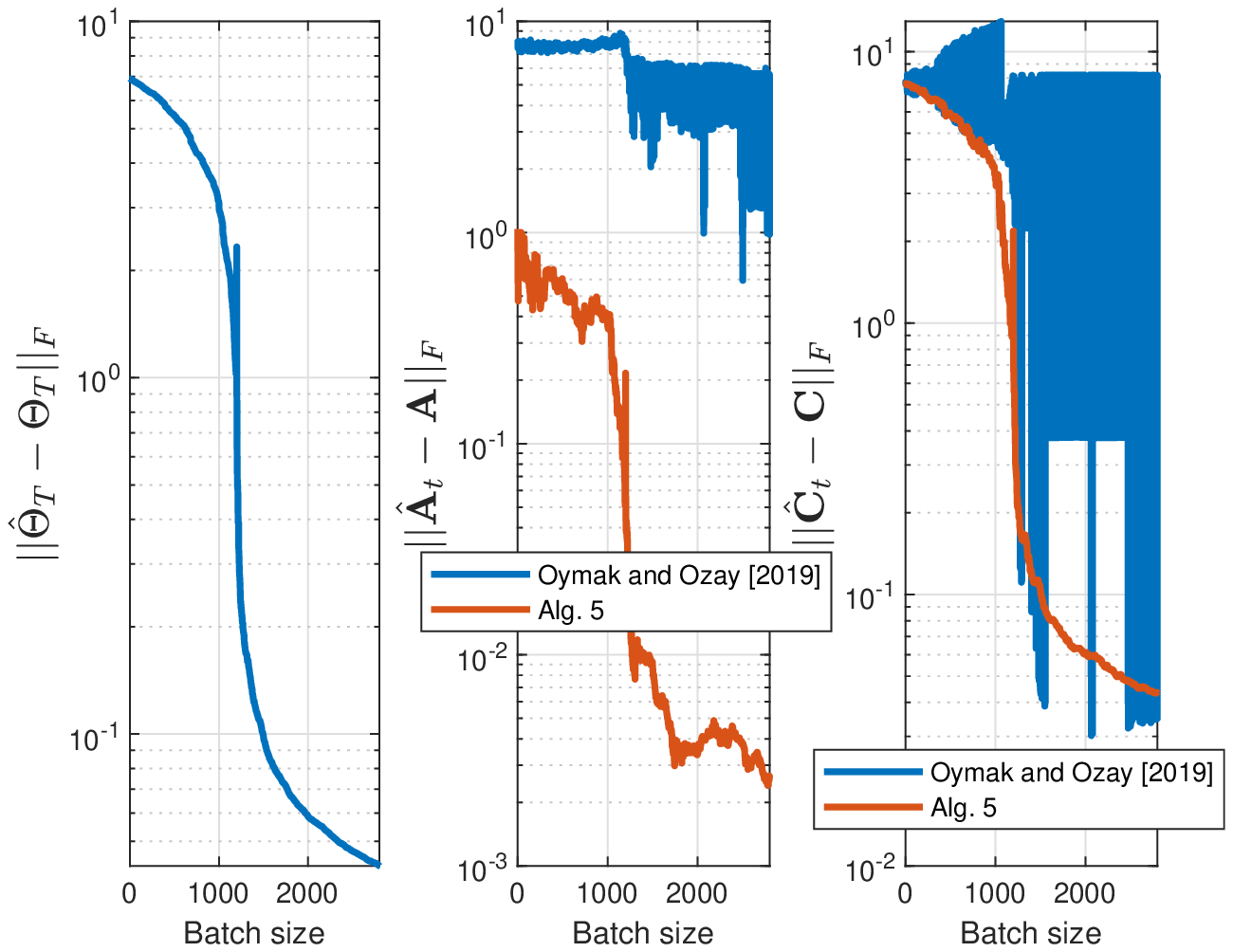}			\label{fig:oymak2}}\\
	\subfigure{\includegraphics[width=0.84\textwidth]{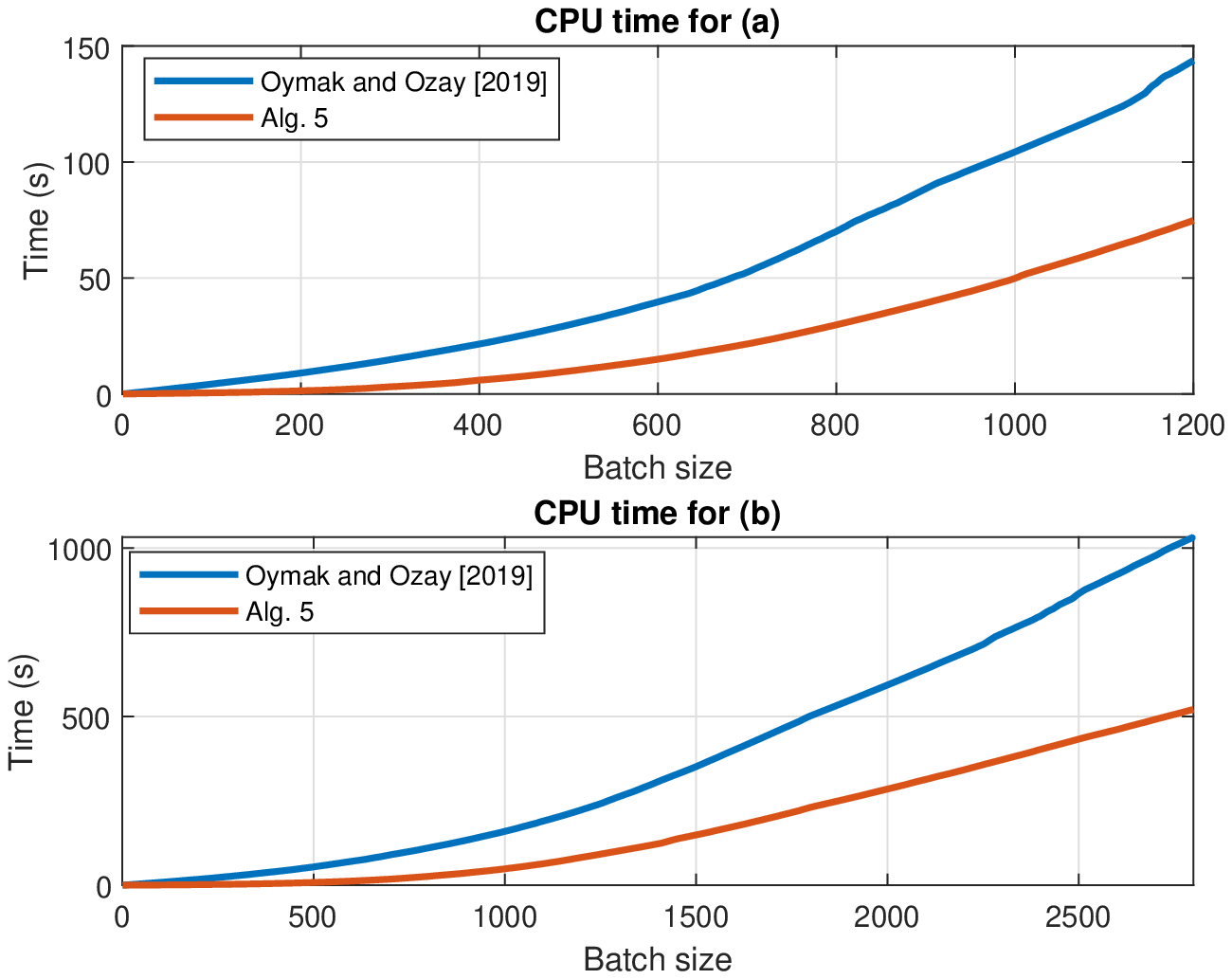}			\label{fig:oymak3}}
	\caption{Comparison of the performance of Algorithm \ref{al:inv} with \cite{oymak2019non}: (a)  $m=4$, $n=5$, and $p=4$; and (b) $m=6$, $n=5$, and $p=6$; and (c) the CPU time.}
\end{figure}

Since the Ho-Kalman Algorithm in \cite{oymak2019non} solves \eqref{opt:lim} by the pseudo-inverse method, we use Algorithm \ref{al:inv} for comparisons. We assume that input-output pairs arrive in an online streaming fashion and the batch size increases gradually. Both approaches share an identical estimation for the set of Markov parameters. From Figs. \ref{fig:oymak1} and \ref{fig:oymak2}, we observe that Algorithm \ref{al:inv} outperforms the Ho-Kalman approach in the estimation of $\mathbf{A}$. The reason is that Algorithm \ref{al:inv} directly extracts $\mathbf{A}$ from the Markov parameters. However, the Ho-Kalman approach estimates $\mathbf{H}$, $\mathbf{O}$ and $\mathbf{Q}$ first and based on these matrices, $\mathbf{A}$ is recovered. Therefore, the errors of estimations for $\mathbf{O}$, $\mathbf{Q}$ and $\mathbf{H}$ are added to each other in the estimation of $\mathbf{A}$. Furthermore, we observe from Fig. \ref{fig:oymak3} that the required CPU time by Algorithm \ref{al:inv} is significantly less than that by the Ho-Kalman algorithm as computing SVD is more costly compared to solving a linear system.

Next, we continue to evaluate the performance of our online and offline SGD algorithms on noisy and noisy-free linear dynamical systems when the system is SISO, single-input multi-output (SIMO), MISO, and MIMO. For each case, we consider three different hidden state dimensions and evaluate proposed algorithms in noisy and noise-free scenarios. The initial state of the system is zero. 
\subsection{SISO}
We consider three different hidden state dimensions, 20, 25 and 30 for the SISO system. It is observed from Figs. \ref{fig:siso20}, \ref{fig:siso25} and \ref{fig:siso30} that in all three cases  Algorithm \ref{al:comb} learns the unknown parameters $\boldsymbol{\Theta}_T$, $\mathbf{A}$ and $\mathbf{C}$ at a linear convergence rate. The convergence of Algorithm \ref{al:comb1} for identical systems is depicted in Figs. \ref{fig:siso20f}-\ref{fig:siso30f}. Each iteration of either approach is implemented based on the obtained gradient from one input-output sample while the measurement noise is zero. We observe that when the hidden state dimension increases, the required iterations by both algorithms to reach a certain residual error increase. In Fig. \ref{fig:siso20}, we have $\rho(\mathbf{A})=0.93$, and for Figs. \ref{fig:siso25} and \ref{fig:siso30} we have $\rho(\mathbf{A})=0.975$. To tackle the heavy-tail issue of the transfer function, we increase $T$ and also decrease the learning rate when the size of the hidden state increases. For Algorithms \ref{al:comb1} and \ref{al:comb}, we have $(T,\eta)=\{(800,3\times 10^{-4}),(1300,3\times 10^{-4}),(1600,2\times 10^{-4})\}$, when the $n=\{20,25,30\}$, respectively. For Algorithm \ref{al:comb1}, the batch size is $10,000$. The performance of Algorithms \ref{al:comb} and \ref{al:comb1} for previously described systems is depicted in Figs. \ref{fig:siso20n}-\ref{fig:siso30n} and \ref{fig:siso20offline}-\ref{fig:siso30offline}, respectively, when the measurement noise follows a normal distribution with zero mean and standard deviation $0.1$. The input to the system is Gaussian noise with zero mean and standard deviation $1$. For both approaches, we set $(T,\eta)=\{(170,5\times 10^{-8}),(400,4\times 10^{-8}),(600,3\times 10^{-8})\}$ for the three considered systems. The batch size is $10^7$ for Algorithm \ref{al:comb1}. 
In both noisy and noise-free systems, we observe that Algorithm \ref{al:comb1} requires a greater number of iterations compared to Algorithm \ref{al:comb} to reach a certain residual error. 
\begin{figure}
	\centering
	\subfigure{\includegraphics[width=0.32\textwidth]{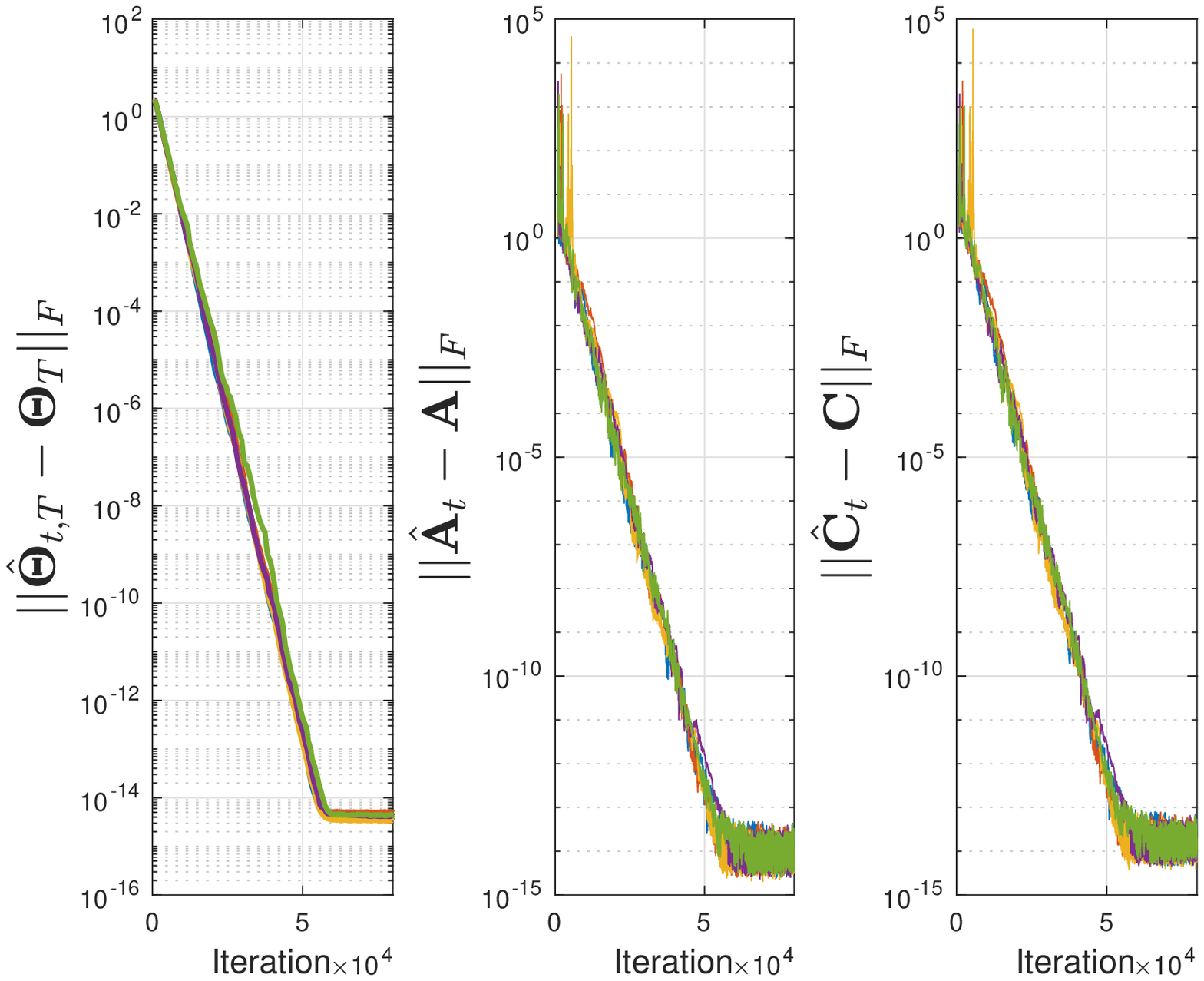}			\label{fig:siso20}}
	\subfigure{\includegraphics[width=0.32\textwidth]{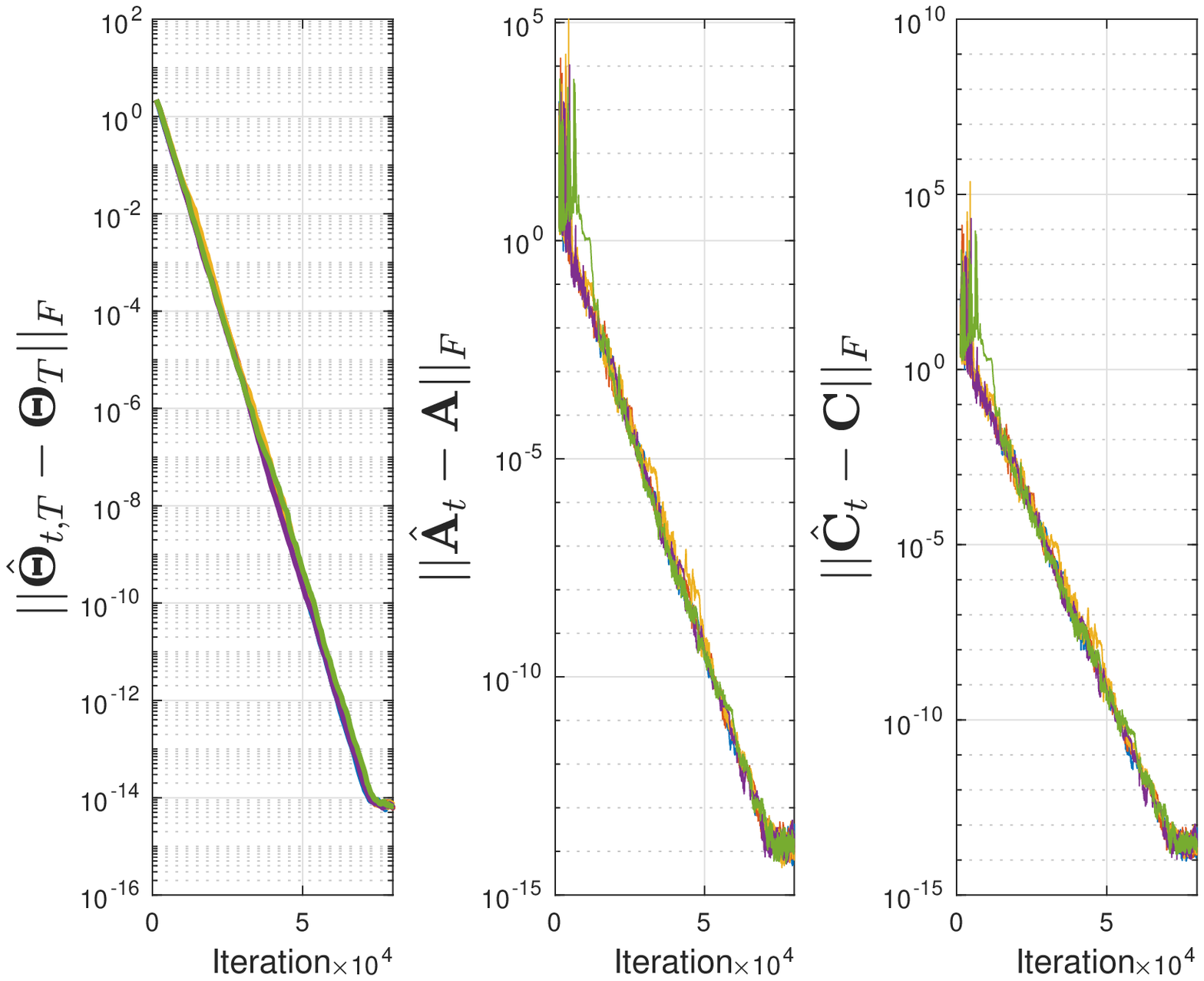}			\label{fig:siso25}}
	\subfigure{\includegraphics[width=0.32\textwidth]{figs/siso30.eps}			\label{fig:siso30}}
	\subfigure{\includegraphics[width=0.32\textwidth]{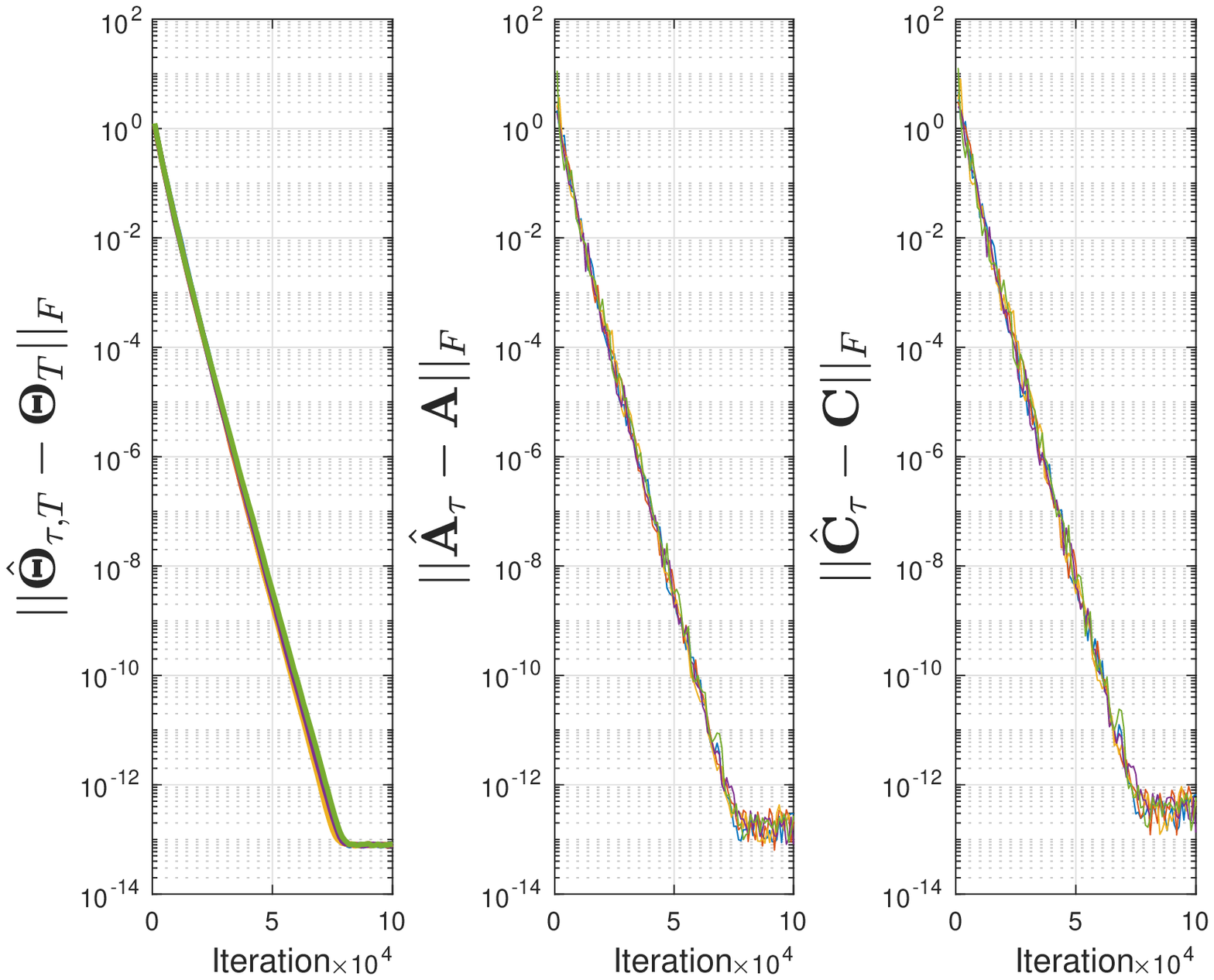}			\label{fig:siso20f}}
	\subfigure{\includegraphics[width=0.32\textwidth]{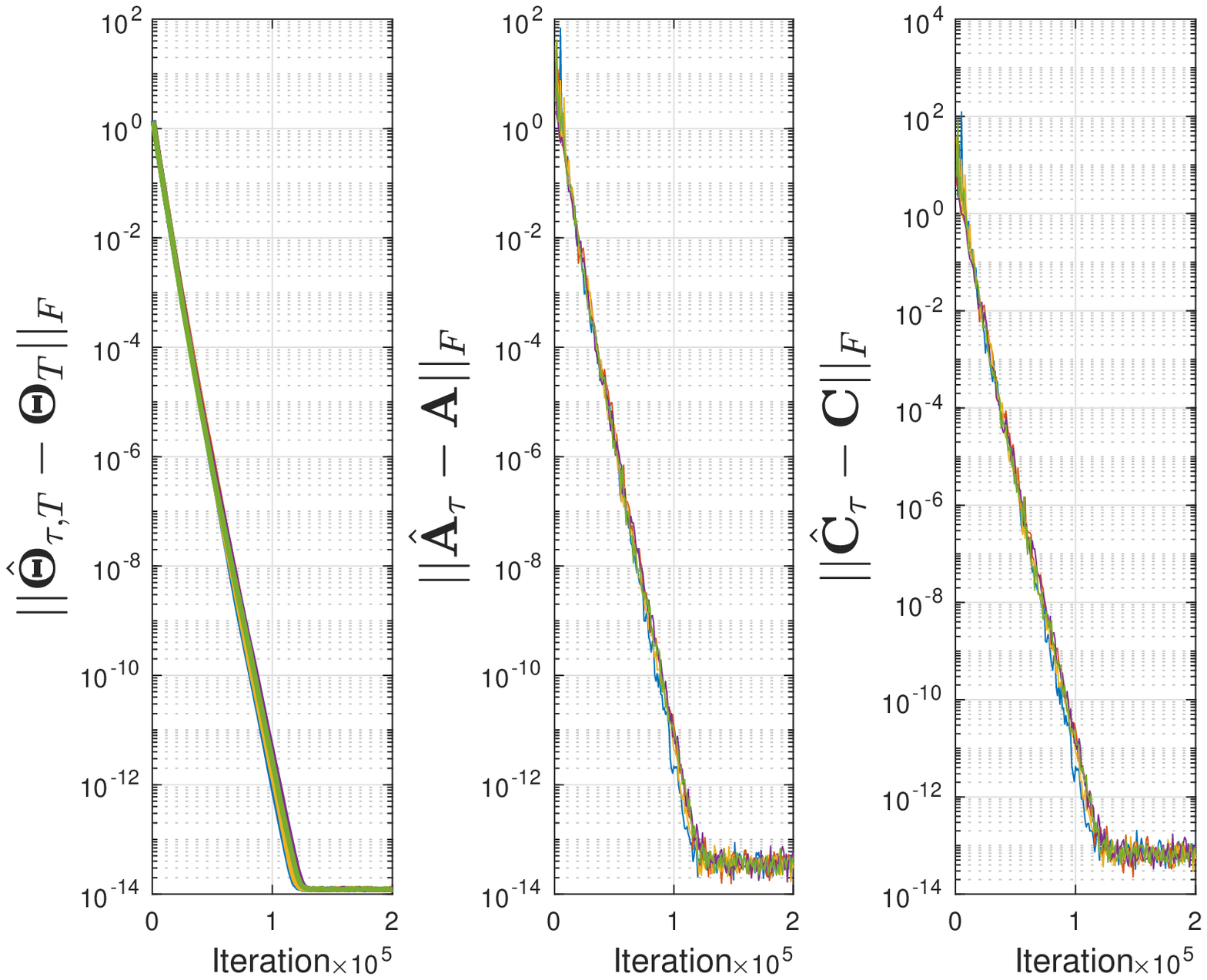}			\label{fig:siso25f}}
	\subfigure{\includegraphics[width=0.32\textwidth]{figs/siso30f.eps}			\label{fig:siso30f}}
	\subfigure{\includegraphics[width=0.32\textwidth]{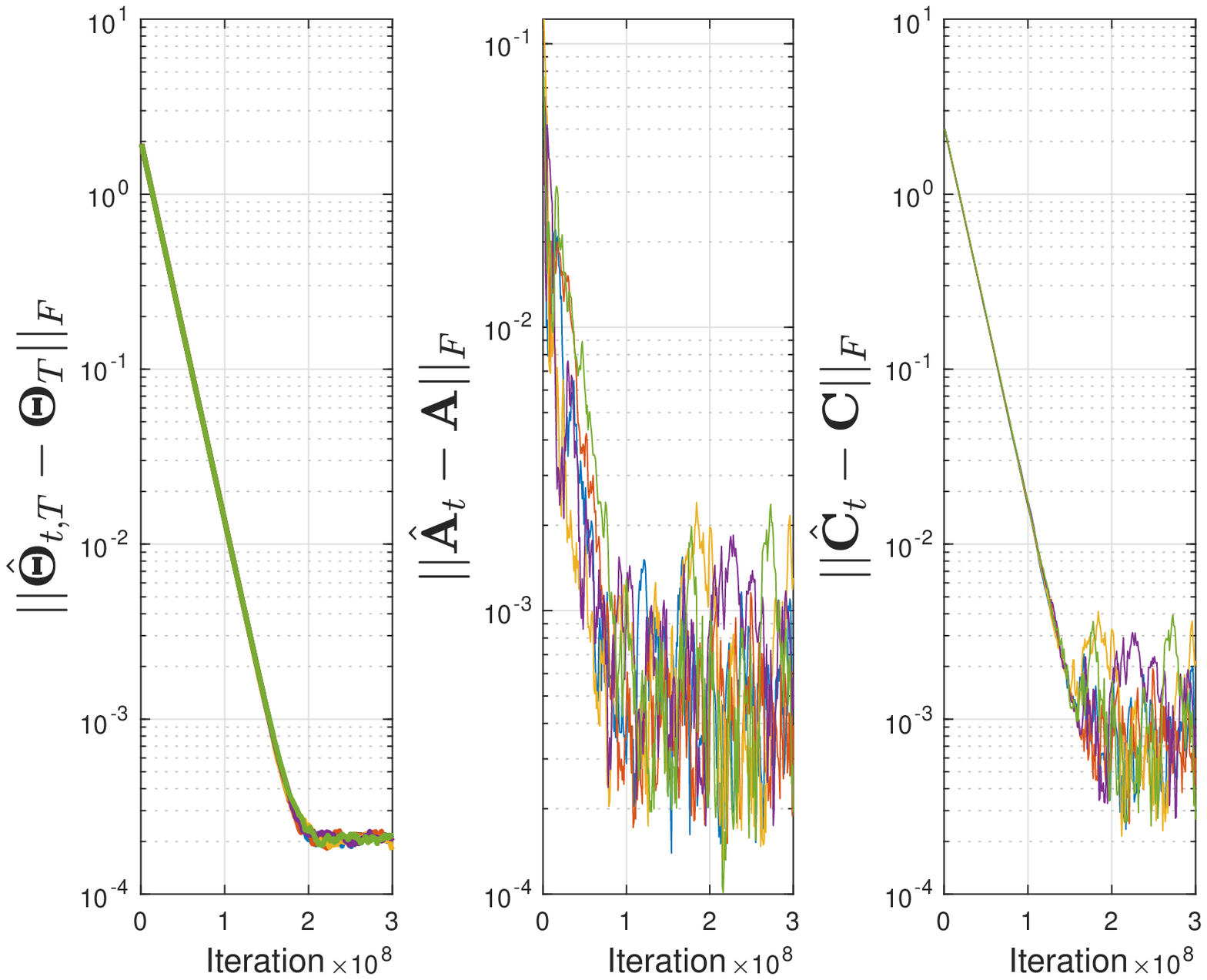}			\label{fig:siso20n}}
	\subfigure{\includegraphics[width=0.32\textwidth]{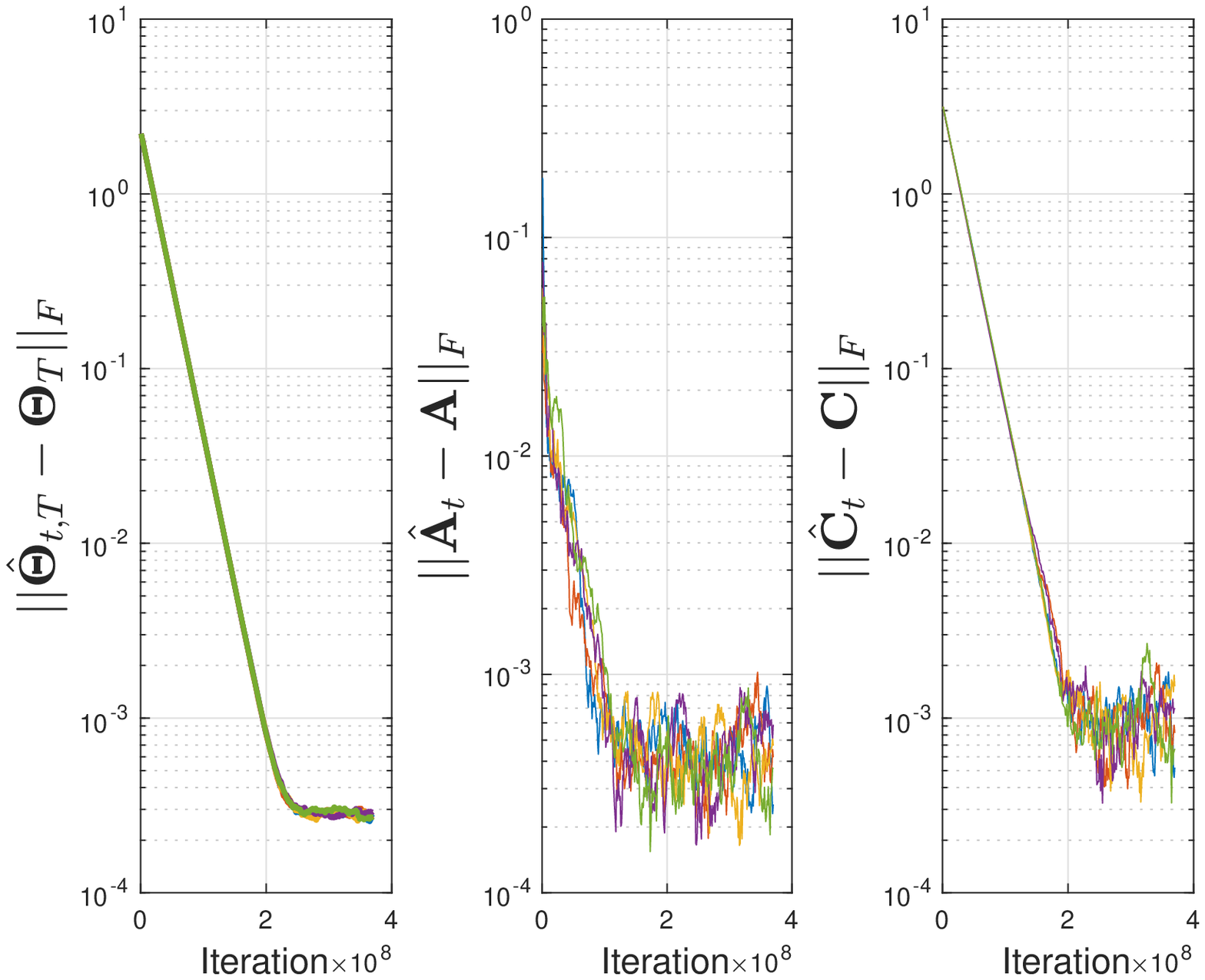}			\label{fig:siso25n}}
	\subfigure{\includegraphics[width=0.32\textwidth]{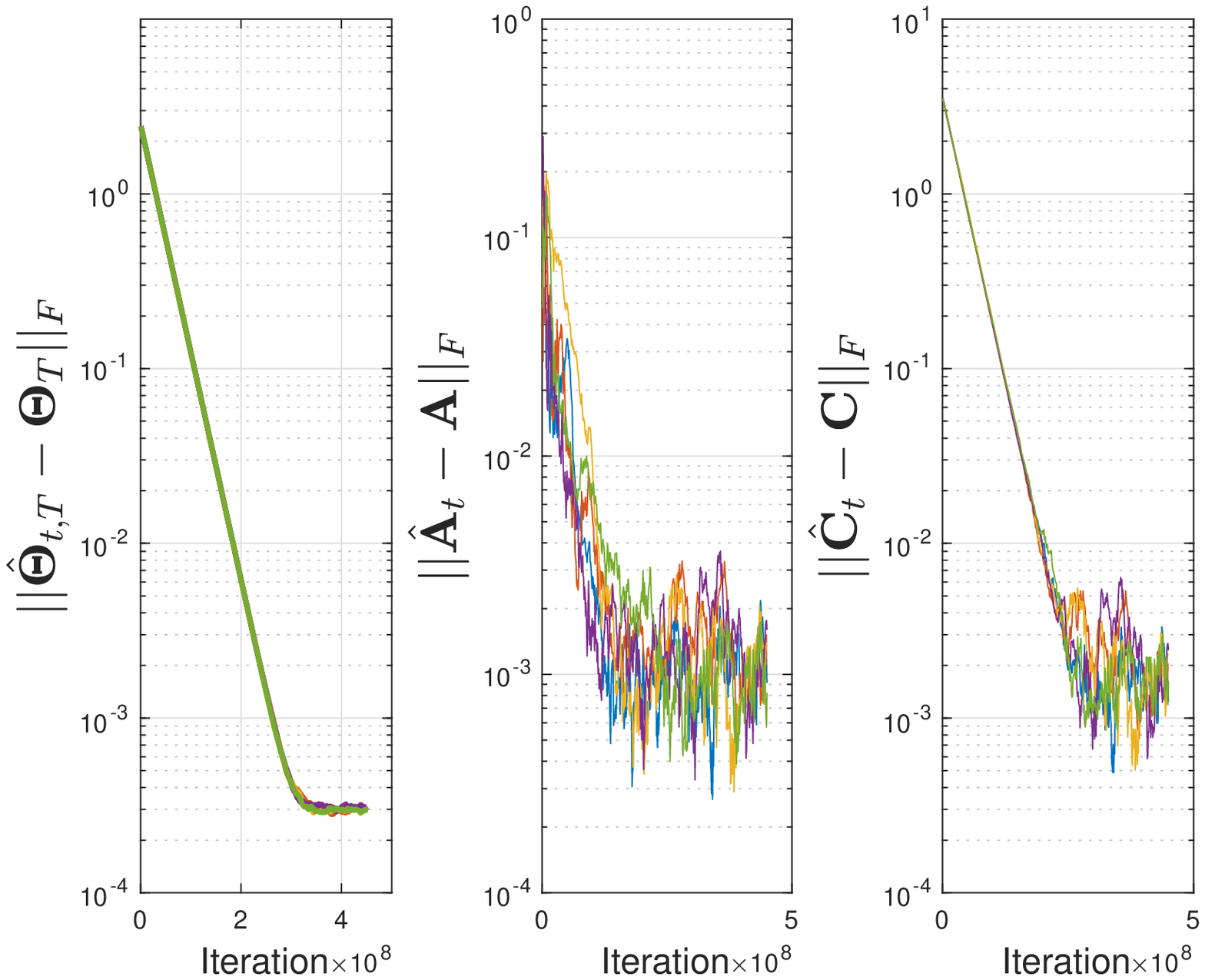}			\label{fig:siso30n}}
	\subfigure{\includegraphics[width=0.32\textwidth]{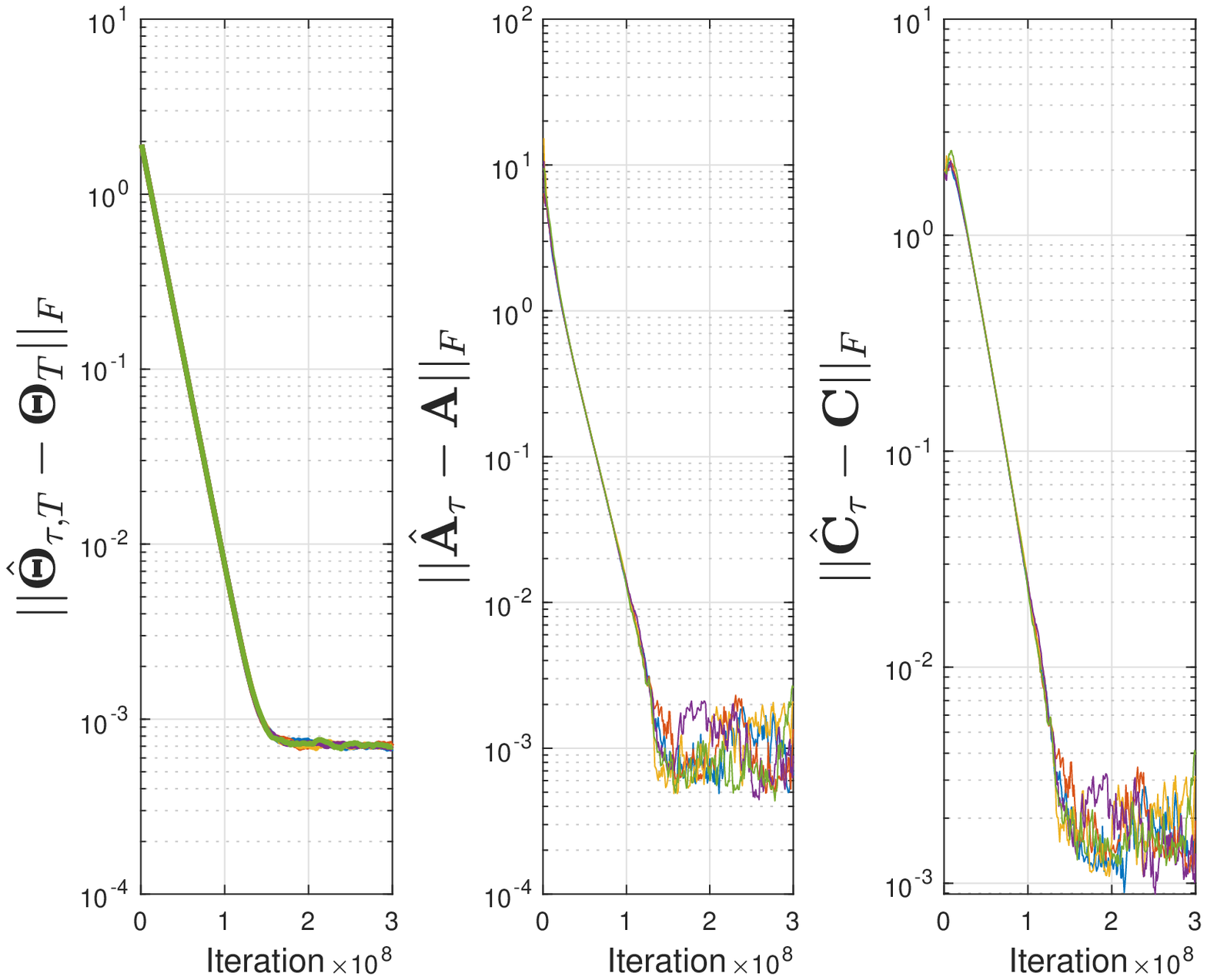}			\label{fig:siso20offline}}
	\subfigure{\includegraphics[width=0.32\textwidth]{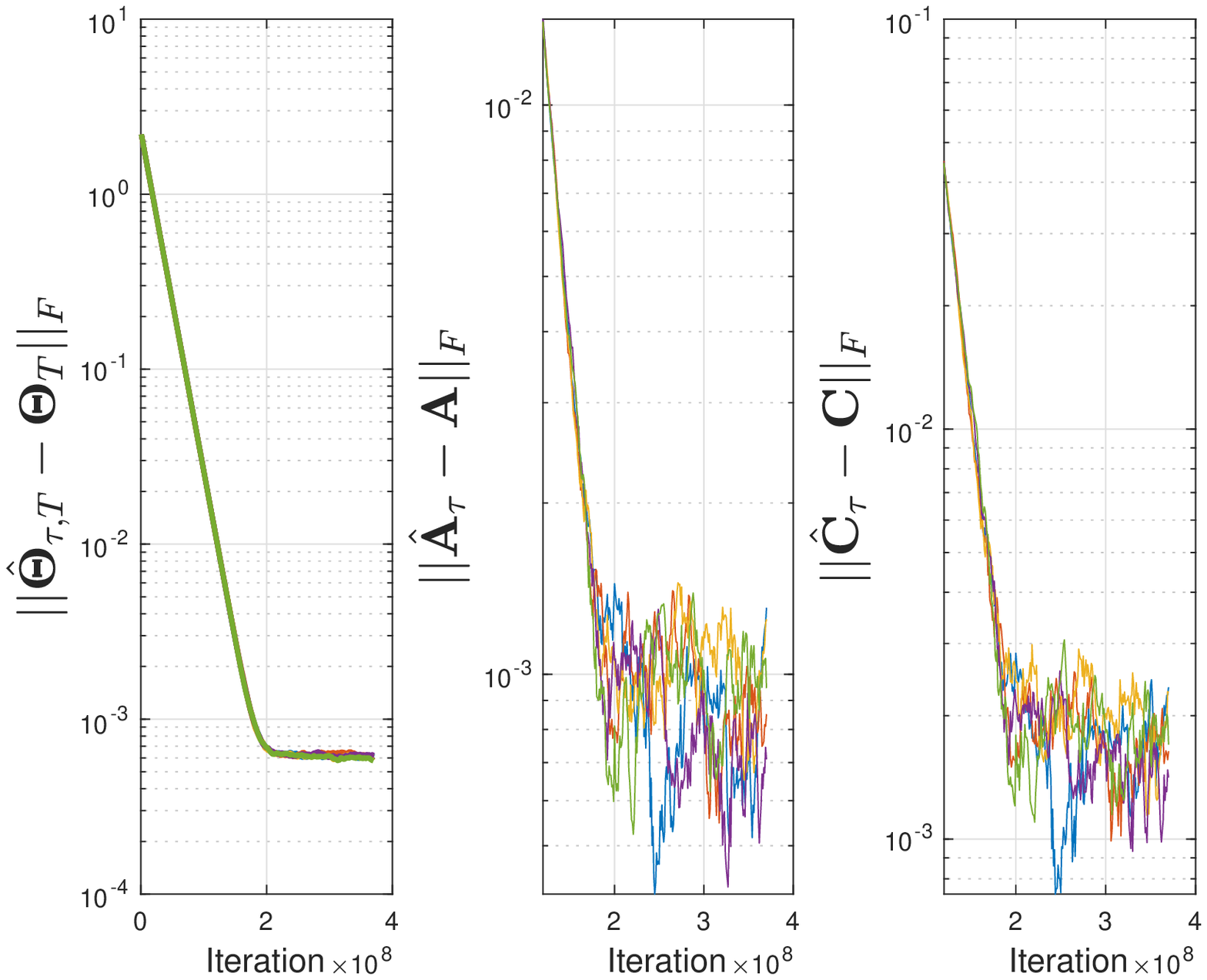}			\label{fig:siso25offline}}
	\subfigure{\includegraphics[width=0.32\textwidth]{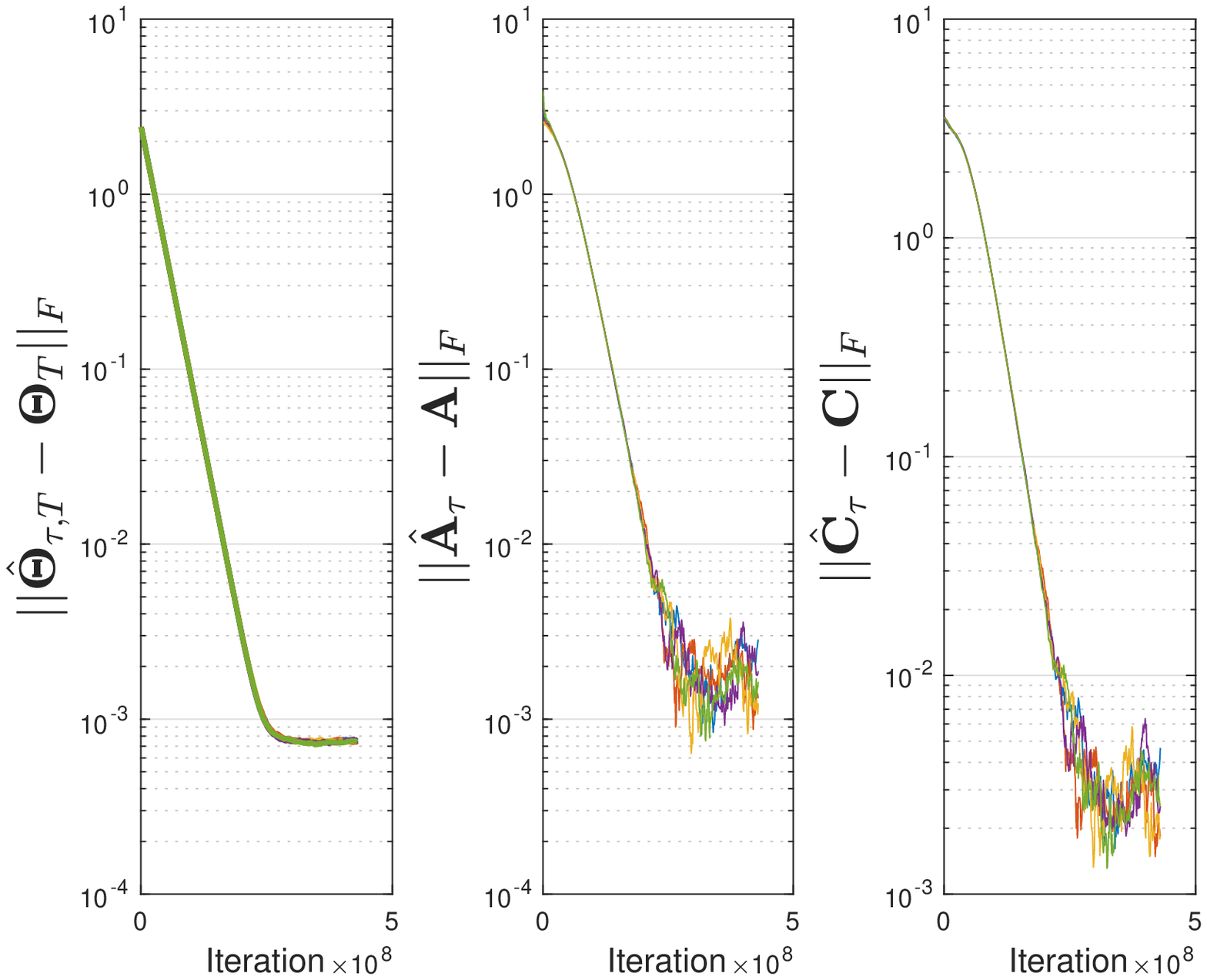}			\label{fig:siso30offline}}
	\caption{Results for SISO systems.  In (a)-(f), the systems are noise-free. In (g)-(l), the systems are noisy. In (a), (d), (g) and (j), $n=20$, $m=1$, $p=1$. In (b), (e), (h) and (k), $n=25$, $m=1$, $p=1$. In (c), (f), (i) and (l), $n=30$, $m=1$, $p=1$.}
\end{figure}

\subsection{SIMO}
We consider three different SIMO systems where $(n,p)=(20,4)$, $(n,p)=(25,5)$ and $(n,p)=(30,6)$. We observe from Figs. \ref{fig:simo20}, \ref{fig:simo25} and \ref{fig:simo30} that when the measurement noise is zero, in all three cases Algorithm \ref{al:comb}  reaches very close to the machine epsilon. For the above three dimensions, the necessary truncation length and the learning rate do not change significantly when the hidden state dimension and output size increase. The spectral radius of $\mathbf{A}$ in Figs. \ref{fig:simo20}, \ref{fig:simo25} and \ref{fig:simo30} is $0.93$, $0.95$ and $0.96$, respectively. The convergence of Algorithm \ref{al:comb1} for the three considered systems is depicted in Figs. \ref{fig:simo20f}-\ref{fig:simo30f}, when the batch size is $10,000$. For both algorithms, we have $(T,\eta)=\{(800,10^{-5}),(800, 10^{-5}),(800, 10^{-5})\}$. For the above systems, we consider measurement noise with zero mean and standard deviation $0.1$. For both approaches, we set $(T,\eta)=\{(300,4\times 10^{-8}),(500,3\times 10^{-8}),(700,3\times 10^{-8})\}$. The batch size is $10^7$ for Algorithm \ref{al:comb1}. The convergence of Algorithm \ref{al:comb} for noisy systems is depicted in Figs. \ref{fig:simo20n}-\ref{fig:simo30n}, and the convergence of Algorithm \ref{al:comb1}  is depicted in Figs. \ref{fig:simo20offline}-\ref{fig:simo30offline}. The input to the system is Gaussian noise with zero mean and standard deviation $1$, and system noise has the standard deviation $0.1$.
\begin{figure}
	\centering
	\subfigure{\includegraphics[width=0.32\textwidth]{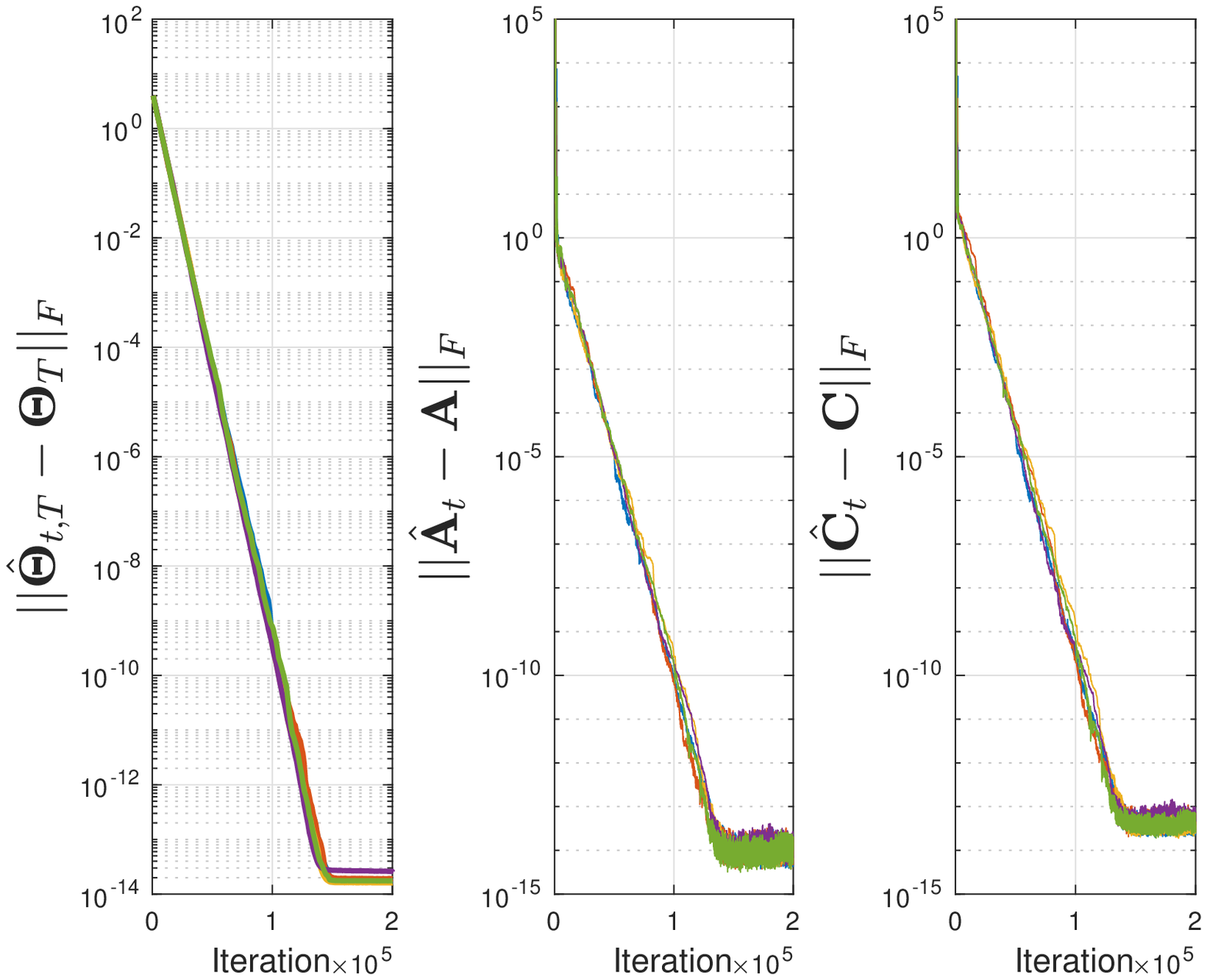}			\label{fig:simo20}}
	\subfigure{\includegraphics[width=0.32\textwidth]{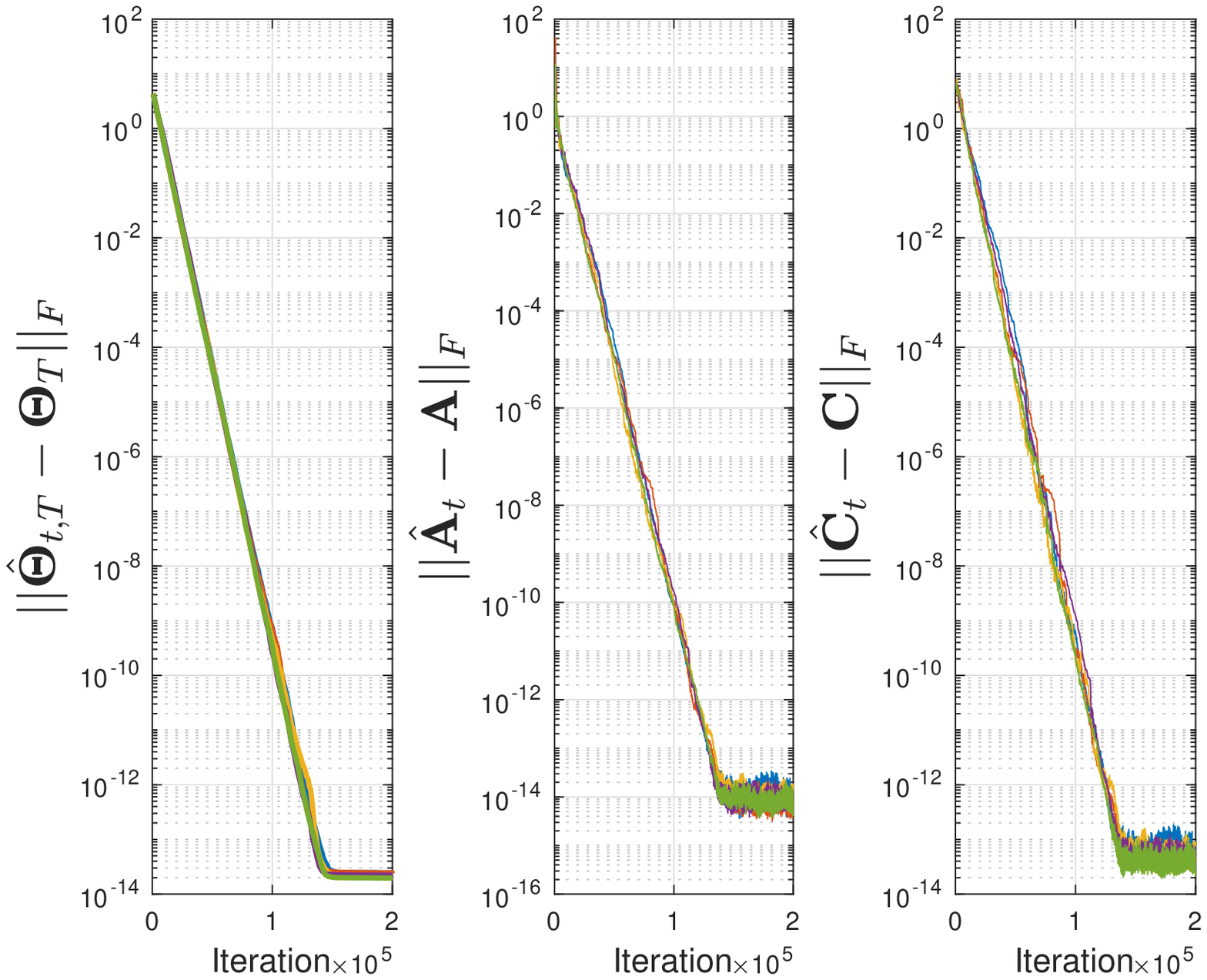}			\label{fig:simo25}}
	\subfigure{\includegraphics[width=0.32\textwidth]{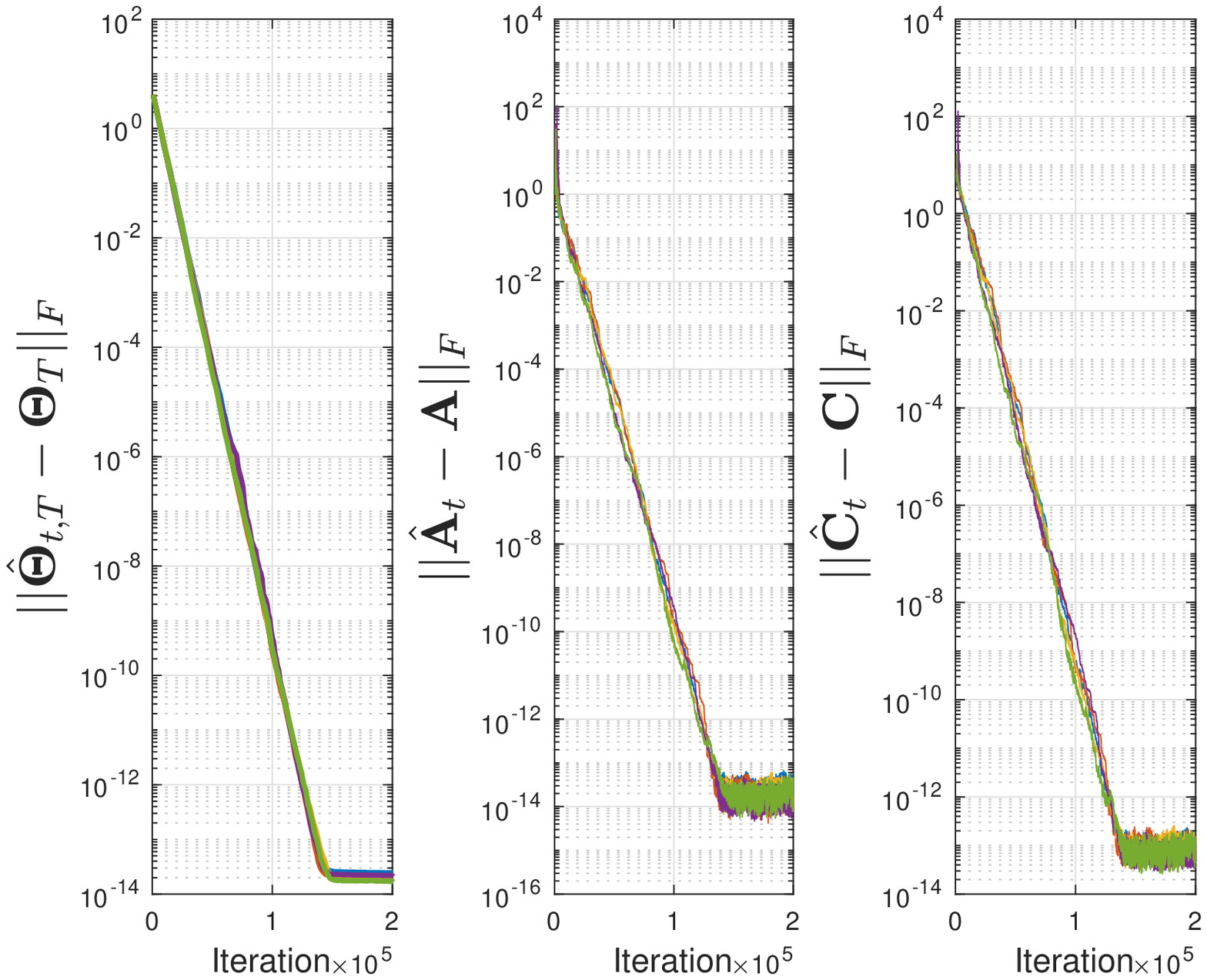}			\label{fig:simo30}}
	\subfigure{\includegraphics[width=0.32\textwidth]{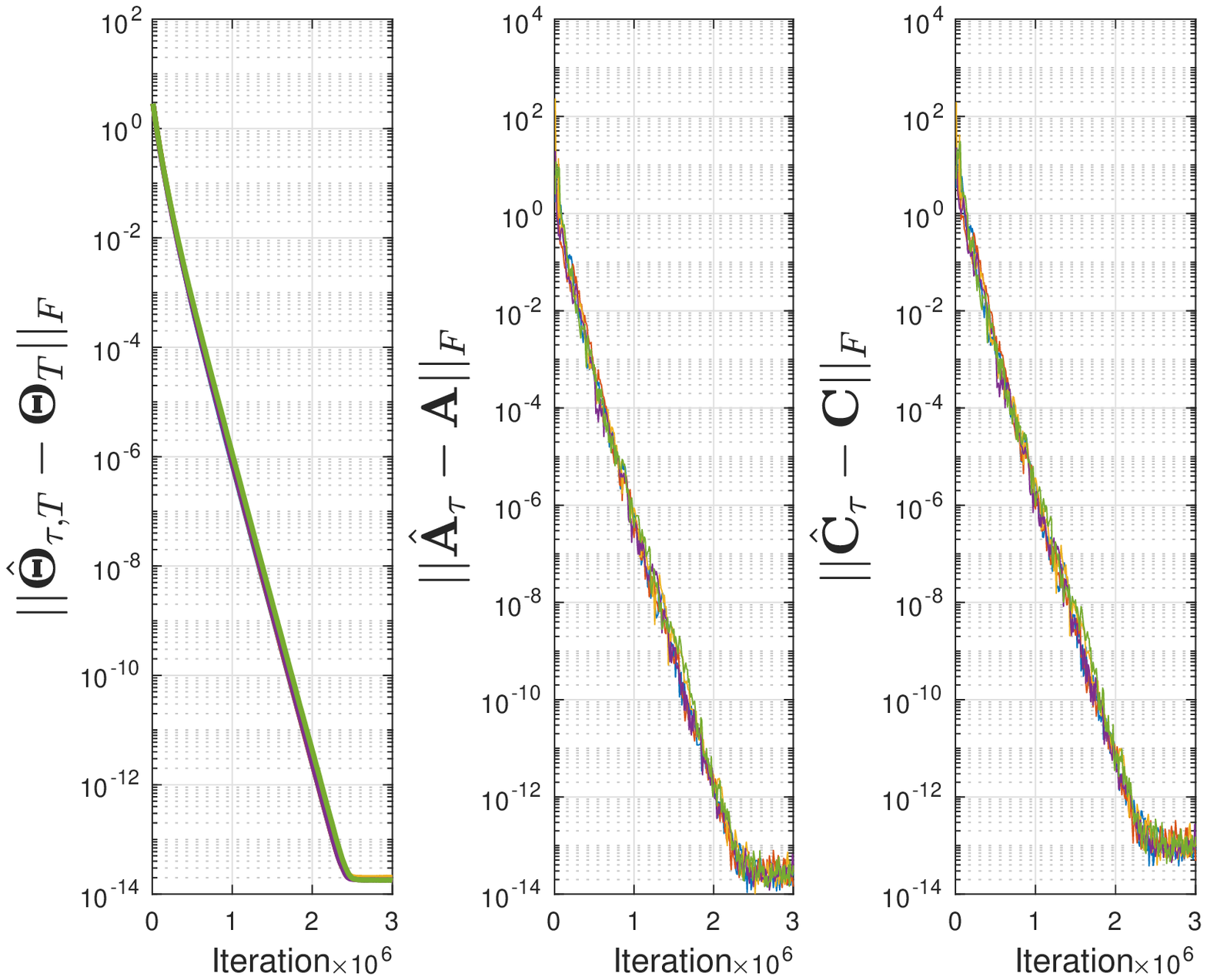}			\label{fig:simo20f}}
	\subfigure{\includegraphics[width=0.32\textwidth]{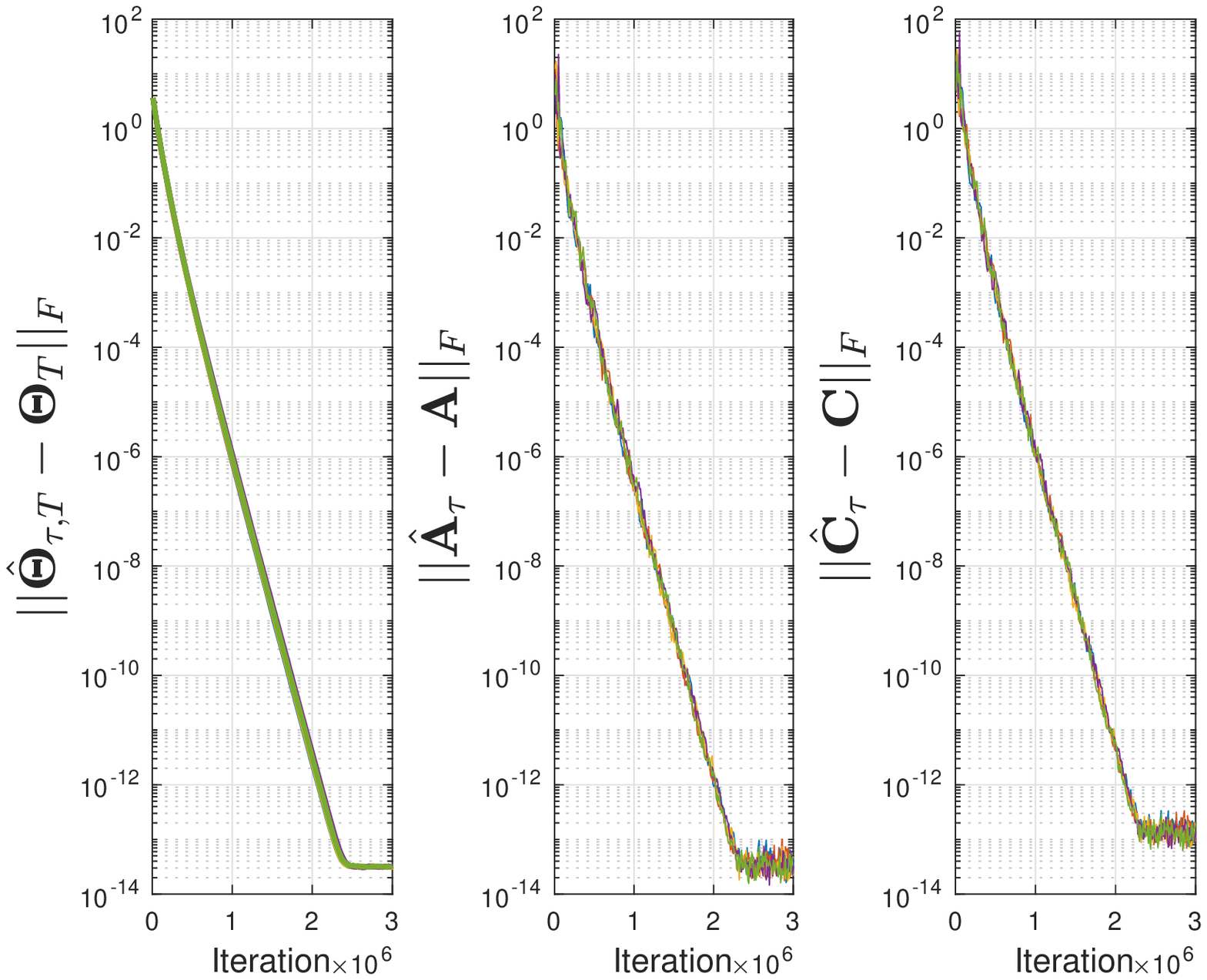}			\label{fig:simo25f}}
	\subfigure{\includegraphics[width=0.32\textwidth]{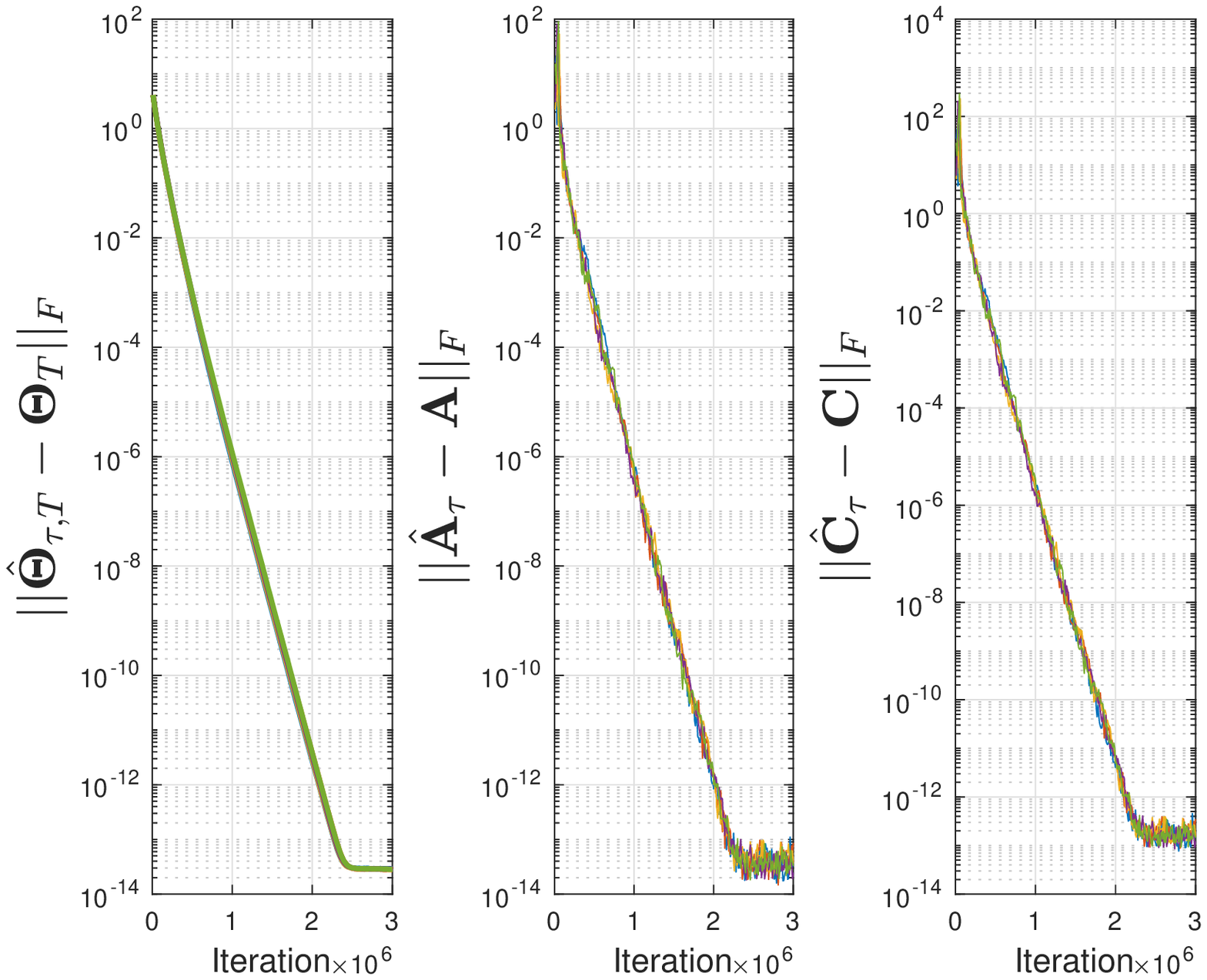}			\label{fig:simo30f}}
	\subfigure{\includegraphics[width=0.32\textwidth]{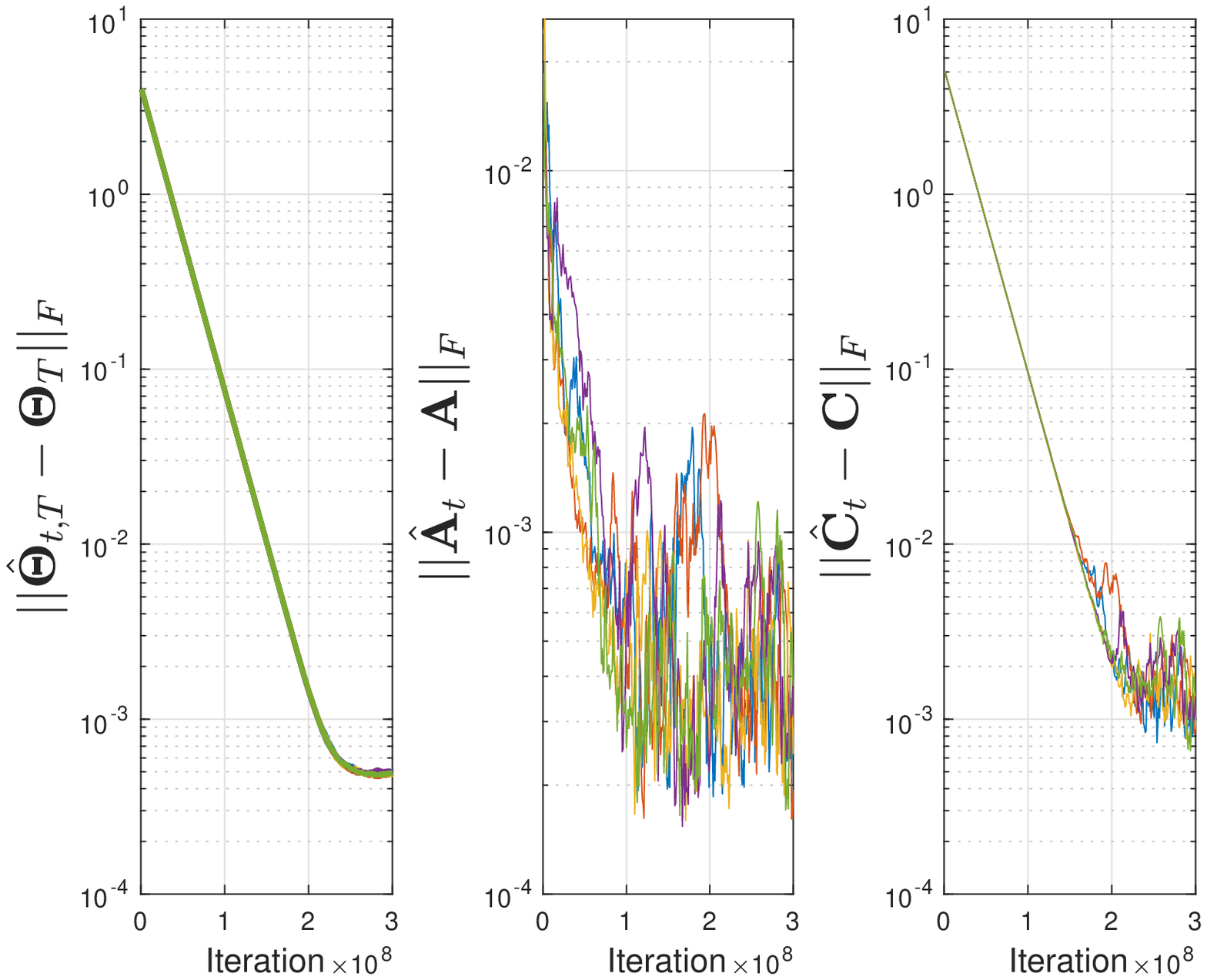}			\label{fig:simo20n}}
	\subfigure{\includegraphics[width=0.32\textwidth]{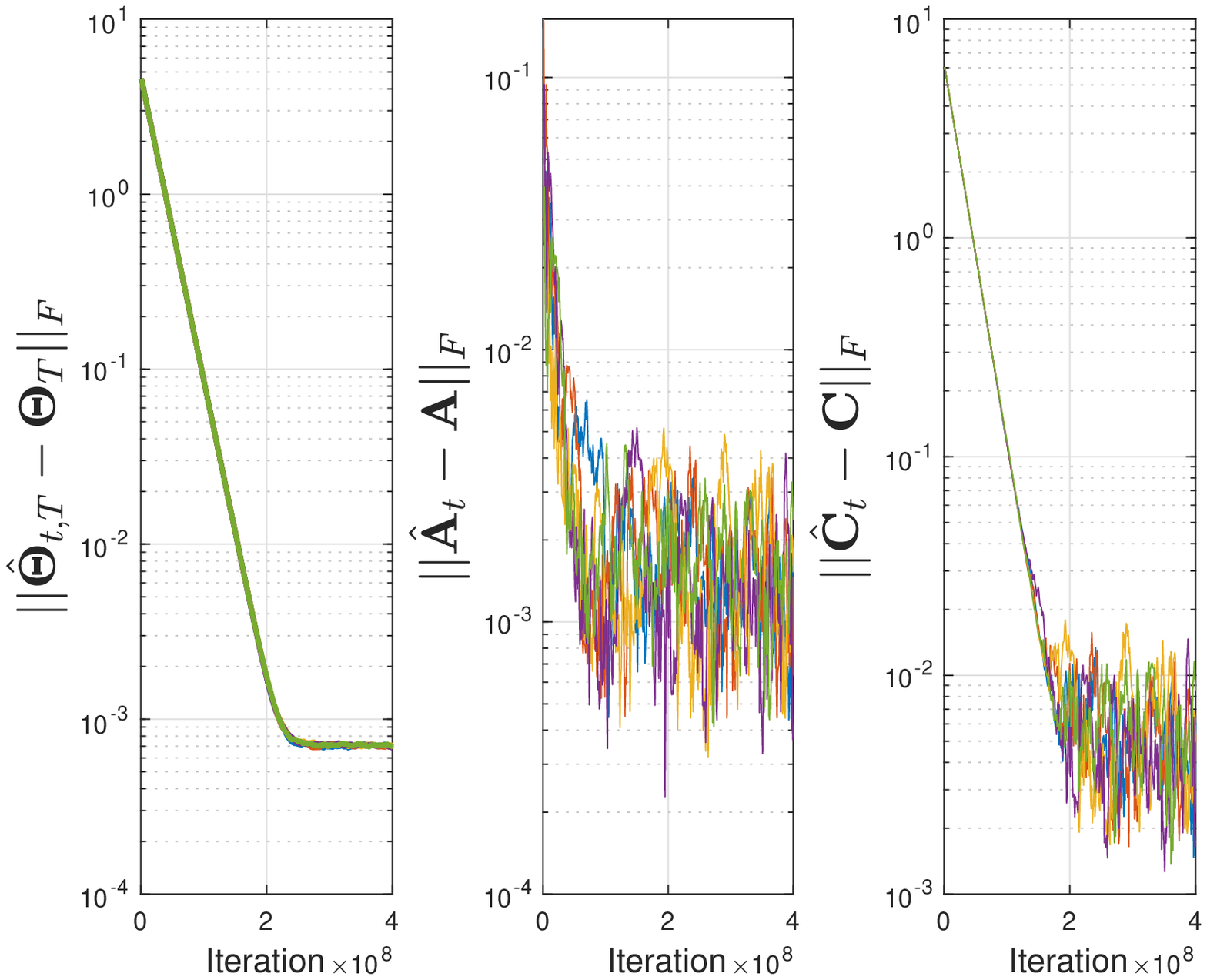}			\label{fig:simo25n}}
	\subfigure{\includegraphics[width=0.32\textwidth]{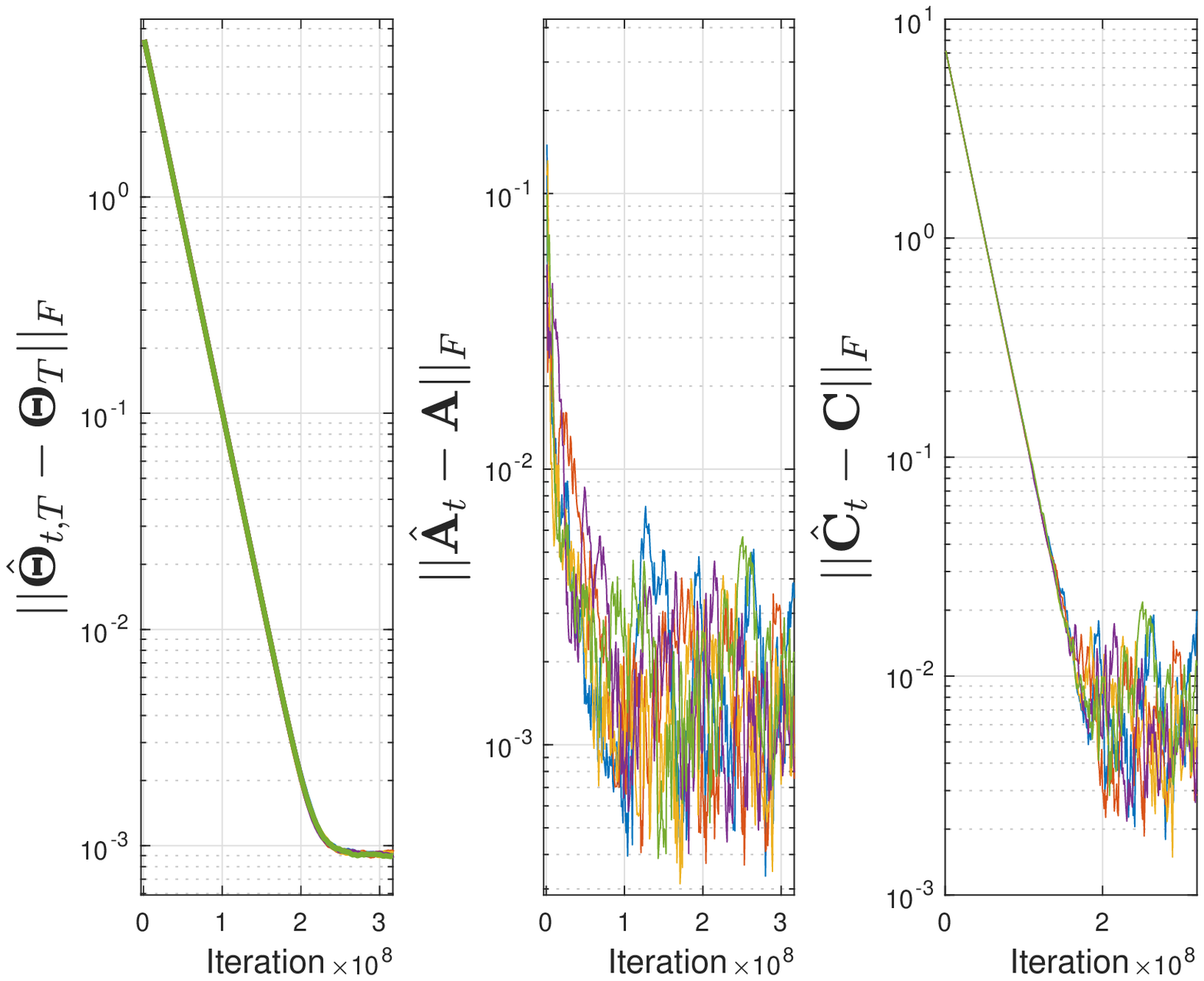}			\label{fig:simo30n}}
	\subfigure{\includegraphics[width=0.32\textwidth]{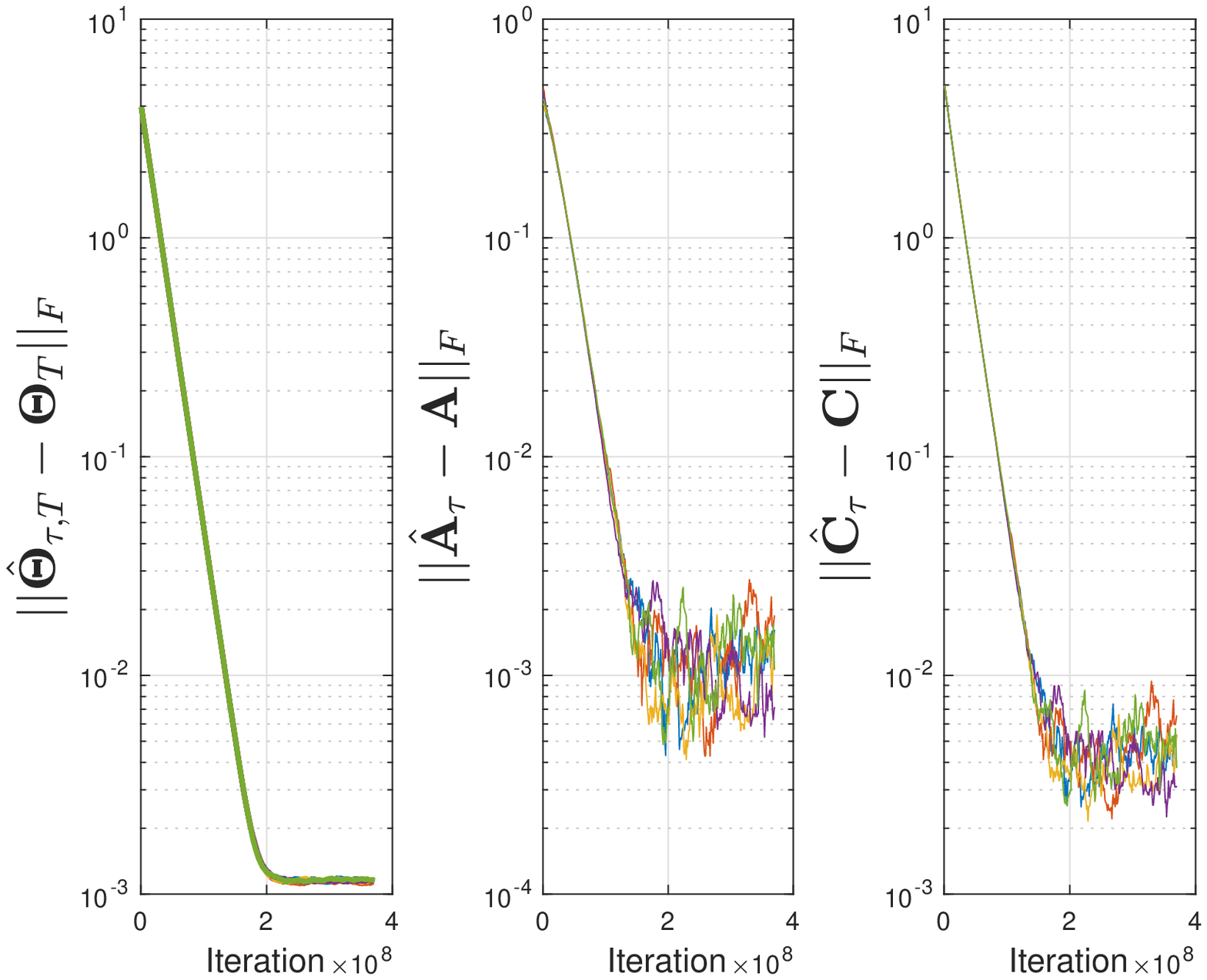}			\label{fig:simo20offline}}
	\subfigure{\includegraphics[width=0.32\textwidth]{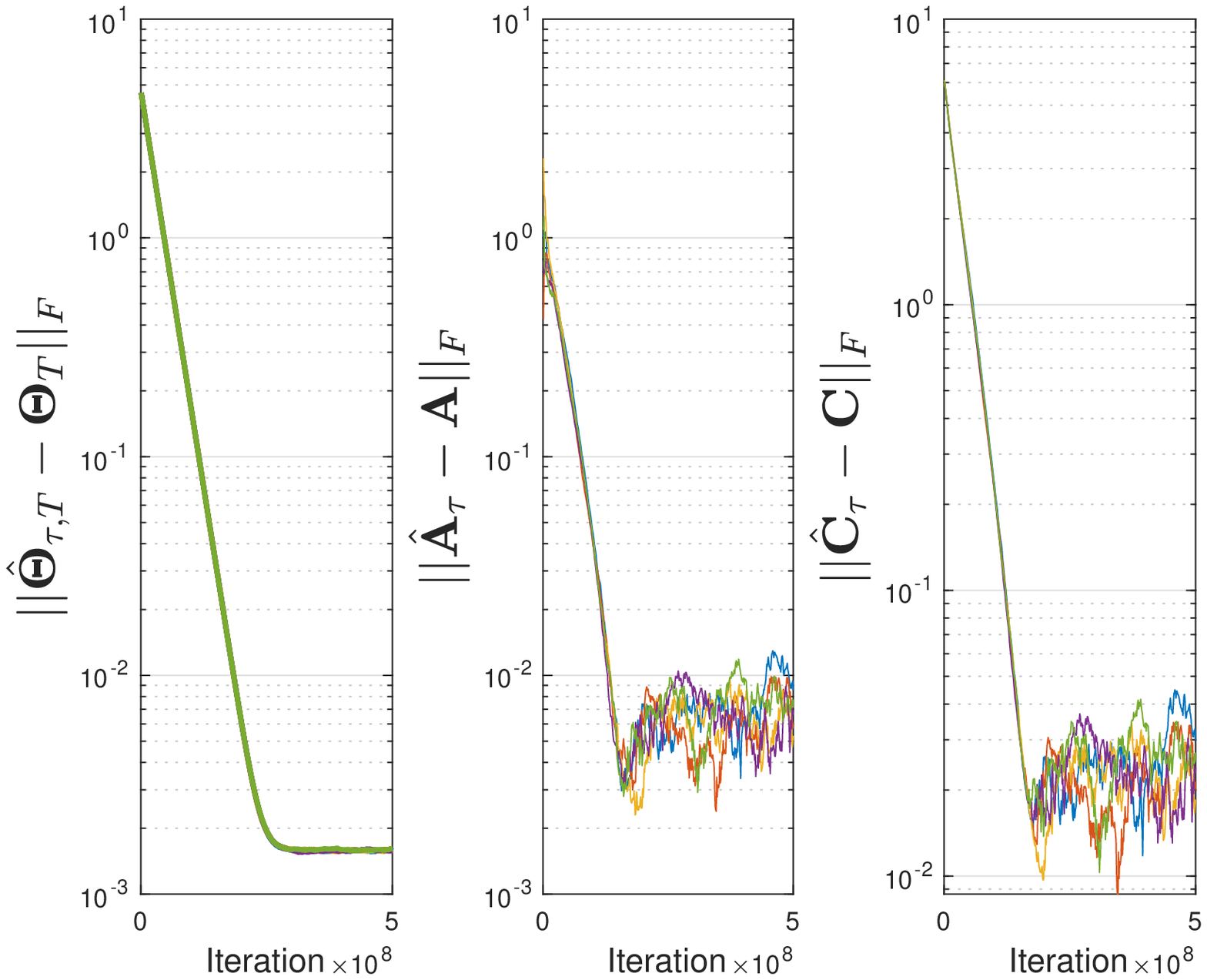}			\label{fig:simo25offline}}
	\subfigure{\includegraphics[width=0.32\textwidth]{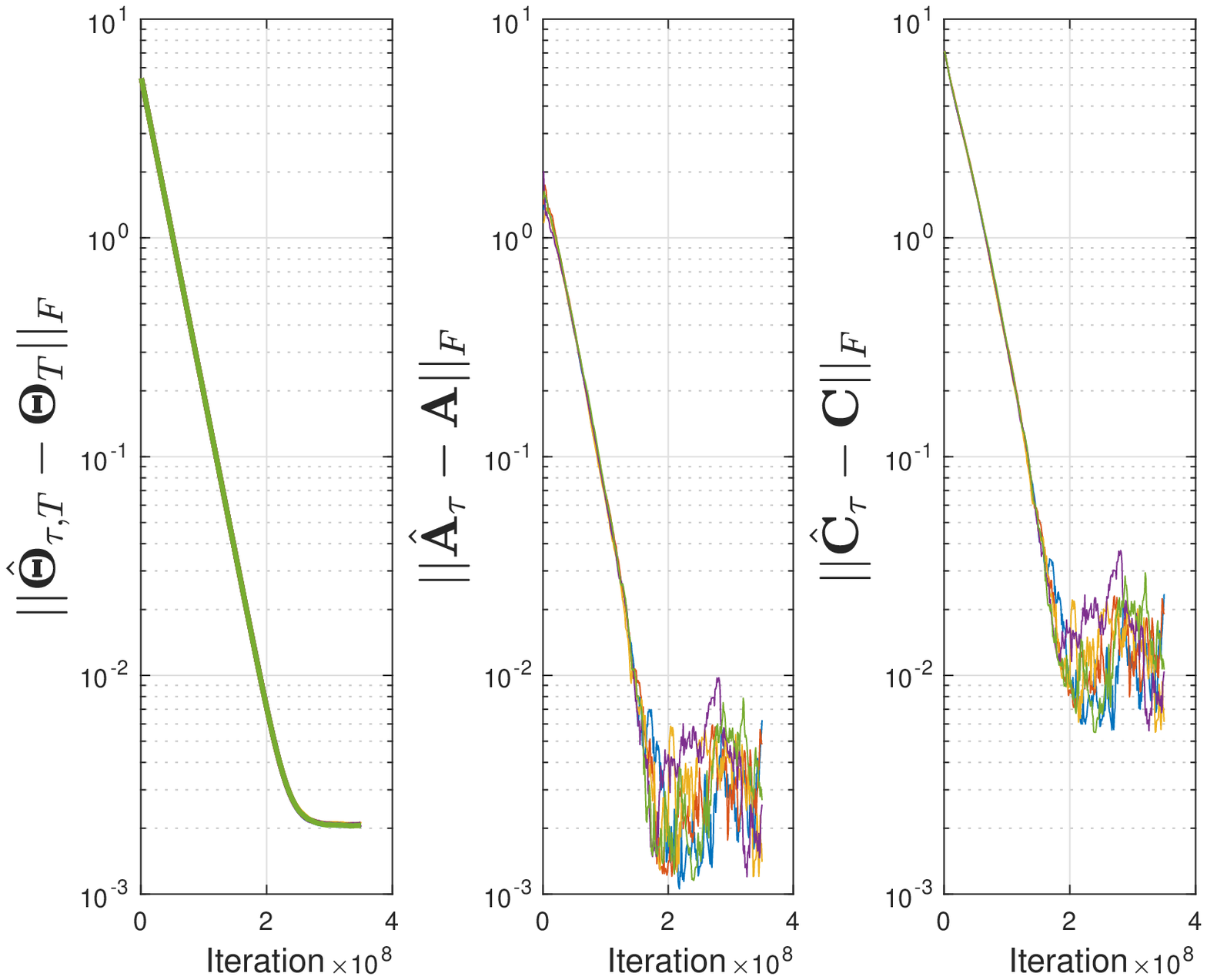}			\label{fig:simo30offline}}
	\caption{Results for SISO systems.  In (a)-(f), the systems are noise-free. In (g)-(l), the systems are noisy. In (a), (d), (g) and (j), $n=20$, $m=1$, $p=1$. In (b), (e), (h) and (k), $n=25$, $m=1$, $p=1$. In (c), (f), (i) and (l), $n=30$, $m=1$, $p=1$.}
\end{figure}

\subsection{MISO}
We consider three different MISO systems for which the hidden state dimensions are 20, 25 and 30. For these three systems, the input sizes are 4, 5 and 6, respectively.  As depicted in Figs. \ref{fig:miso20}, \ref{fig:miso25} and \ref{fig:miso30}, Algorithm \ref{al:comb} learns the unknown parameters at a linear convergence rate. In Fig. \ref{fig:miso20}, we have $\rho(\mathbf{A})=0.75$ and for Figs. \ref{fig:miso25} and \ref{fig:miso30}, we have $\rho(\mathbf{A})=0.70$. In Figs. \ref{fig:miso20}-\ref{fig:miso30}, we have $(T,\eta)=\{(800,10^{-5}),(800, 10^{-5}),(800, 10^{-5})\}$. Identical truncation length and the learning rate are considered for Algorithm \ref{al:comb1} in \ref{fig:miso20f}-\ref{fig:miso30f}. The batch size for Algorithm \ref{al:comb1}  is $10,000$. For noisy systems, the measurement noise is white, and its mean is zero and its variance is $0.01$. The standard deviation of the input signal is $0.1$. When the measurement noise is considered for the above systems, the convergence of Algorithm \ref{al:comb} is depicted in Figs. \ref{fig:miso20n}-\ref{fig:miso30n}. Moreover, Figs. \ref{fig:miso20offline}-\ref{fig:miso30offline} show the convergence of Algorithm \ref{al:comb1} for the three noisy systems. The truncation length and step-size for both algorithms are $(T,\eta)=\{(60,5\times 10^{-8}),(60,5\times 10^{-8}),(60,5\times 10^{-8})\}$ for different systems. The batch size for Algorithm \ref{al:comb1} is $10^7$.
\begin{figure}
	\centering
	\subfigure{\includegraphics[width=0.32\textwidth]{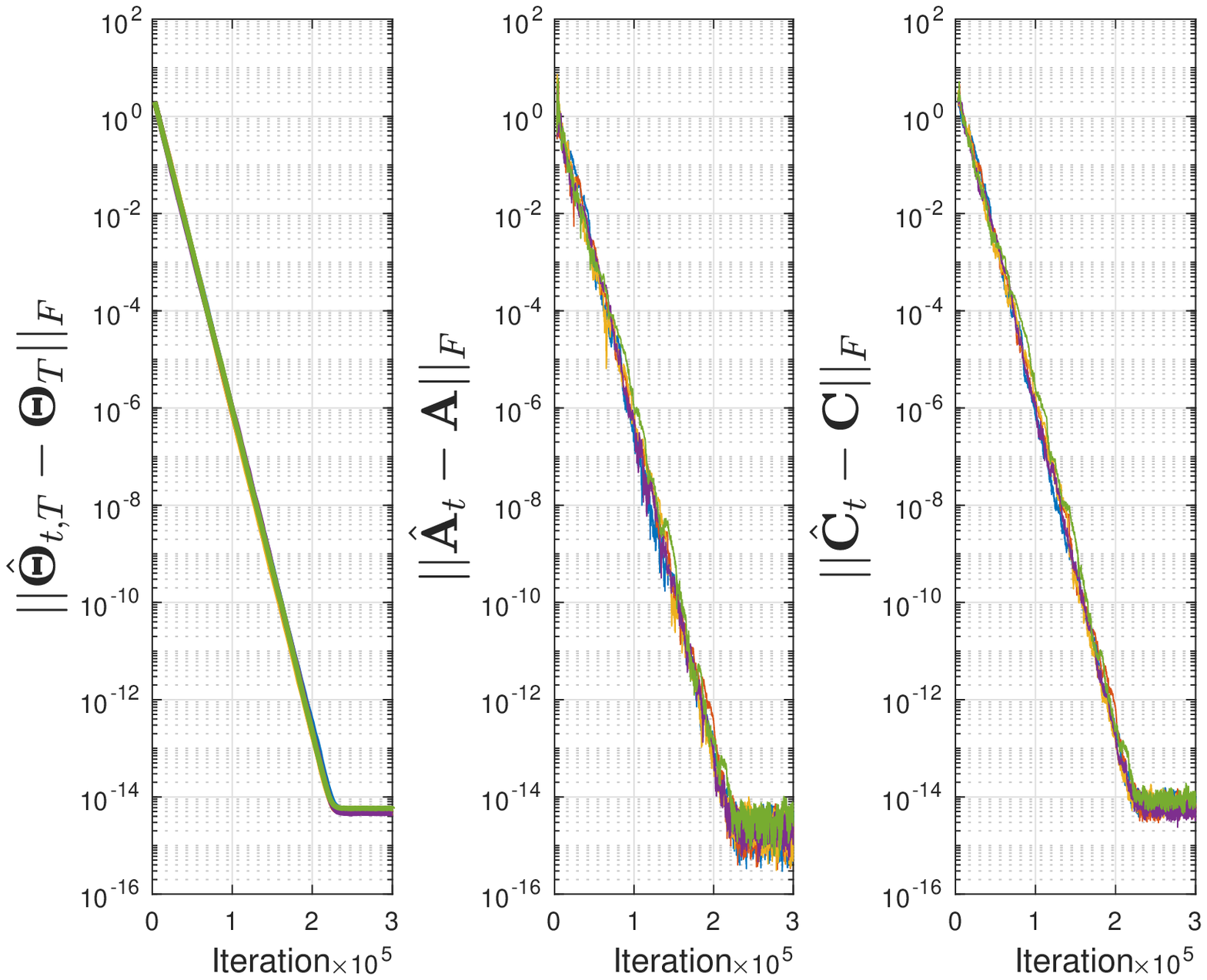}			\label{fig:miso20}}
	\subfigure{\includegraphics[width=0.32\textwidth]{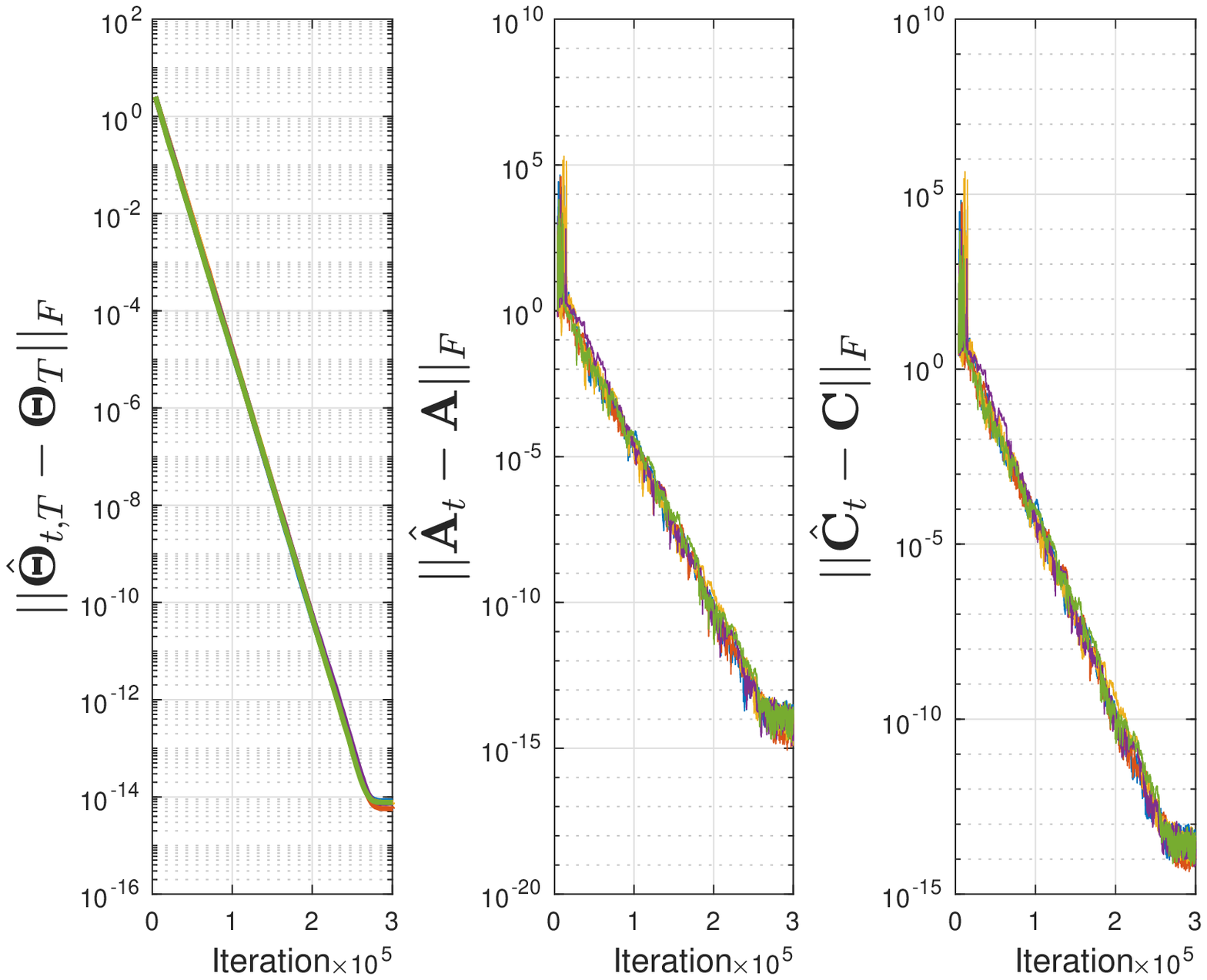}			\label{fig:miso25}}
	\subfigure{\includegraphics[width=0.32\textwidth]{figs/miso30.eps}			\label{fig:miso30}}
	\subfigure{\includegraphics[width=0.32\textwidth]{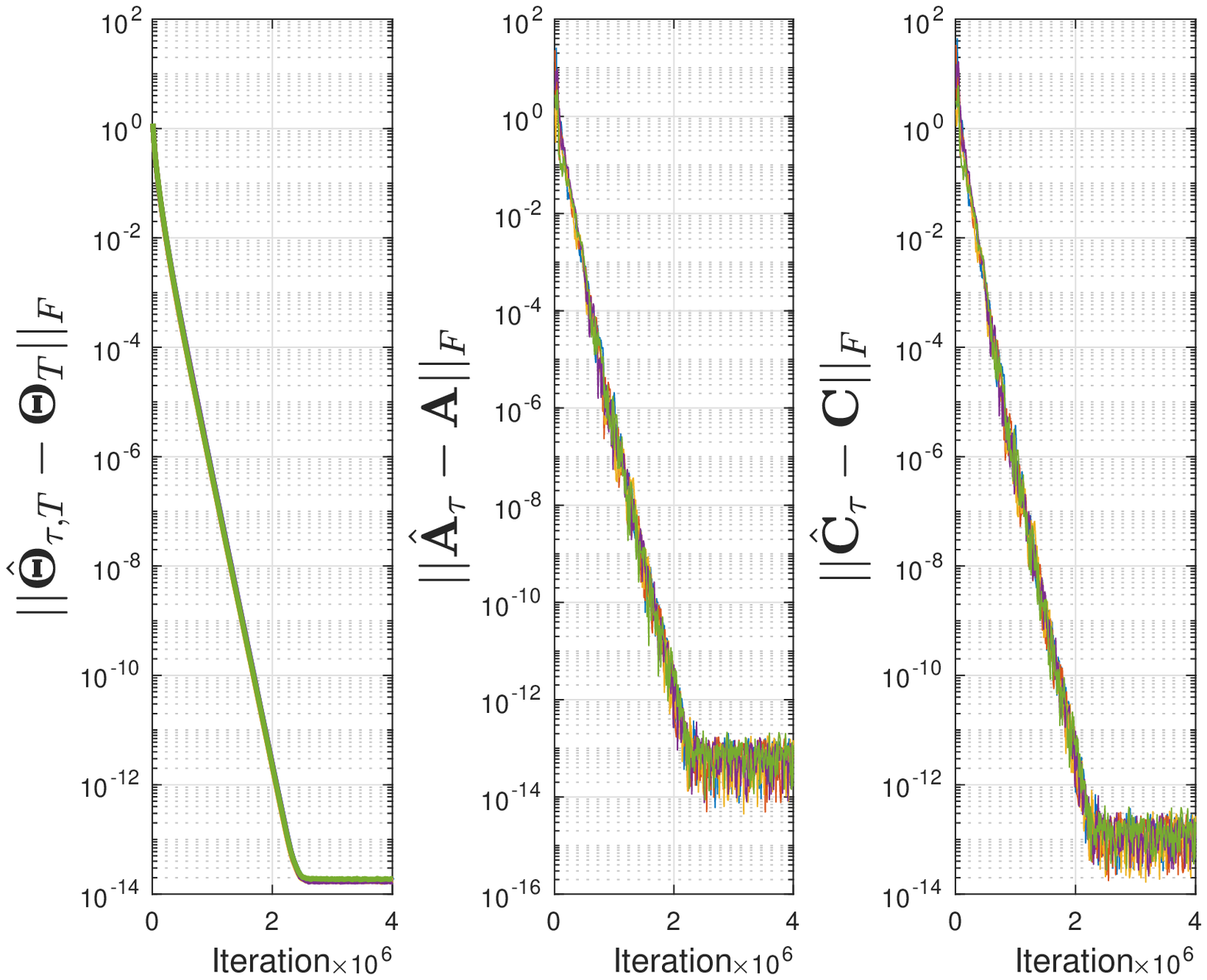}			\label{fig:miso20f}}
	\subfigure{\includegraphics[width=0.32\textwidth]{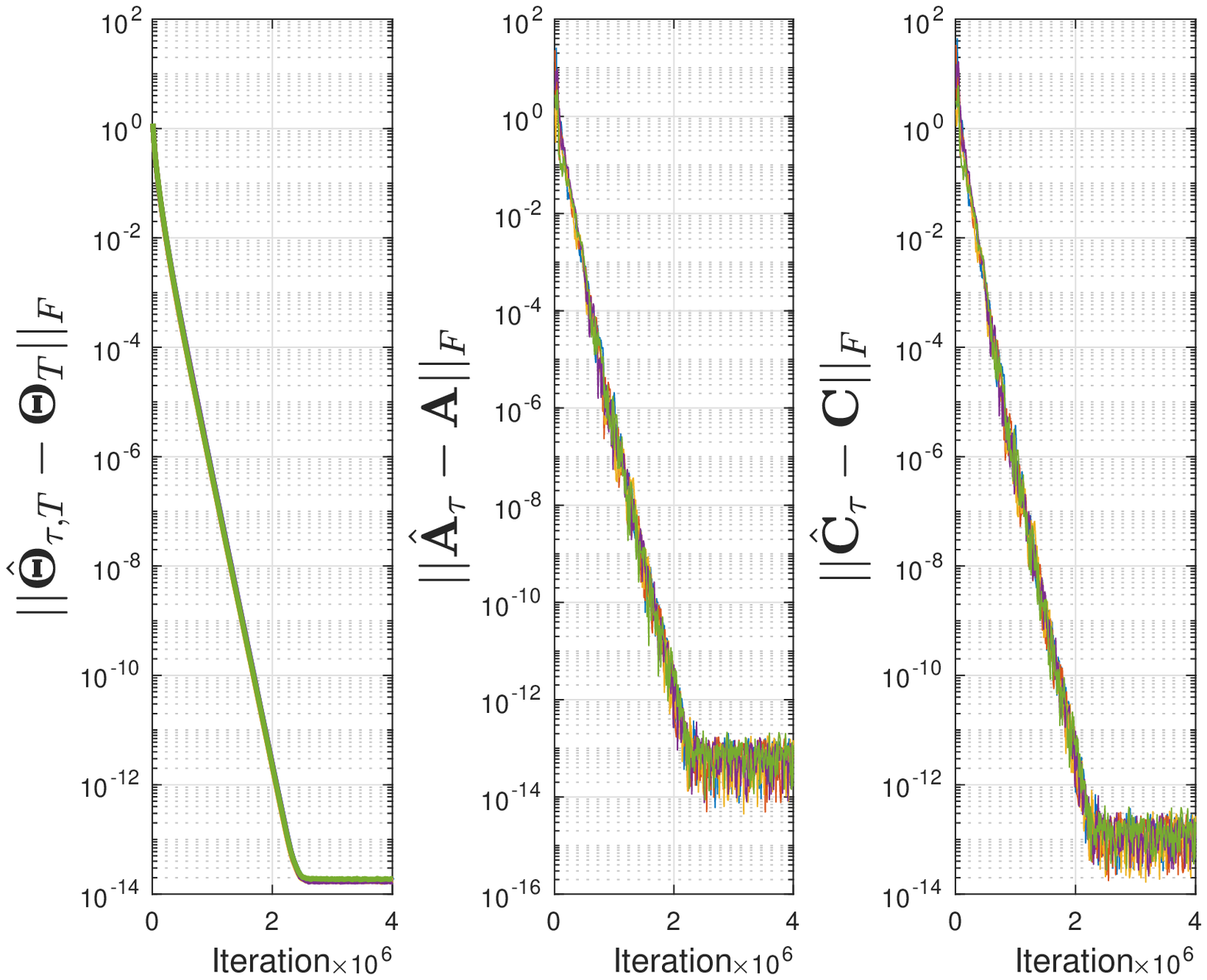}			\label{fig:miso25f}}
	\subfigure{\includegraphics[width=0.32\textwidth]{figs/miso30f.eps}			\label{fig:miso30f}}
	\subfigure{\includegraphics[width=0.32\textwidth]{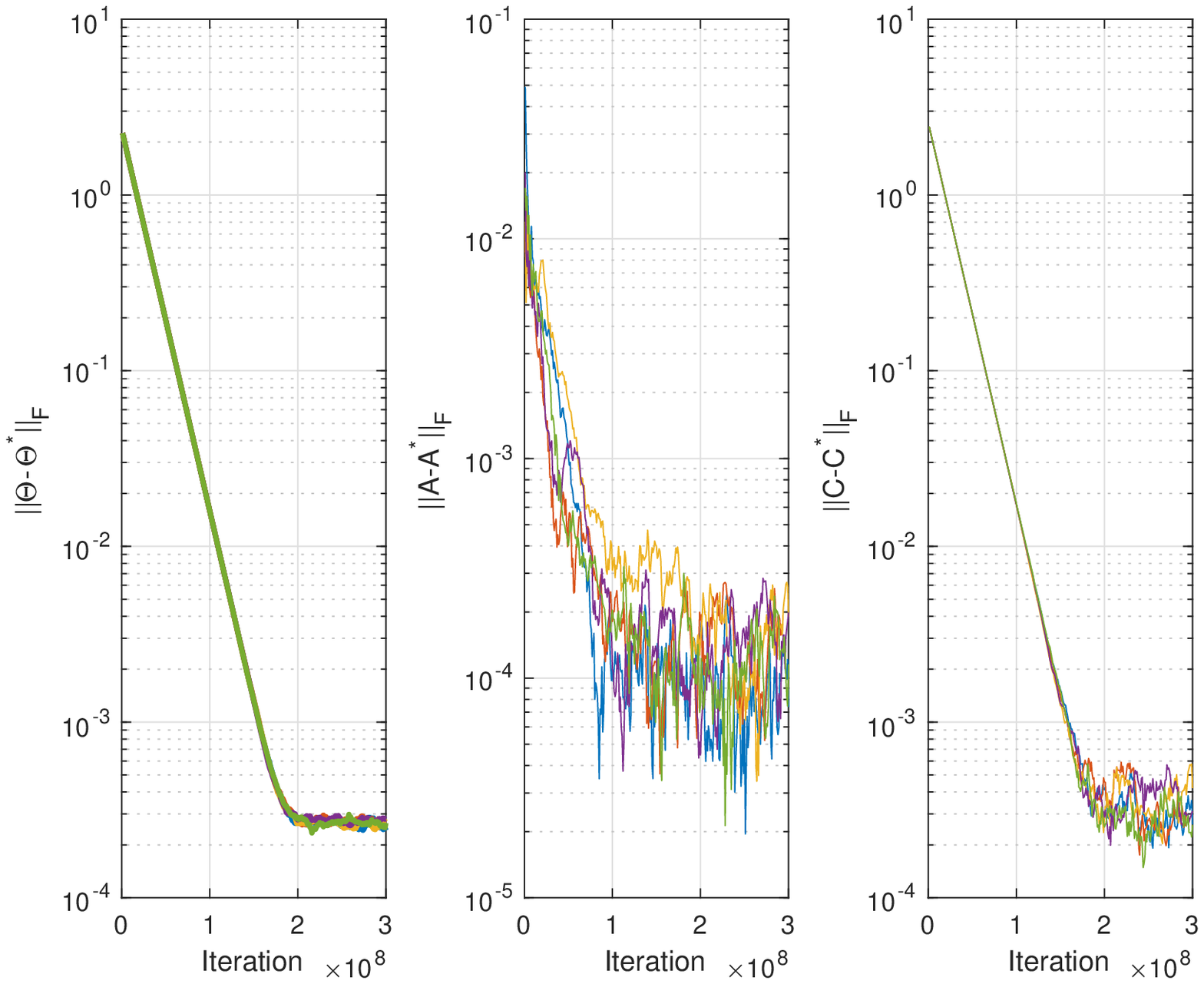}			\label{fig:miso20n}}
	\subfigure{\includegraphics[width=0.32\textwidth]{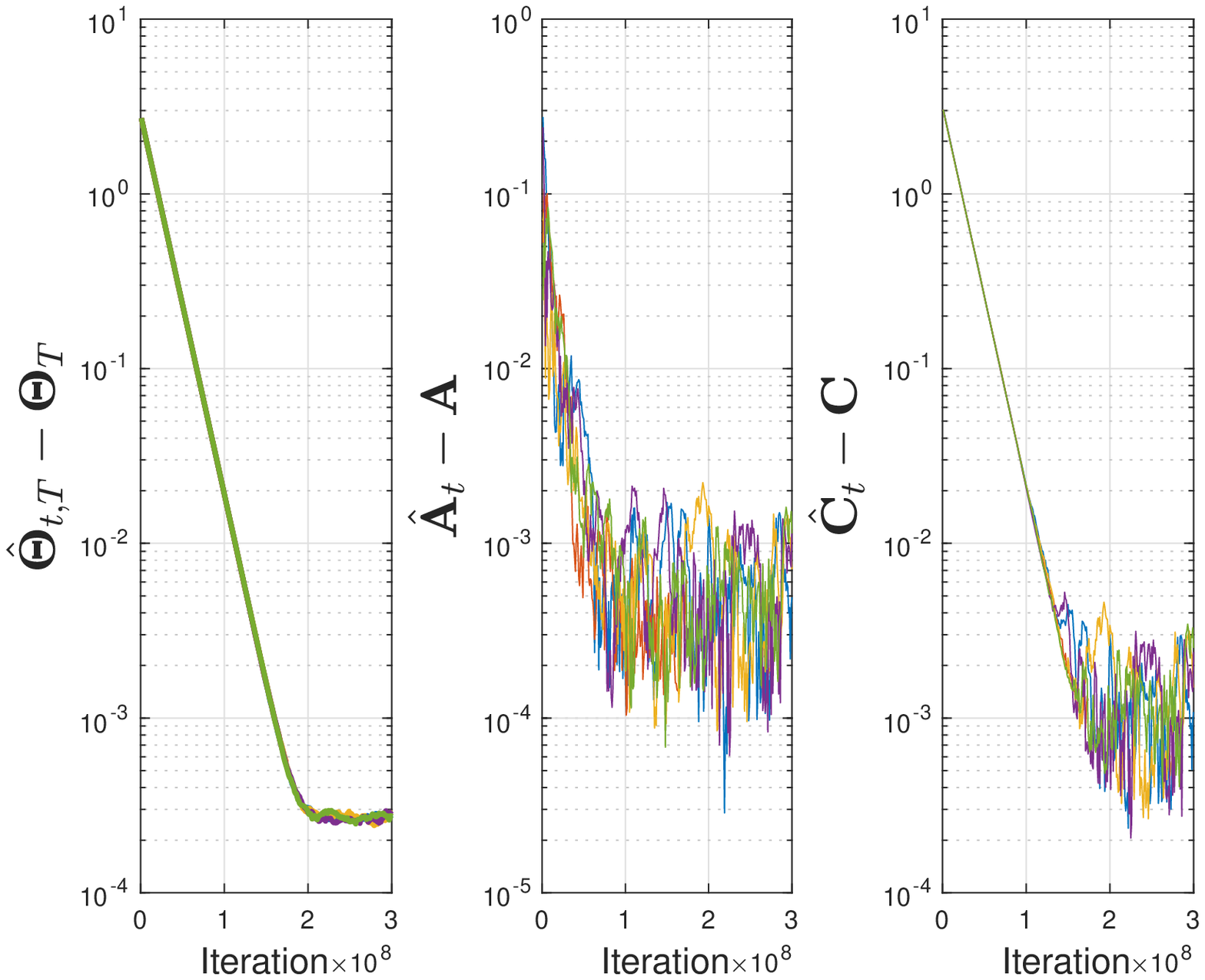}			\label{fig:miso25n}}
	\subfigure{\includegraphics[width=0.32\textwidth]{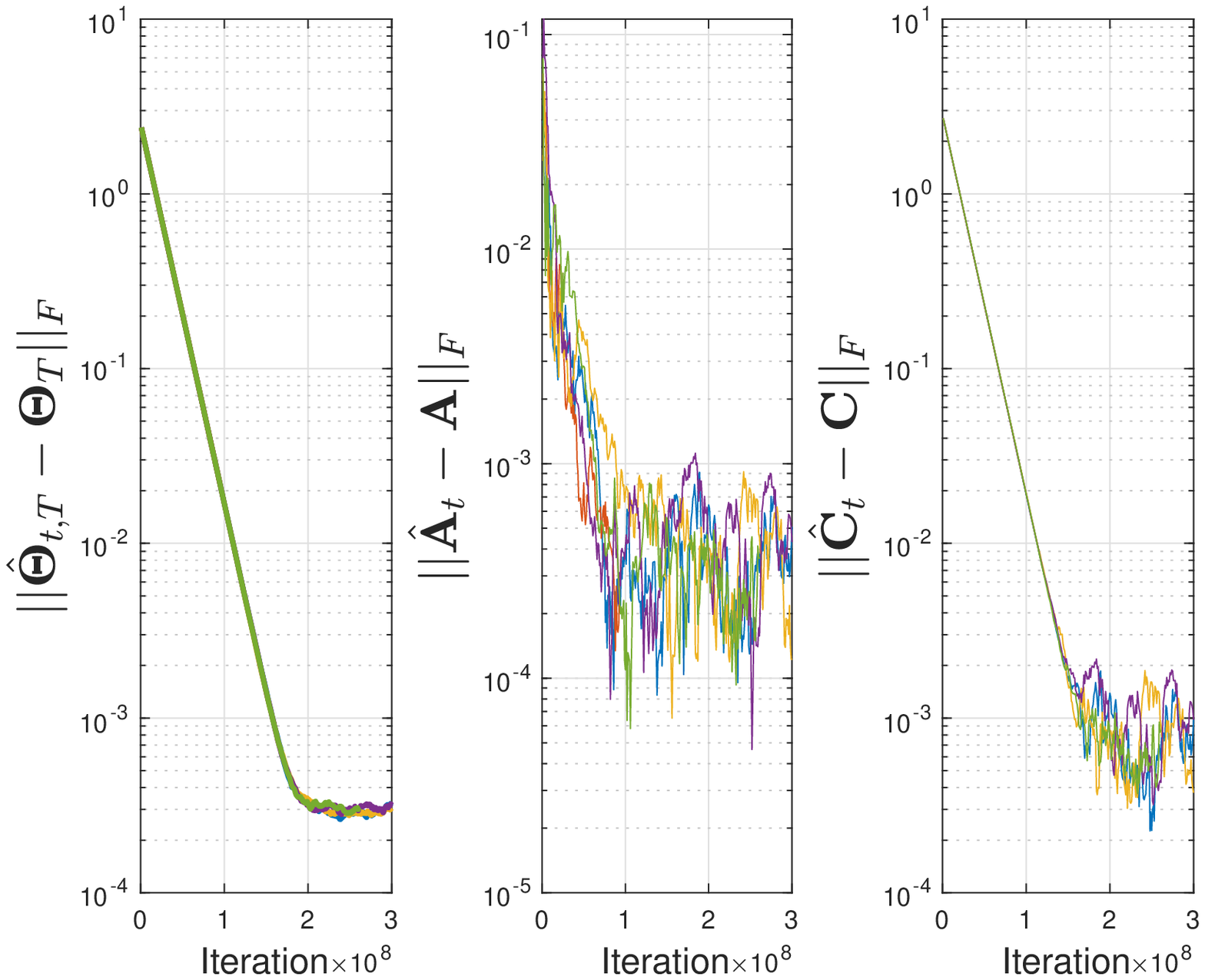}			\label{fig:miso30n}}
	\subfigure{\includegraphics[width=0.32\textwidth]{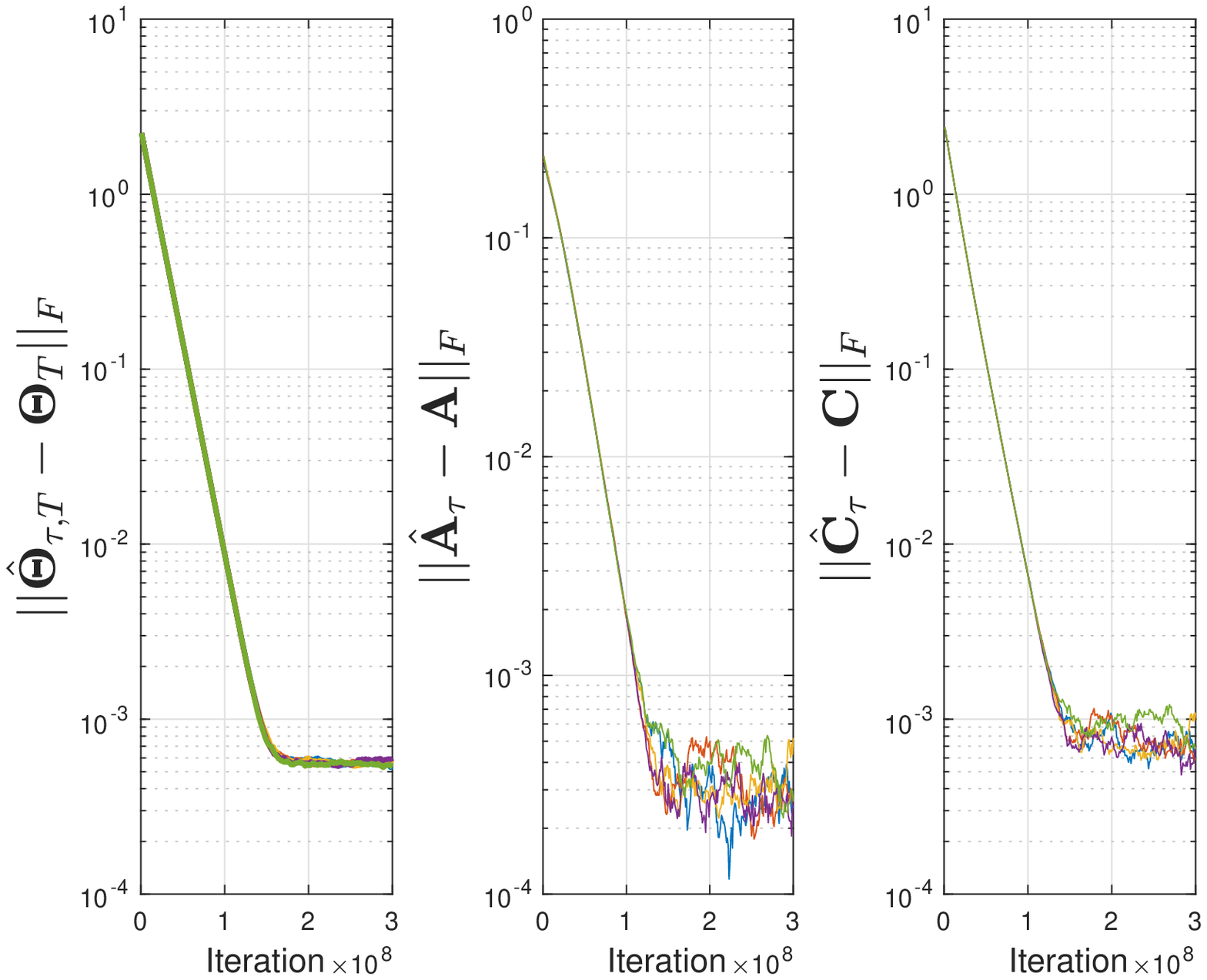}			\label{fig:miso20offline}}
	\subfigure{\includegraphics[width=0.32\textwidth]{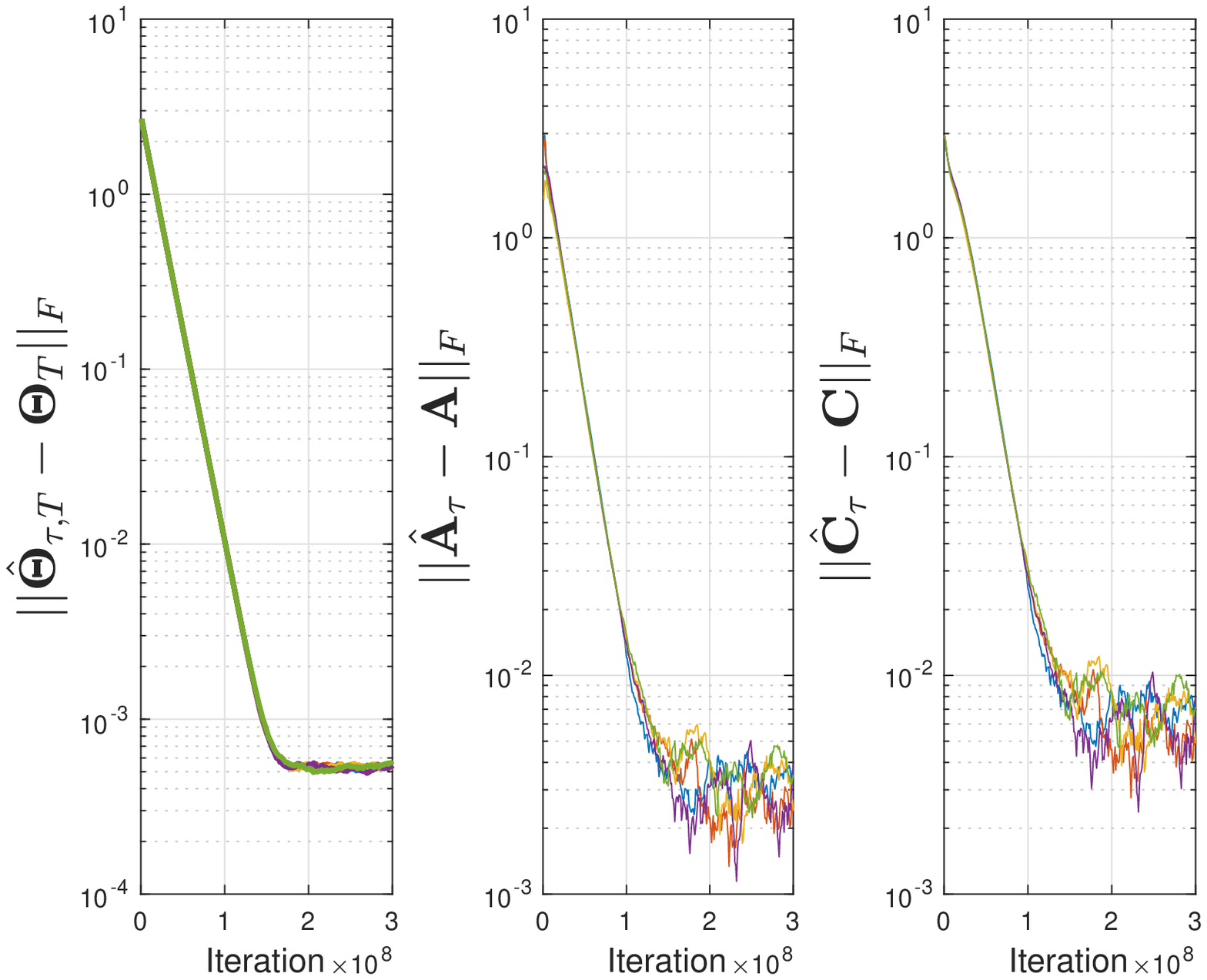}			\label{fig:miso25offline}}
	\subfigure{\includegraphics[width=0.32\textwidth]{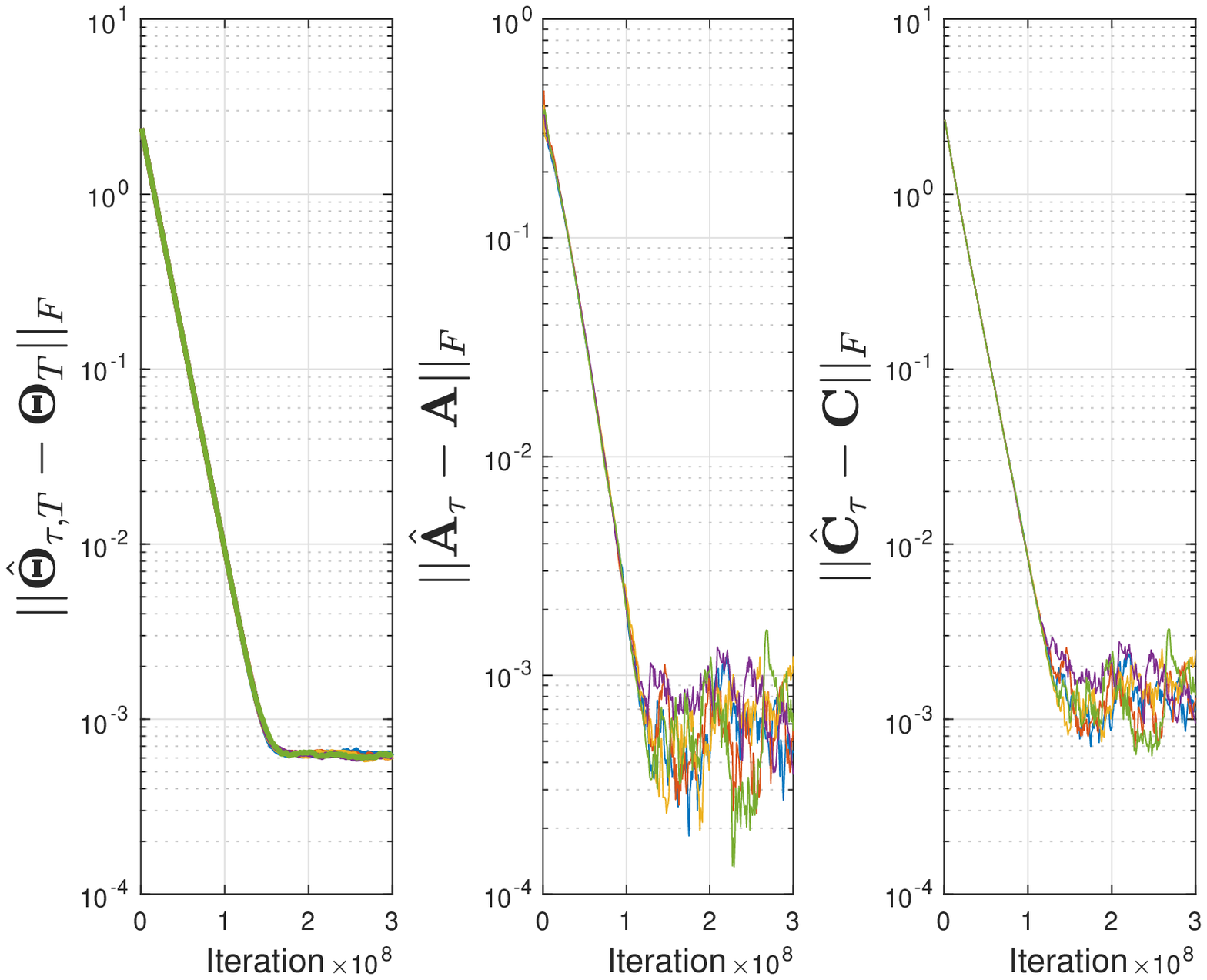}			\label{fig:miso30offline}}
	\caption{Results for MISO systems.  In (a)-(f), the systems are noise-free. In (g)-(l), the systems are noisy. In (a), (d), (g) and (j), $n=5$, $m=4$, $p=1$. In (b), (e), (h) and (k), $n=5$, $m=5$, $p=1$. In (c), (f), (i) and (l), $n=5$, $m=6$, $p=1$.}
\end{figure}

\subsection{MIMO}
Three different MIMO systems with hidden state dimensions 20, 25 and 30 are considered. In particular, for the three systems, we have $(n,m,p)=(5,4,4)$, $(n,m,p)=(5,5,4)$ and $(n,m,p)=(5,6,4)$. As the hidden state dimension increases, the number of required samples to reach a certain distance from the ground truth unknowns increases. The spectral radius of $\mathbf{A}$ in Figs. \ref{fig:mimo20}, \ref{fig:mimo25} and \ref{fig:mimo30} is $0.55$, $0.75$ and $0.64$, respectively. The convergence of Algorithm \ref{al:comb} for three  noise-free systems is depicted in Figs. \ref{fig:mimo20}-\ref{fig:mimo30}, where $(T,\eta)=\{(800,10^{-5}),(800, 10^{-5}),(800, 10^{-5})\}$. Figs. \ref{fig:mimo20f}-\ref{fig:mimo30f} depict the convergence of Algorithm \ref{al:comb1} with the same truncation lengths and learning rates. For Algorithm \ref{al:comb1}, the batch size is $10^4$. The convergence of Algorithm \ref{al:comb} for the three systems after the addition of white noise with standard deviation $0.1$ is given Figs. \ref{fig:mimo20n}-\ref{fig:mimo30n}. Furthermore, the convergence of Algorithm \ref{al:comb1} for identical systems is depicted Figs. \ref{fig:mimo20offline}-\ref{fig:mimo30offline}. For both approaches, we have $(T,\eta)=\{(60,5\times 10^{-8}),(60,5\times 10^{-8}),(60,5\times 10^{-8})\}$. The batch size is $10^7$. The control signal is white noise with zero mean and standard deviation $1$.

\begin{figure}[t!]
	\centering
	\subfigure{\includegraphics[width=0.32\textwidth]{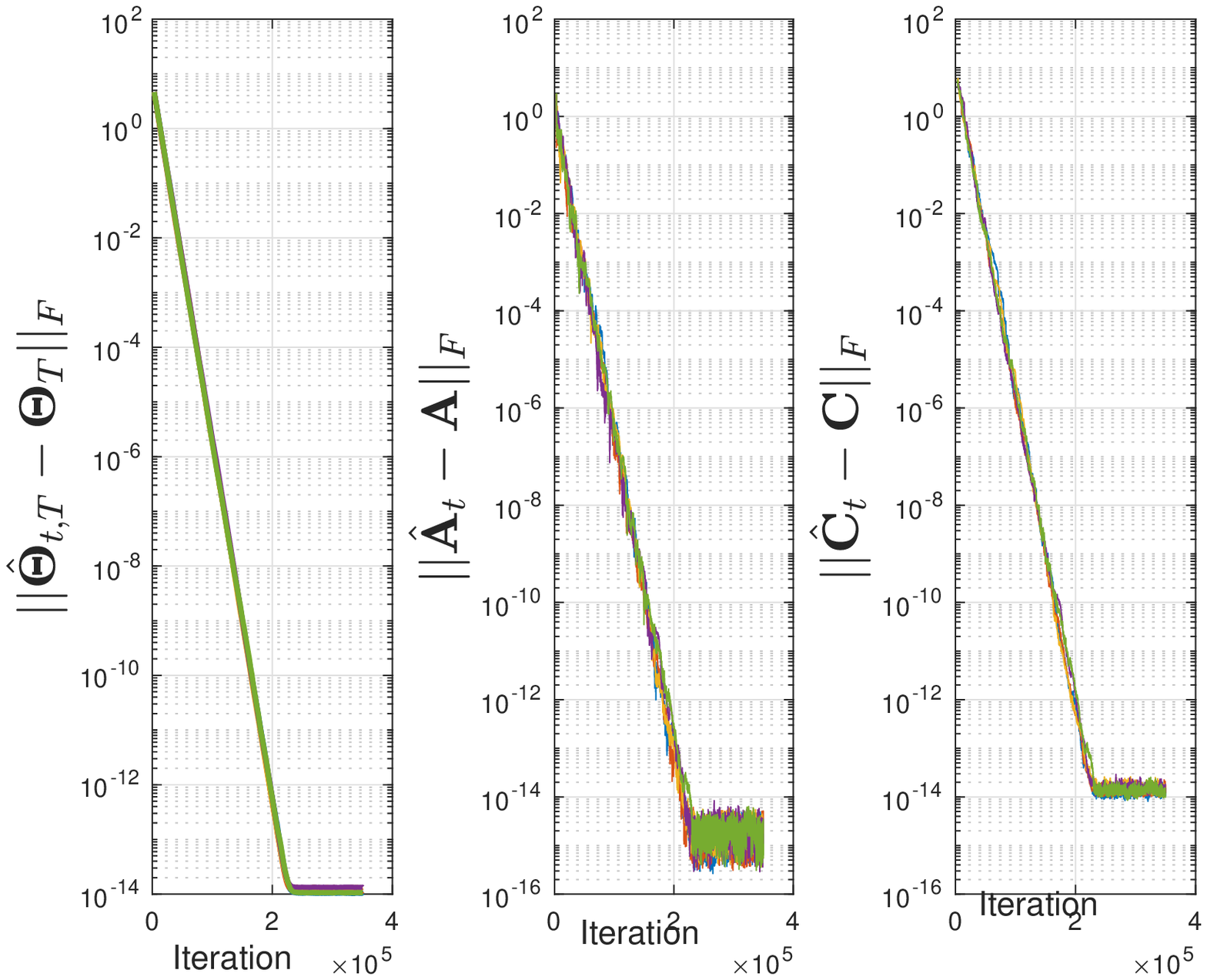}			\label{fig:mimo20}}
	\subfigure{\includegraphics[width=0.32\textwidth]{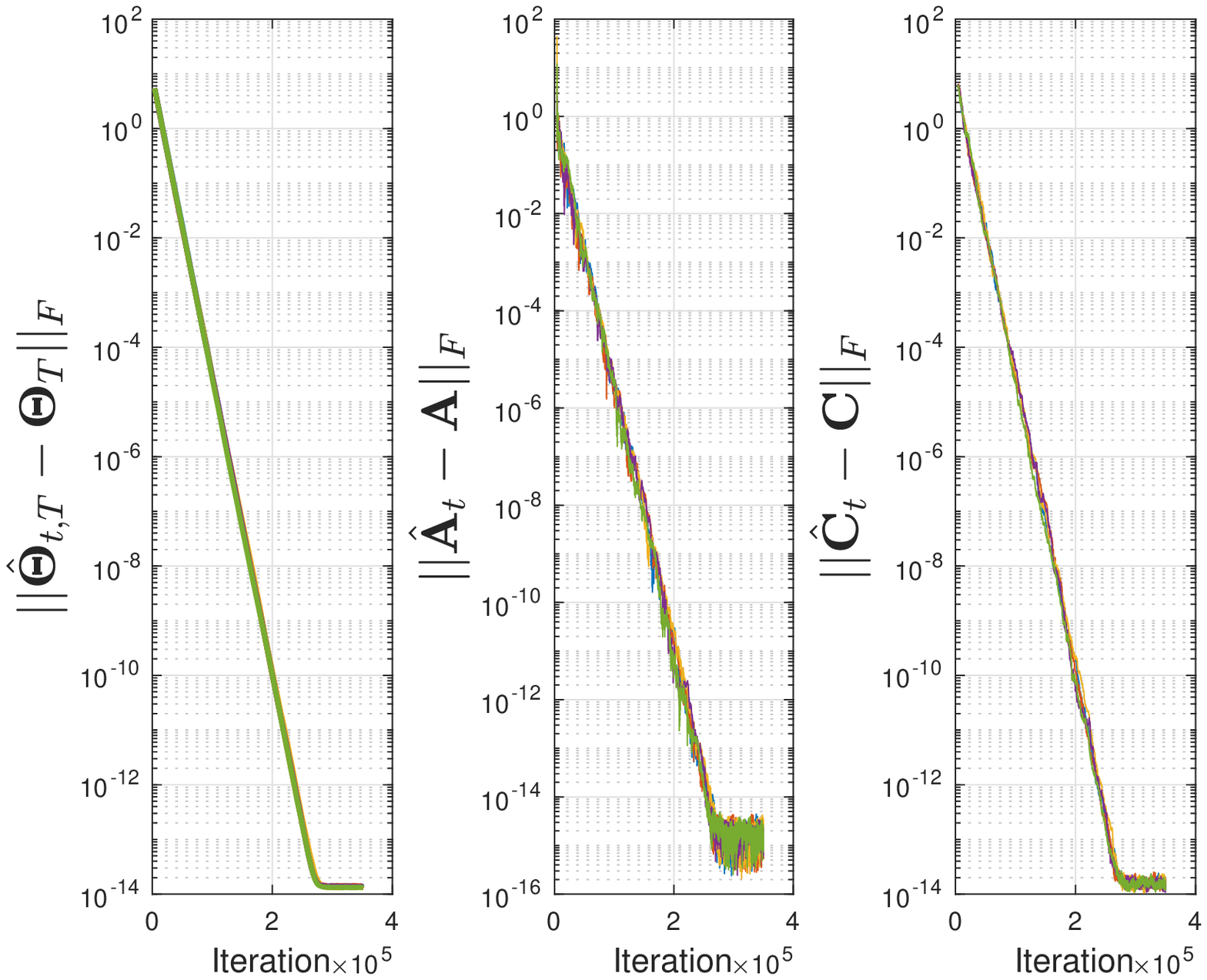}			\label{fig:mimo25}}
	\subfigure{\includegraphics[width=0.32\textwidth]{figs/mimo30.eps}			\label{fig:mimo30}}
	\subfigure{\includegraphics[width=0.32\textwidth]{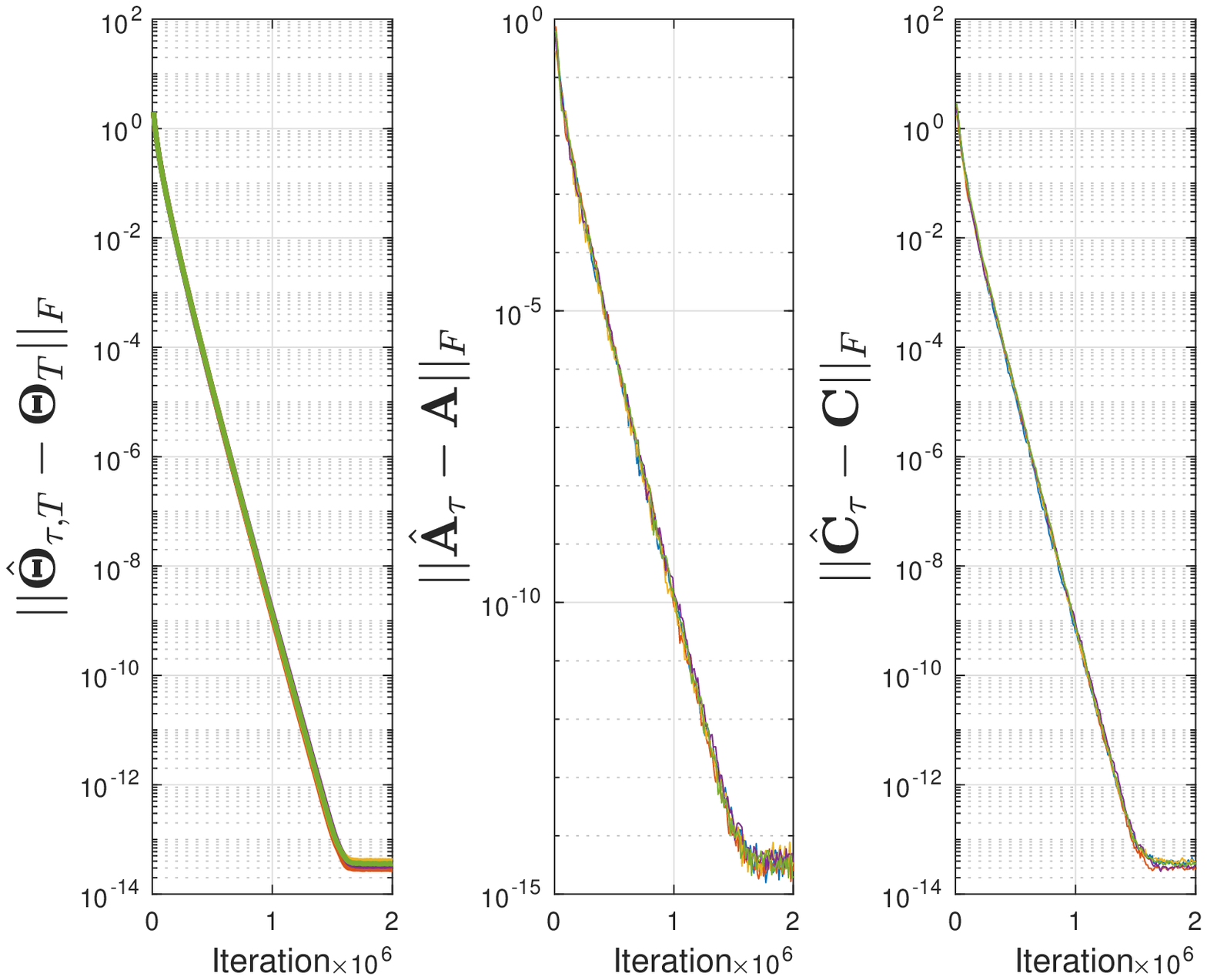}			\label{fig:mimo20f}}
	\subfigure{\includegraphics[width=0.32\textwidth]{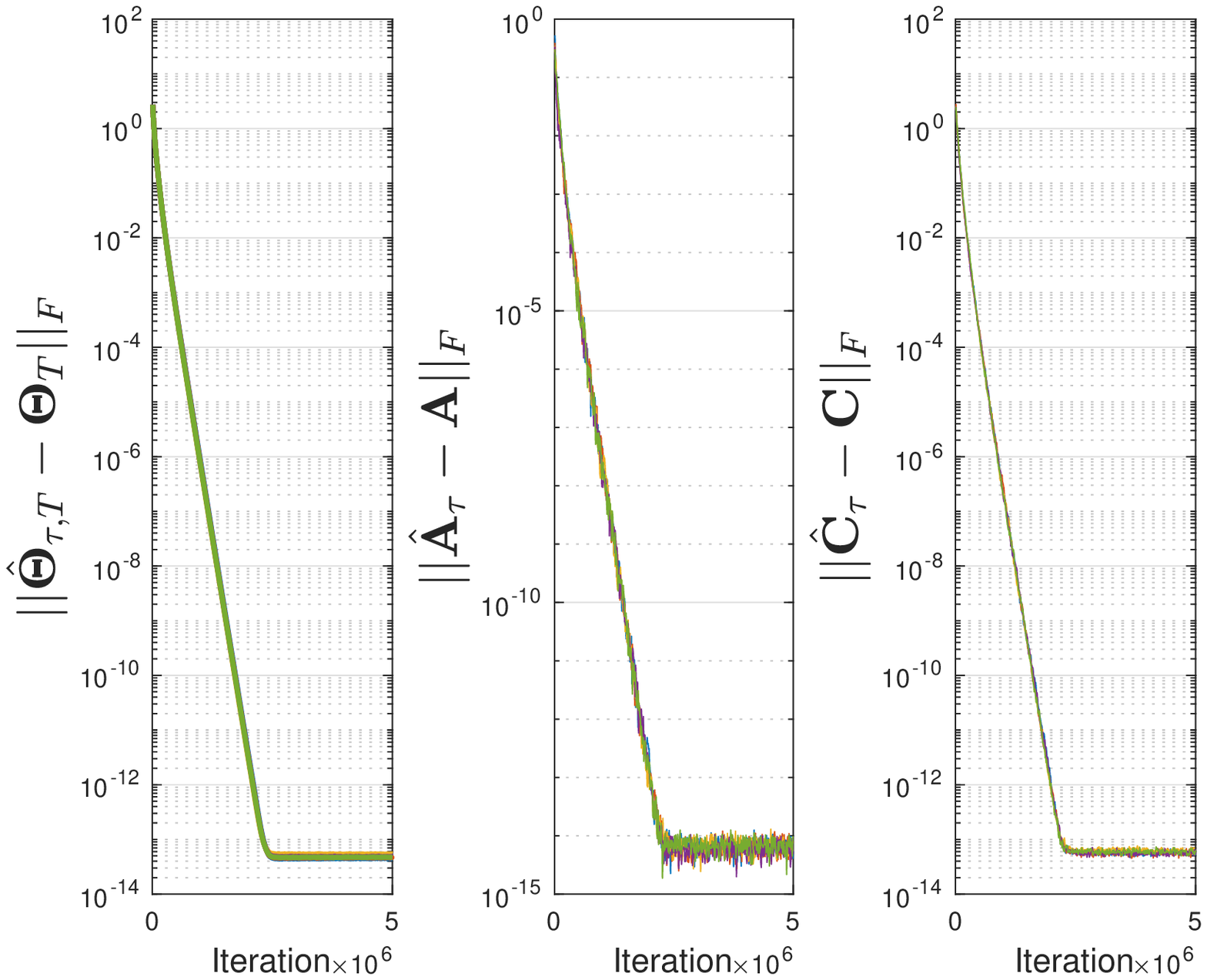}			\label{fig:mimo25f}}
	\subfigure{\includegraphics[width=0.32\textwidth]{figs/mimo30f.eps}			\label{fig:mimo30f}}
	\subfigure{\includegraphics[width=0.32\textwidth]{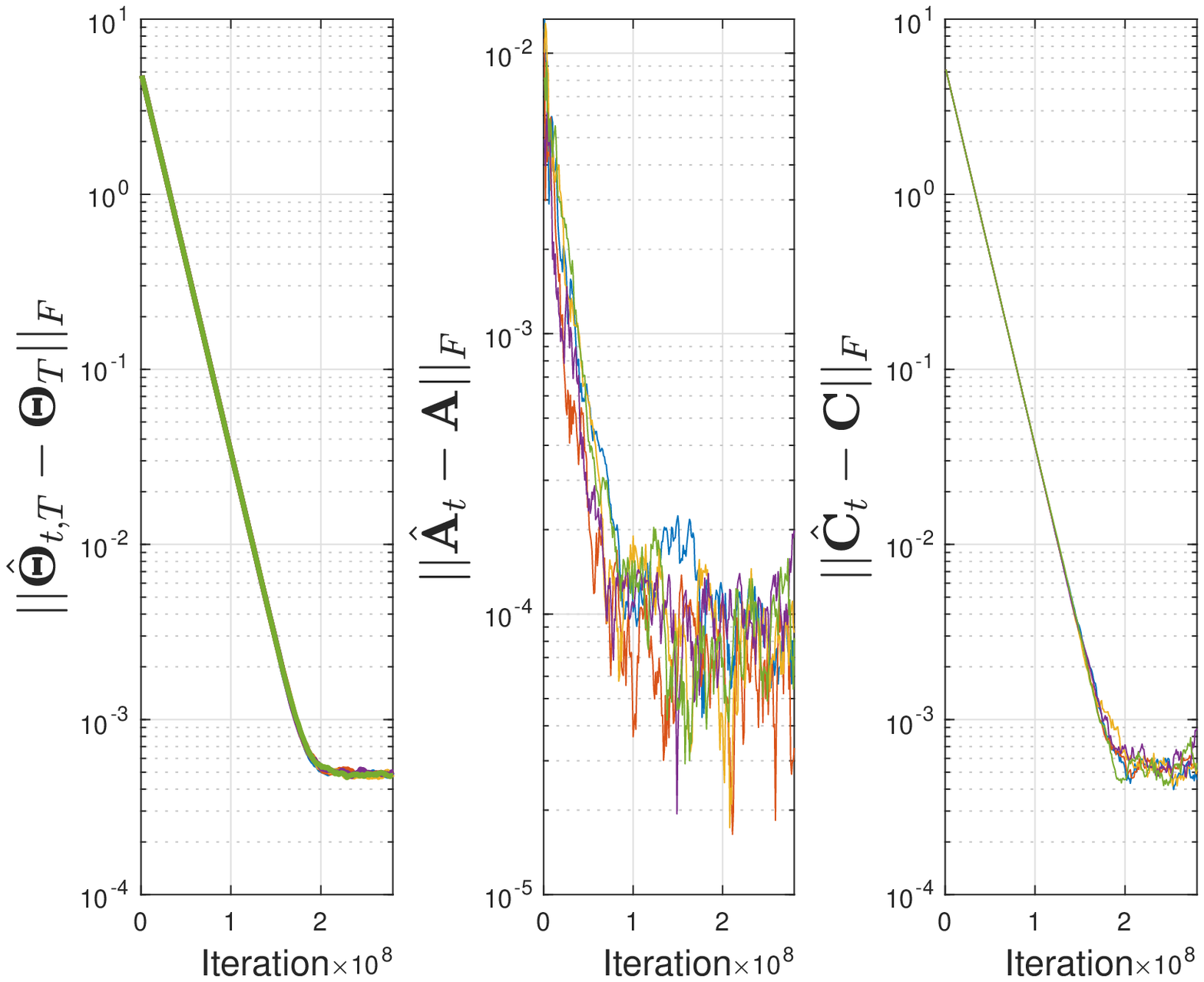}			\label{fig:mimo20n}}
	\subfigure{\includegraphics[width=0.32\textwidth]{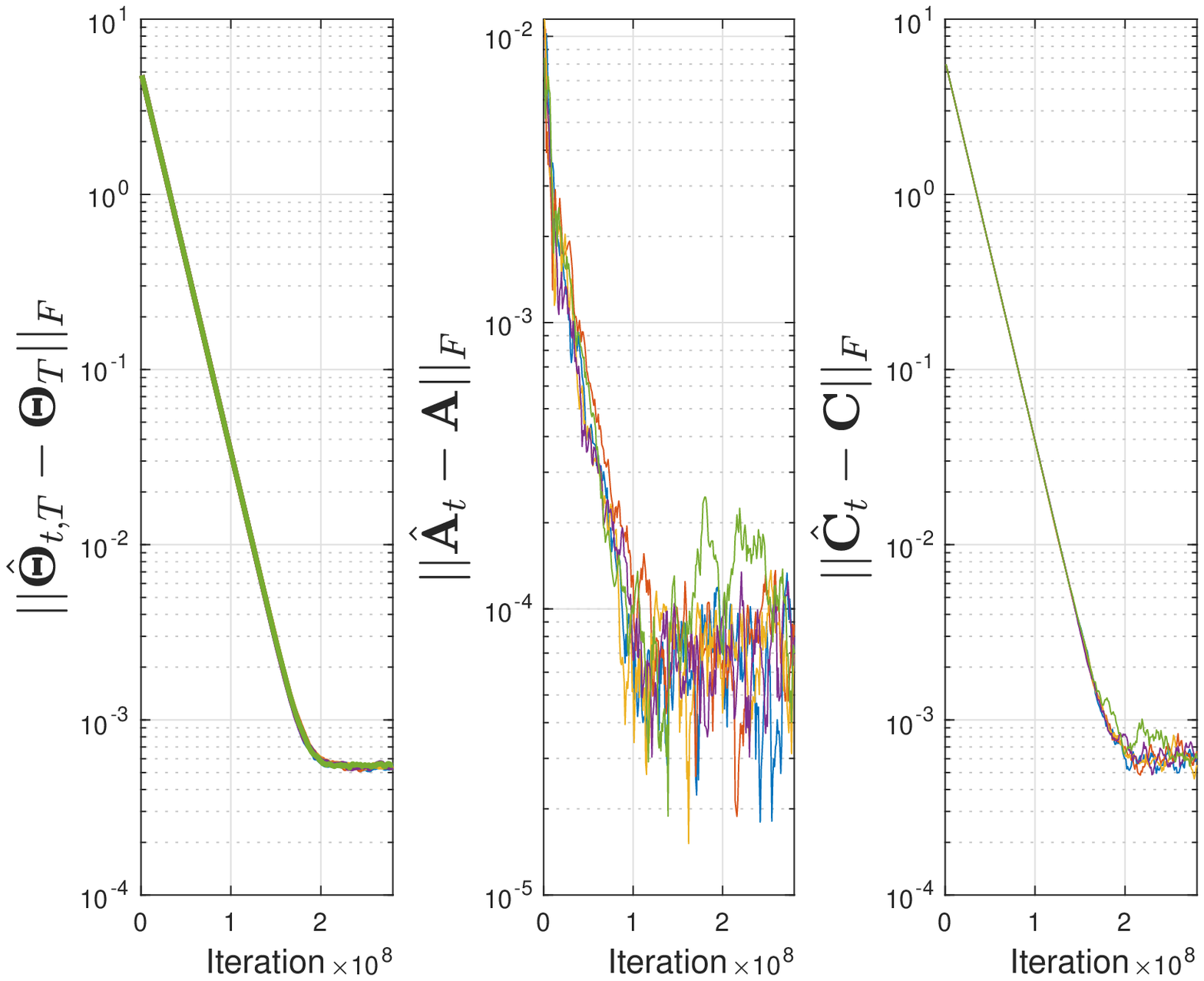}			\label{fig:mimo25n}}
	\subfigure{\includegraphics[width=0.32\textwidth]{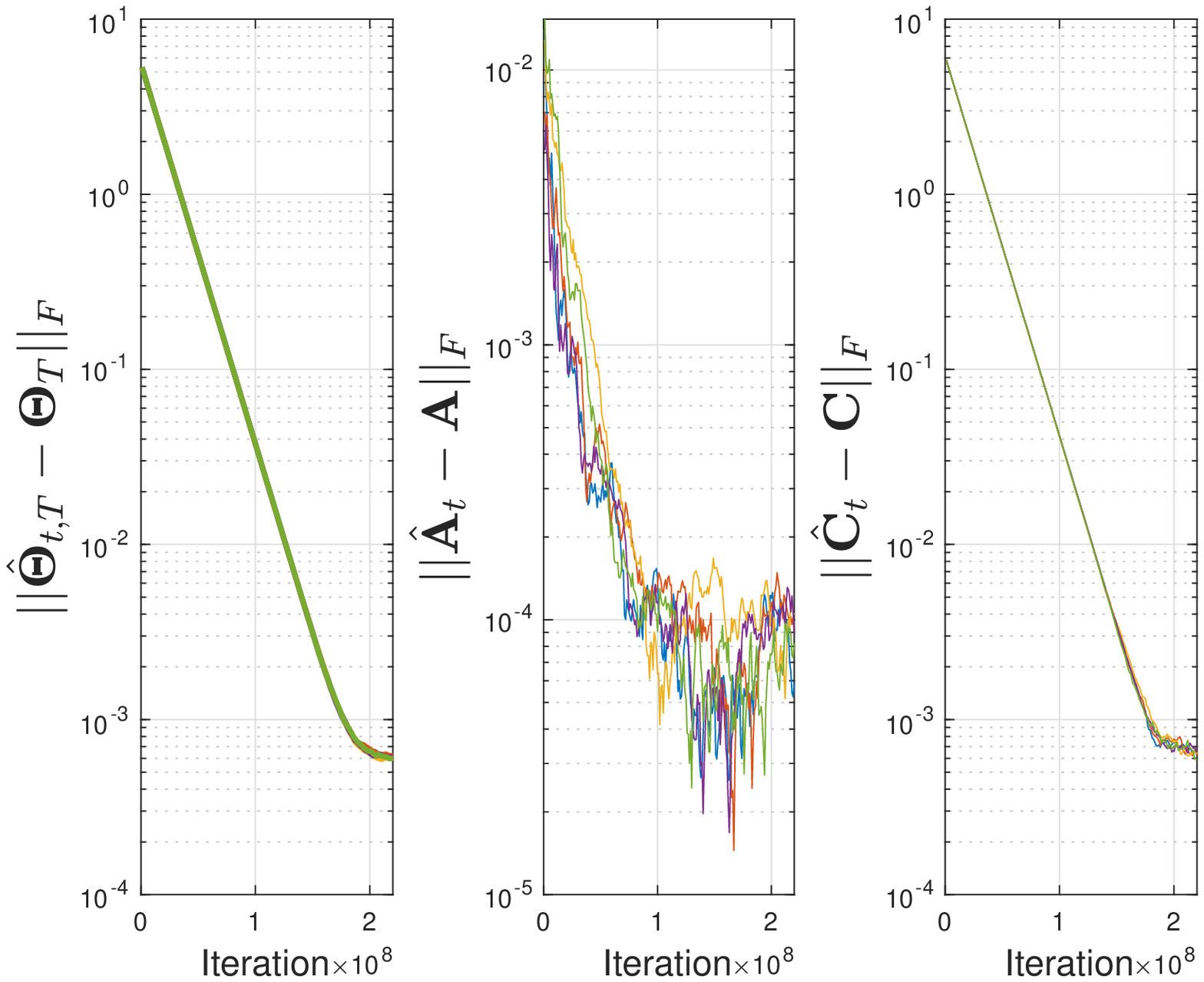}			\label{fig:mimo30n}}
	\subfigure{\includegraphics[width=0.32\textwidth]{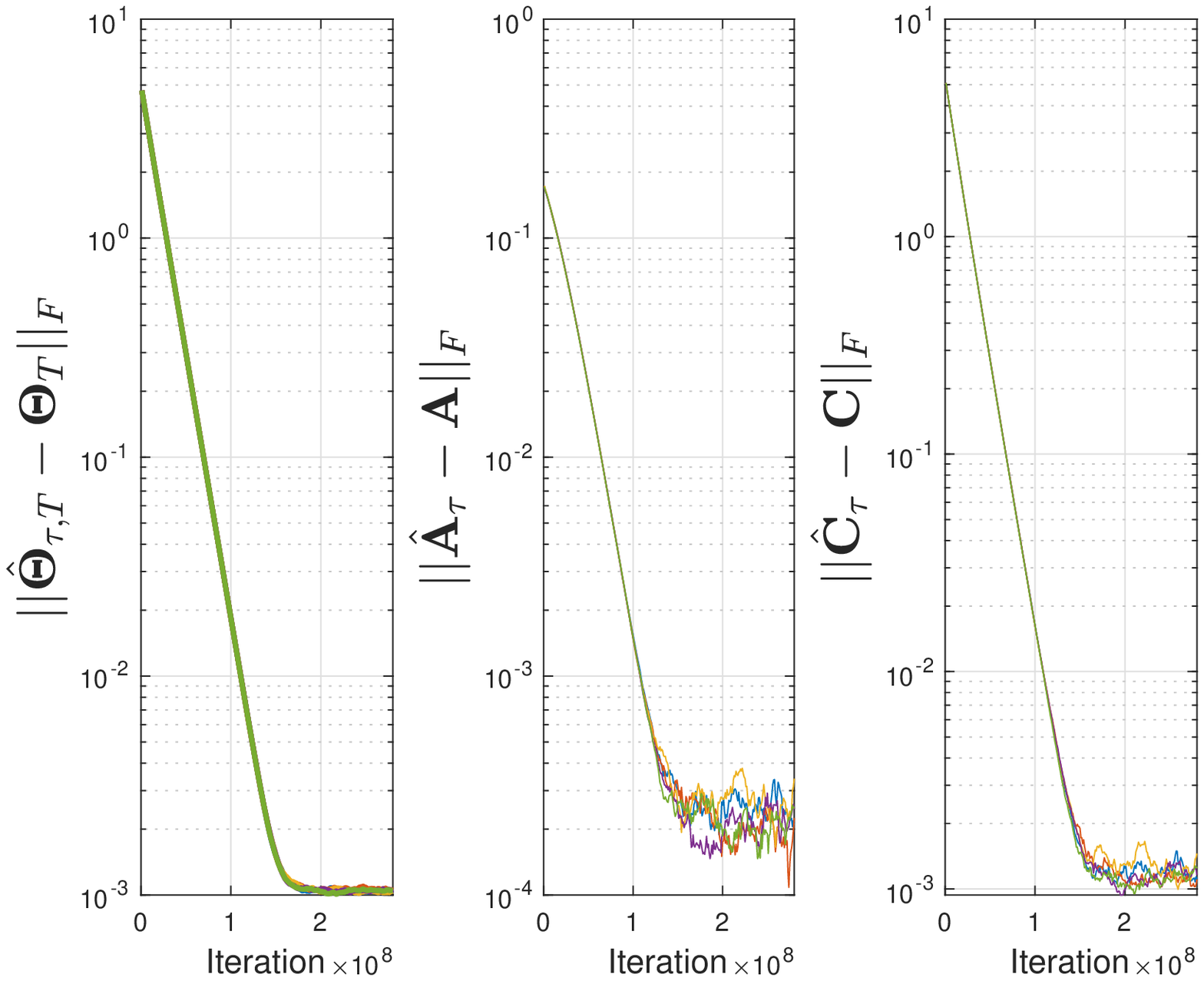}			\label{fig:mimo20offline}}
	\subfigure{\includegraphics[width=0.32\textwidth]{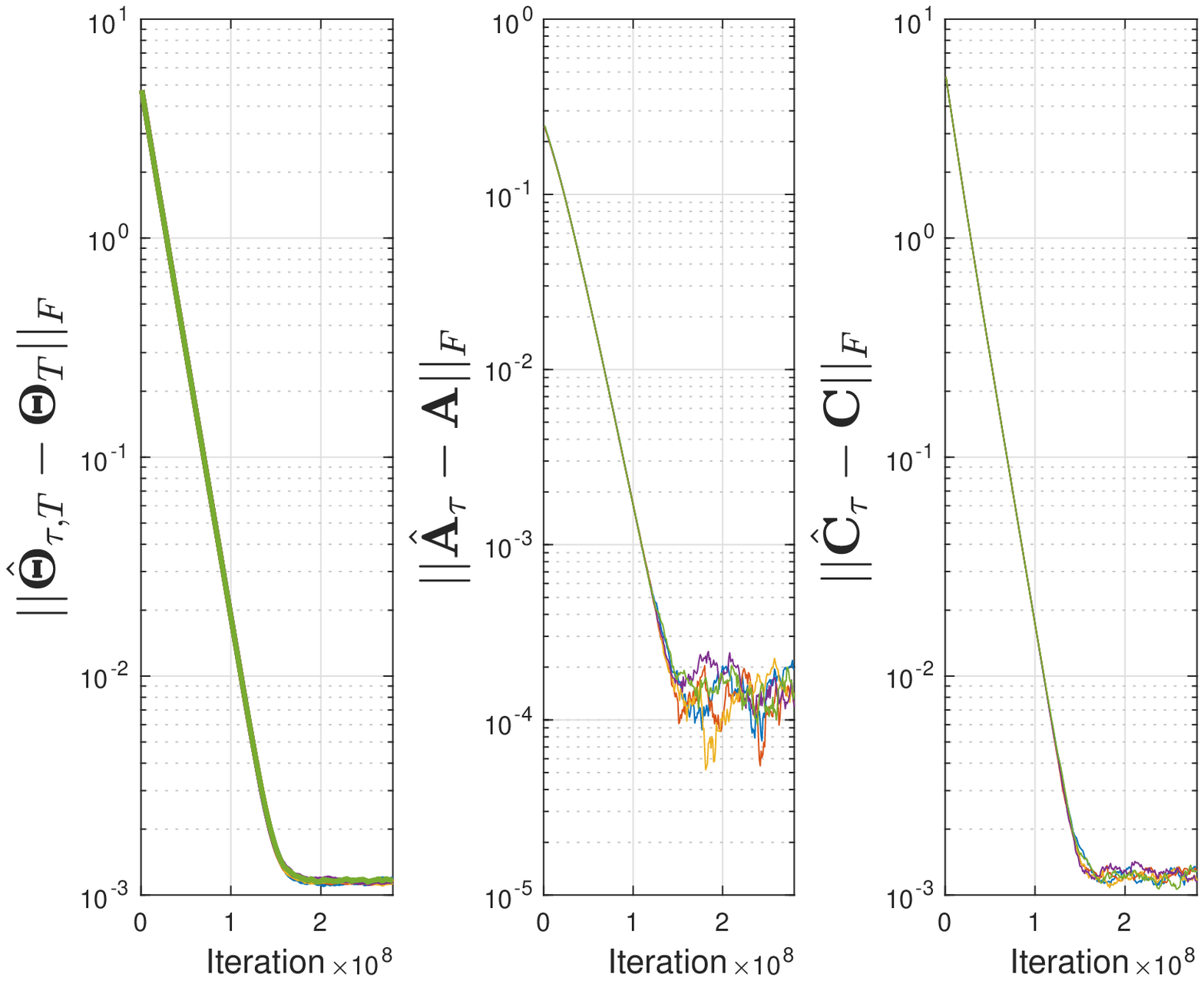}			\label{fig:mimo25offline}}
	\subfigure{\includegraphics[width=0.32\textwidth]{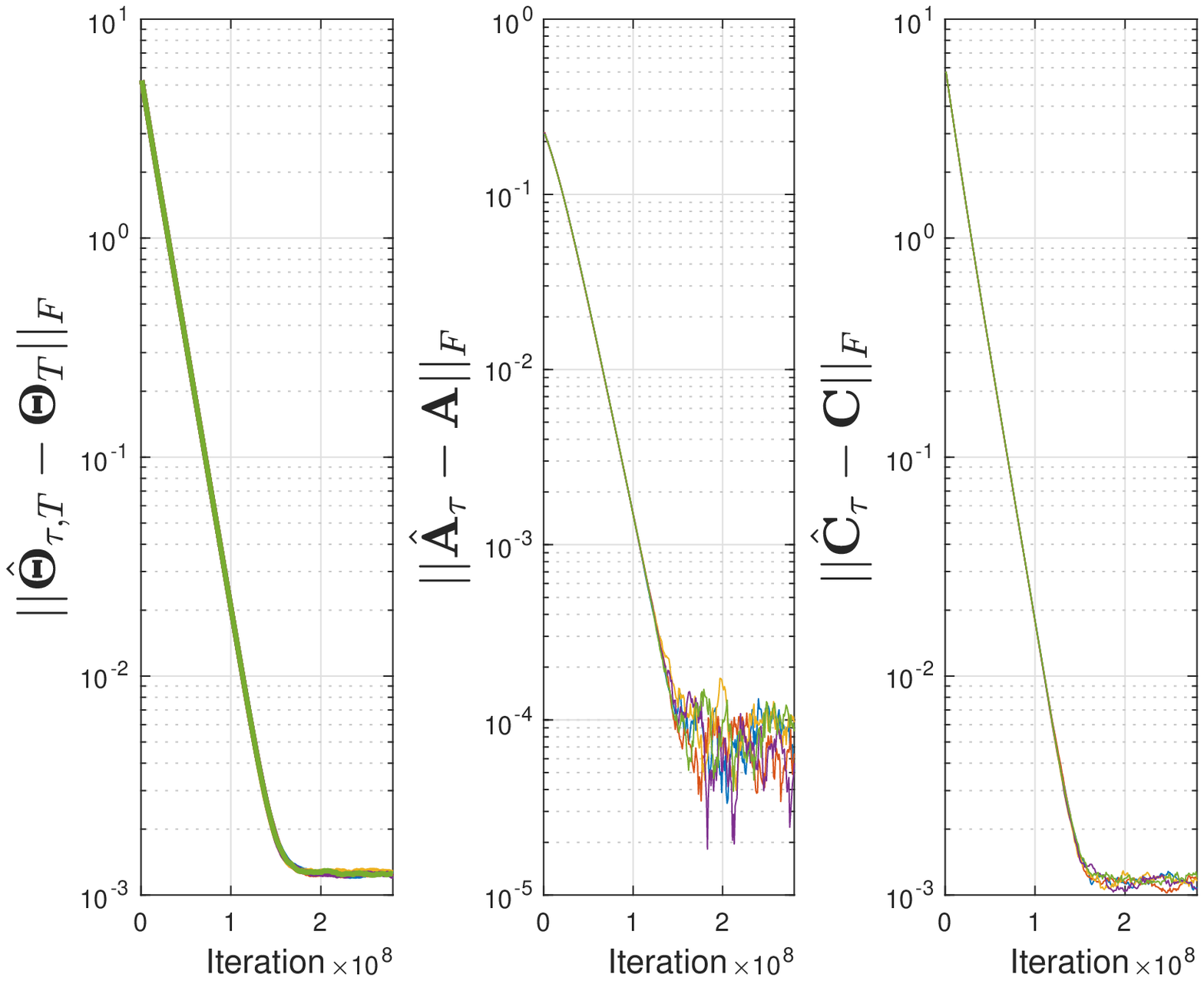}			\label{fig:mimo30offline}}
	\caption{Results for MIMO systems. In (a)-(f), the systems are noise-free. In (g)-(l), the systems are noisy. In (a), (d), (g) and (j), $n=5$, $m=4$, $p=4$. In (b), (e), (h) and (k), $n=5$, $m=5$, $p=4$. In (c), (f), (i) and (l), $n=5$, $m=6$, $p=4$.}
\end{figure}

\end{document}